%% file: main.tex
\documentclass[manuscript,nonacm]{acmart}

\usepackage{amsthm}
\usepackage{subcaption}
\usepackage{adjustbox}
\usepackage{bm}
\usepackage{listings}
\usepackage{tikz}
\usepackage{algorithm}
\usepackage{algorithmic}
\usepackage{enumitem}
\usepackage{cleveref}
\usepackage[altpo]{backnaur}
\usepackage{pgfplots}
\usepackage{nicefrac}       
\usepackage{soul}

\usetikzlibrary{arrows,positioning,backgrounds}
\usetikzlibrary{plotmarks}
\pgfplotsset{compat=1.18}

\pgfkeys{/pgf/number format/.cd,
    set thousands separator={}}

\input{macros.tex}

\title{\ffoil: A Declarative Query Language for Explaining Classification Models}

\author{Marcelo Arenas}
\affiliation{%
  \institution{Pontificia Universidad Católica de Chile}
  \city{Santiago}
  \country{Chile}
}

\author{Pablo Barcel\'o}
\affiliation{%
  \institution{Pontificia Universidad Católica de Chile}
  \city{Santiago}
  \country{Chile}
}

\author{Diego Bustamante}
\affiliation{%
  \institution{Pontificia Universidad Católica de Chile}
  \city{Santiago}
  \country{Chile}
}

\author{Jose Caraball}
\affiliation{%
  \institution{Pontificia Universidad Católica de Chile}
  \city{Santiago}
  \country{Chile}
}

\author{Mar\'ia Alejandra Schild}
\affiliation{%
  \institution{Pontificia Universidad Católica de Chile}
  \city{Santiago}
  \country{Chile}
}

\author{Bernardo Subercaseaux}
\affiliation{%
  \institution{Carnegie Mellon University}
  \city{Pittsburgh}
  \state{Pennsylvania}
  \country{USA}
}

\newcommand{\supplementary}{}

\begin{document}
	
\begin{abstract}
The XAI community has studied a wide range of queries and scores for explaining predictions of ML models. From a data management perspective, this proliferation of explanation notions calls for declarative query languages in which such notions can be specified, combined, and analyzed uniformly. In this paper, we develop such a framework for Boolean models. We first revisit $\foil$, an interpretability query language for black-box models, and show that it has two fundamental limitations: it cannot express central optimality-based explanation queries, and its evaluation problem over decision trees is hard for every level of the polynomial hierarchy. We then introduce $\ffoil$, a query language based on $\foil$ with an extended vocabulary and a layered structure. We show that $\ffoil$ can express a broad family of explanation notions, including abductive, contrastive, feature-based, and distance-based queries. We also prove that the evaluation problem for each query in $\ffoil$ belongs to the Boolean hierarchy over every class of Boolean models for which some basic predicates can be evaluated in polynomial time. In particular, that property holds for deterministic and decomposable Boolean circuits. Finally, we introduce $\optfoil$, an optimization-oriented fragment of $\ffoil$ for computing explanations that are minimal with respect to strict partial orders, and prove that its evaluation problem is in $\mathrm{FP}^{\mathrm{NP}}$ under the same tractability assumptions. These complexity results have a direct algorithmic consequence: a fixed ExplAIner query can be evaluated with a fixed number of calls to a SAT solver, while a notion of explanation specified in $\optfoil$ can be computed with a polynomial number of such calls. This is particularly relevant in formal XAI, where SAT solvers have been successfully used to compute explanations for several classes of ML models.
\end{abstract}

 \maketitle

\section{Introduction}\label{sec:intro}
\input{new-sections/new-intro}

\section{Background}\label{sec:background}
\input{new-sections/background}

\section{First Order Interpretability Logic}\label{sec:foil-limitations}
\input{new-sections/foil}

\section{\ffoil: a tractable logic for explainability}\label{sec:dt-foil}
\input{new-sections/f-foil}

\section{\optfoil: computing explanations efficiently}\label{sec:opt-dt-foil} 
\input{new-sections/opt-foil}

\section{Concluding remarks and future work}\label{sec:conclusions}
\input{new-sections/conclusions}

\section{Acknowledgements}\label{sec:acknowledgements}
\input{new-sections/acknowledgements}




\bibliographystyle{ACM-Reference-Format}
\bibliography{main}

\ifdefined\supplementary
\newpage

\appendix
\newpage

\section{Supplementary Material}\label{sec:appendix}
\input{new-sections/app-proofs}
\fi
	
\end{document}

%% file: macros.tex
\newcommand{\bh}{{\rm BH}}

\newcommand{\ph}{{\rm PH}}

\newcommand{\ptime}{{\rm P}}
\newcommand{\fptime}{{\rm FP}}
\newcommand{\np}{{\rm NP}}
\newcommand{\stp}{\Sigma_2^{\rm{P}}}

\newcommand{\conp}{{\rm coNP}}
\newcommand{\fpnp}{\fptime$^{\np}$}
\newcommand{\es}{\mathbf{e}}

\newcommand{\M}{\mathcal{M}}
\newcommand{\T}{\mathcal{T}}
\newcommand{\dm}{{\textit dim}}
\newcommand\false{\mathbf{false}}
\newcommand\true{\mathbf{true}}

\newcommand\var{\mathsf{var}}
\newcommand{\yes}{\textsc{Yes}}
\newcommand{\no}{\textsc{No}}

\newcommand{\C}{\mathcal{C}}
\newcommand{\dt}{\mathsf{DTree}}
\newcommand{\bdd}{\mathsf{BDD}}

\newcommand{\dnf}{\mathsf{DNF}}
\newcommand{\cnf}{\mathsf{CNF}}
\newcommand{\nnf}{\mathsf{NNF}}

\newcommand{\ddnnf}{\mathsf{d}\text{-}\mathsf{DNNF}}
\newcommand{\tcnf}{2\mathsf{CNF}}
\newcommand{\horn}{\mathsf{HORN}}

\newcommand{\pos}{\mathsf{Pos}}

\newcommand{\lel}{\preceq}
\newcommand{\lnel}{\prec}

\newcommand{\led}{\mathsf{LEH}}
\newcommand{\full}{\mathsf{Full}}
\newcommand{\fullset}{\textit{comp}}
\newcommand{\allpos}{\mathsf{AllPos}}
\newcommand{\allneg}{\mathsf{AllNeg}}

\newcommand{\meet}{\mathsf{GLB}}
\newcommand{\node}{\mathsf{Node}}

\newcommand{\comp}{\mathsf{Comp}}
\newcommand{\add}{\mathsf{Add}}
\newcommand{\leh}{\mathsf{LEH}}
\newcommand{\lu}{\mathsf{LU}}

\newcommand{\pred}{\mathsf{Pr}}
\newcommand{\minf}{\text{\rm min}}

\newcommand{\astruct}{\mathfrak{A}}

\newcommand{\EF}{{Ehrenfeucht-Fra\"\i ss\'e\ }}
\newcommand{\FO}{{\rm FO}} 	
\newcommand{\qr}{{\rm qr}}
\newcommand{\dom}{{\rm dom}}  

\newcommand{\foil}{\textsc{FOIL}}
\newcommand{\ffoil}{\textsc{ExplAIner}}

\newcommand{\optfoil}{\textsc{Opt-FOIL}}

\newcommand{\nf}{\mathsf{UF}}

\newcommand{\sem}{\mathsf{wDFS}}
\newcommand{\minsem}{\mathsf{DFS}}

\newcommand{\suf}{\mathsf{SUF}}

\newcommand{\sr}{\mathsf{wAXp}}
\newcommand{\rsr}{\mathsf{rwAXp}}
\newcommand{\minsr}{\mathsf{AXp}}
\newcommand{\msr}{\mathsf{mAXp}}
\newcommand{\nmsr}{\mathsf{nmAXp}}
\newcommand{\rmsr}{\mathsf{rmAXp}}

\newcommand{\axp}{\mathsf{AXp}}

\newcommand{\cx}{\mathsf{CXp}}
\newcommand{\wcx}{\mathsf{wCXp}}
\newcommand{\mcx}{\mathsf{mCXp}}

\newcommand{\mcr}{\mathsf{MCR}}
\newcommand{\mca}{\mathsf{MCA}}

\newcommand{\rfeat}{\mathsf{RF}}
\newcommand{\nfeat}{\mathsf{NF}}

\newcommand{\csr}{\mathsf{CAXp}}
\newcommand{\nsr}{\mathsf{SAXp}}

\newcommand{\sat}{\textsc{SAT}}
\newcommand{\unsat}{\textsc{UNSAT}}

\newcommand{\logicEvaluation}{\textsc{Eval}}
\newcommand{\logicComputation}{\textsc{Comp}}

\newcommand{\pat}{{\operatorname{pat}}}

\newcommand{\ecrowlen}{3.4pt}

\newtheorem{theorem}{Theorem}[section]
\newtheorem{proposition}[theorem]{Proposition}

\newtheorem{corollary}[theorem]{Corollary}

\newtheorem{lemma}[theorem]{Lemma}

\definecolor{codegreen}{rgb}{0,0.6,0}
\definecolor{codegray}{rgb}{0.5,0.5,0.5}
\definecolor{codepurple}{rgb}{0.58,0,0.82}
\definecolor{backcolour}{rgb}{0.95,0.95,0.92}

\lstdefinestyle{mystyle}{
    backgroundcolor=\color{backcolour},   
    commentstyle=\color{codegreen},
    keywordstyle=\color{magenta},
    numberstyle=\tiny\color{codegray},
    stringstyle=\color{codepurple},
    basicstyle=\ttfamily\footnotesize,
    breakatwhitespace=false,         
    breaklines=true,                 
    captionpos=b,                    
    keepspaces=true,                 
    showspaces=false,                
    showstringspaces=false,
    showtabs=false,                  
    tabsize=2
}

\tikzset{
    rt/.style={
		rectangle,
		fill = white,
		draw=black, 
		text centered,
		inner sep=0.5ex
		},
    rtt/.style={ 
    	rt,
    	inner sep=0.1ex
    	},
    ert/.style={ 
     	rt,
     	dashed
     	}, 
    ertt/.style={ 
        rtt,
        dashed
        }, 
    rect/.style={ 
        rectangle,
        fill = white,
        rounded corners,
        draw=black, 
        text centered,
        inner sep=0.8ex
        },
    rectw/.style={
        rect,
        draw=white
        },
    erect/.style={ 
    	rect,
    	dashed
    	},
    erectw/.style={ 
     	rectw,
     	dashed
     	},
    arrout/.style={
           ->,
           -latex,
           },
    arrin/.style={
           <-,
           latex-,
           },
    arrb/.style={
           <->,
           >=latex,
           }
}

\newcommand{\feat}[1]{\mbox{\footnotesize$ #1$}}

\tikzset{
	circ/.style={
		circle,
		draw=black, 
		thick,
		text centered,
	},
	circw/.style={
		circle,
		draw=white, 
		thick,
		text centered,
	},
	arrout/.style={
		->,
		-latex,
		thick,
	},
	arrin/.style={
		<-,
		latex-,
		thick,
	},
	arrw/.style={
		-,
		thick,
	},
	arrww/.style={
		-,
		thick,
		draw=white, 
	}
}

%% file: new-sections/new-intro.tex
\paragraph{Explainability as a query-language problem.}
The increasing use of machine learning (ML) models in decision-making systems has created a pressing need for principled methods to understand the predictions produced by such models. This need has led to a large body of work in explainable AI (XAI)~\cite{gunningDARPAExplainableArtificial2019,DBLP:journals/csur/GuidottiMRTGP19,DBLP:journals/inffus/ArrietaRSBTBGGM20,molnar2022}, and in particular to a variety of queries, scores, and explanation notions aimed at identifying why a model classifies a given input in a particular way~\cite{Marques-Silva_2023,Marques-Silva2024,darwiche2023logic}. Examples include abductive explanations, contrastive explanations, counterfactual-style queries, and feature-necessity or feature-relevance notions~\cite{DBLP:conf/aaai/IgnatievNM19,Darwiche_Hirth_2020,DBLP:conf/aaai/Ribeiro0G18,Huang_Cooper_Morgado_Planes_Marques-Silva_2023}.

From a data management perspective, this proliferation of explanation notions suggests a natural question: rather than designing a separate algorithm or formalism for each explanation task, can we develop a declarative language in which users specify what explanation they are looking for? This is in line with a long tradition in databases: complex computational tasks are exposed through query languages with well-defined syntax and semantics, while the study of their expressive power and evaluation complexity provides a principled understanding of what can be asked and how hard it is to answer~\cite{AbHV95,K90,Vard82,PY99}. In this view, a model becomes an object over which one poses queries, and explanation notions become fixed queries evaluated over that object~\cite{DBLP:conf/nips/ArenasBBPS21,arenas2024data-management-xai}.

This perspective has several advantages. First, it provides a uniform
framework for comparing and combining explanation notions. This is
important because there is no single explanation concept that is best
suited for all users, models, or applications; in many cases, the most
informative explanation is obtained by combining several
criteria~\cite{doshivelez2017rigorous,Marques-Silva_Ignatiev_2023}. Second,
it makes it possible to study explainability through standard
database-theoretic lenses, such as expressiveness and evaluation
complexity~\cite{AbHV95,Vard82,fmt-book}. Third, it opens the door to the
development of general optimization techniques for the operators of a
query language for explainability. Such techniques can reduce the
evaluation time of several explainability queries simultaneously, rather
than treating each query in isolation.

A central issue in such a framework is how to measure the complexity of
query evaluation. Since an explanation notion is intended to be specified
by a fixed formula of the language, the appropriate measure is
\emph{data complexity}: the query is fixed, while the input consists of
the model representation and the instance to be explained~\cite{Vard82}.
From this perspective, polynomial-time data complexity is desirable, but
it is not the only meaningful tractability target. Many explanation tasks
are inherently computationally demanding~\cite{NEURIPS2020_b1adda14,DBLP:journals/jair/WaldchenMHK21},
and therefore a useful explainability language should also allow
controlled forms of non-polynomial complexity. In particular, data
complexity in $\mathrm{P}^{\mathrm{NP}}$ remains a reasonable and useful
target: it corresponds to computation with a polynomial number of calls
to an NP oracle, and it is compatible with the use of SAT solvers as
evaluation engines. This complexity level is especially appropriate in
our setting because the inputs are model representations, not database
instances in the traditional sense. In contrast with large relational
databases, the tree-based models commonly considered in formal
explainability are often of moderate size, and SAT-based methods have
been successfully used to compute explanations for such models~\cite{DBLP:conf/sat/IgnatievS21,DBLP:conf/ijcai/Izza021,DBLP:conf/cp/YuISB20}.
Thus, in this paper we regard polynomial time and $\mathrm{P}^{\mathrm{NP}}$ as desirable data-complexity bounds for an explainability query~language.

\paragraph{Toward declarative languages for model interpretability.}
A first step in this direction was taken by Arenas et al.~\cite{DBLP:conf/nips/ArenasBBPS21}, who introduced FOIL, a first-order interpretability logic for querying ML models. FOIL is model-agnostic: it treats a model as a black box and provides access to its positive instances together with the natural subsumption relation over partial instances. This simple design makes FOIL an appealing foundational language, and it is expressive enough to capture several basic explanation notions.

However, model-agnosticism also has limitations. If the language is too weak, it cannot express explanation concepts that are central in practice. If it is too unconstrained, its evaluation problem may become too complex to support query evaluation in the sense expected from a database-oriented framework. Thus, the challenge is to design a language that balances two requirements. On the one hand, it should be expressive enough to capture a broad family of explanation queries, including optimality-based notions such as minimum or maximum explanations. On the other hand, it should have well-behaved evaluation and computation problems over relevant classes of Boolean models.

In this paper, we address this challenge by developing a declarative framework for explaining Boolean models. Our setting is not tied to decision trees. Instead, we consider models abstractly as Boolean functions, while also studying concrete representation classes such as deterministic and decomposable Boolean circuits and decision trees. This allows us to separate the logical specification of explanation queries from the representation-dependent complexity of evaluating them. Such a separation is particularly natural from a database perspective, where query specification and query evaluation over different representation classes are treated as distinct but connected problems.

\paragraph{The limitations of FOIL}
We begin by revisiting FOIL from the perspective of query-language design. We show that, despite its foundational role, FOIL does not satisfy the requirements above.
First, FOIL lacks the expressive power needed to capture some natural optimality-based explanation notions. In particular, we prove that minimum abductive explanations cannot be expressed in FOIL, even when the underlying model is restricted to be a decision tree. This shows that the limitation is not caused by the use of complex model classes, but by the expressive resources of the language itself.
Second, FOIL has high evaluation complexity. We prove that, for every level of the polynomial hierarchy, there is a fixed FOIL formula whose evaluation problem over decision trees is hard for that level. Thus, even on a class of models traditionally regarded as interpretable, unrestricted FOIL does not provide the kind of controlled data complexity that one would expect from a practical declarative language for explanations. This complements earlier complexity-theoretic approaches to model interpretability, which study the difficulty of answering explanation queries over different model classes~\cite{NEURIPS2020_b1adda14,DBLP:journals/jair/WaldchenMHK21}.

\paragraph{\ffoil: a tractable logic for explanation queries.}
Motivated by these limitations, we introduce \ffoil, a first-order logic designed to express explanation queries over Boolean models while retaining controlled evaluation complexity. The language extends the basic FOIL vocabulary with a relation that compares partial instances according to the number of defined features. 
This addition is essential for expressing optimality conditions based on cardinality, such as minimum abductive explanations~\cite{NEURIPS2020_b1adda14,Darwiche_Hirth_2020} and maximum contrastive explanations~\cite{DBLP:conf/aaai/IgnatievNM19,Huang_Cooper_Morgado_Planes_Marques-Silva_2023}.

\ffoil\ is organized in layers. Its atomic layer captures structural
properties of partial instances; its quantified layer allows formulas to
refer to the behavior of the model by combining formulas from the atomic
layer with the predicates $\allpos$ and $\allneg$, which express whether
all completions of a partial instance are classified positively or
negatively; and its topmost layer permits Boolean combinations of
explanation properties. This organization is designed to provide enough
expressive power for explanation tasks while keeping evaluation under
control.

We show that \ffoil\ can express the explanation notions studied in this paper, including weak abductive explanations, abductive explanations, and minimum abductive explanations~\cite{DBLP:conf/aaai/IgnatievNM19,Darwiche_Hirth_2020,NEURIPS2020_b1adda14,ijcai2022p91}; weak contrastive explanations, contrastive explanations, and maximum contrastive explanations~\cite{DBLP:conf/aaai/IgnatievNM19,NEURIPS2020_b1adda14,darwiche2023logic}; minimum change required and maximum change allowed~\cite{NEURIPS2020_b1adda14}; and necessary features and relevant features~\cite{Izza2021EfficientEW,Huang_Cooper_Morgado_Planes_Marques-Silva_2023}. At the same time, we prove that the evaluation problem for \ffoil\ belongs to the Boolean hierarchy over every class of models for which the basic $\allpos$ and $\allneg$ checks are tractable. This includes decision trees and also richer representation classes with suitable tractability properties, such as deterministic and decomposable Boolean circuits~\cite{DarwicheCompilation,DBLP:conf/aaai/ArenasBBM21}.

\paragraph{\optfoil: computing explanations.}
Evaluation is only one part of the problem. A language for explainability should also support the computation of explanations. For this reason, we introduce \optfoil, an optimization-oriented fragment built from the quantified layer of \ffoil\ together with a minimization operator over definable strict partial orders.

\optfoil\ captures explanation tasks in which one seeks an object satisfying a logical specification and minimal with respect to a user-defined preference order. This includes standard subset-minimal explanations, cardinality-minimum explanations, and distance-based notions such as minimum change required. By changing the order, the same formalism can also express maximality-based notions, such as maximum contrastive explanations and maximum change allowed.

Our main computational result for \optfoil\ is that its computation problem belongs to $\mathrm{FP}^{\mathrm{NP}}$ over every class of models for which $\allpos$ and $\allneg$ can be evaluated in polynomial time, where $\mathrm{FP}$ is the class of functions that can be computed in polynomial time. In the terminology above, this means that computing explanations specified in \optfoil\ has controlled data complexity: the formula is fixed, and the cost is measured as a function of the model representation and the input instance. This places \optfoil\ within the complexity regime identified above as suitable for declarative explainability languages, while allowing the language to capture optimization-based explanation tasks that are unlikely to admit polynomial-time algorithms in full generality.

\paragraph{Technical contributions.}
The following are the technical contributions of the paper.

\begin{itemize}
    \item We prove two limitations of FOIL over decision trees. On the
    expressiveness side, we show that no FOIL formula can define the
    minimum abductive explanation query. On the complexity side, we
    prove that for every level $\Sigma_k^P$ of the polynomial
    hierarchy, there is a fixed FOIL formula whose evaluation problem
    is $\Sigma_k^P$-hard.

    \item We introduce \ffoil, a logic based on \foil\ with an extended
    vocabulary and a layered structure: the atomic layer, the quantified
    layer, and the full \ffoil\ layer. The vocabulary of \ffoil\ consists
    of the predicates $\subseteq$, $\preceq$, $\allpos$, and $\allneg$,
    where $\subseteq$ is the subsumption relation on partial instances
    and $\preceq$ compares partial instances by their number of defined
    features. We show that both $\subseteq$ and $\preceq$ are necessary
    by proving that neither relation is first-order definable from the
    other over decision trees.

    \item We show that \ffoil\ is expressive enough to encode the
    explanation queries considered in the paper, including weak abductive
    explanations, subset-minimal abductive explanations,
    cardinality-minimum abductive explanations, weak and maximal
    contrastive explanations, minimum change required, maximum change
    allowed, necessary features, and relevant features. Each of these
    notions is expressed by a fixed query in \ffoil.

    \item We establish complexity bounds for the three layers
    of \ffoil.  First, we show that the evaluation problem for queries
    in the atomic layer can be solved in polynomial time over every
    class of Boolean models. Second, for every class of Boolean models
    over which the predicates $\allpos$ and $\allneg$ can be decided
    in polynomial time, we show that the evaluation problem for
    queries in the quantified layer is in $\np$.  Moreover, we show
    that there exists a query in the quantified layer whose evaluation
    problem is $\np$-complete over the class of decision trees. Third,
    for every class of Boolean models over which $\allpos$ and
    $\allneg$ can be decided in polynomial time, we show that the
    evaluation problem for \ffoil\ queries is in the Boolean
    hierarchy; equivalently, such queries can be evaluated with a
    fixed number of calls to an $\np$ oracle. Furthermore, we show
    that for every level $\bh_k$ of the Boolean hierarchy, there
    exists an \ffoil\ query whose evaluation problem is $\bh_k$-hard
    over the class of decision trees.

Importantly, the assumption that the predicates $\allpos$ and $\allneg$
can be decided in polynomial time is not specific to decision trees.
This condition also holds for more general classes of Boolean models,
including deterministic and decomposable Boolean circuits. Thus, the
upper bounds above apply beyond tree-based representations and cover
circuit classes that are central in knowledge compilation~\cite{DarwicheCompilation}.

    \item We define \optfoil\ as an optimization-oriented fragment of
    \ffoil, based on the quantified layer together with a minimization
    operator over strict partial orders. For every class of Boolean
    models over which the predicates $\allpos$ and $\allneg$ can be
    decided in polynomial time, we show that the computation problem for
    \optfoil\ queries is in $\mathrm{FP}^{\mathrm{NP}}$. Hence,
    explanations specified in \optfoil\ can be computed with a polynomial
    number of calls to an $\np$ oracle.

    \item As a result of independent interest, we use Presburger
    arithmetic to show that the problem of verifying whether a sentence
    over the atomic layer of \ffoil\ is valid is decidable. This result
    is needed to provide an effective syntax for \optfoil, since
    \optfoil\ requires verifying that a sentence in the atomic layer
    defines a strict partial order.

    \item Finally, we show that, under standard complexity-theoretic
    assumptions, \optfoil\ is strictly contained in \ffoil, and \ffoil\
    is strictly contained in \foil\ over the extended vocabulary
    consisting of the predicates $\subseteq$, $\preceq$, $\allpos$, and
    $\allneg$.
\end{itemize}

\paragraph{Organization of the paper.}
The remainder of the paper is organized as follows.
Section~\ref{sec:background} introduces the basic notions used
throughout the paper, including Boolean models, partial instances,
model representations, and the explanation queries studied in our
framework. Section~\ref{sec:foil-limitations} revisits $\foil$ and
establishes its limitations in terms of expressiveness and evaluation
complexity. Section~\ref{sec:dt-foil} introduces \ffoil\ and proves its
main expressiveness and evaluation results. Section~\ref{sec:opt-dt-foil}
presents \optfoil, our optimization-oriented language for computing
explanations, together with its computational guarantees.
Section~\ref{sec:conclusions} presents concluding remarks and
directions for future work. Finally, the appendix contains
supplementary material, including technical proofs that are deferred
for readability.

%% file: new-sections/background.tex
	
	We begin by introducing the main components of our framework, followed by a review of various explainability queries that will be addressed in the subsequent sections.

	\subsection{Models and instances}
	
	We use an abstract notion of a model of dimension $n$,
	and define it as a Boolean function $\M :  \{0,1\}^n \to \{0,
	1\}$.\footnote{We focus on Boolean models, as is common in 
	formal XAI research \citep{DBLP:journals/jair/WaldchenMHK21,audemard2021explanatory, 10497107}.}  
    We write $\dm(\M)$ for the dimension of a model
	$\M$. A {\em partial instance} of dimension $n$ is a tuple $\es \in
	\{0,1,\bot\}^n$, where $\bot$ is used to represent undefined features.
	We define $\es_\bot =  \{i \in \{1,\dots,n\} \mid \es[i] = \bot\}$.  
	An {\em instance} of dimension $n$ is a tuple $\es \in \{0,1\}^n$, that is, a partial instance without undefined features. For every instance $\es$ of dimension $\dm(\M)$, we write $\M(\es)$ for the value assigned by $\M$ to $\es$.
	
	Given two partial instances $\es$, $\es'$ of
	dimension $n$, we write $\es \subseteq \es'$ and say that $\es$ is {\em subsumed} by $\es'$ if and only if, for every $i \in \{1, \ldots, n\}$, whenever $\es[i] \neq \bot$, we have $\es[i] = \es'[i]$. In other words, $\es'$ can be obtained from $\es$ by replacing
	some occurrences of $\bot$ by Boolean values. For example, $(1,\bot)$ is subsumed by $(1,0)$, but it is not subsumed by $(0,0)$.
	A partial instance $\es$ can be
	seen as a compact representation of the set of instances $\es'$ such that $\es$ is subsumed by $\es'$. Such instances $\es'$ are called the {\em completions} of $\es$ and the set of all of them is denoted by $\fullset(\es)$. 
	
	For each $n$, partial instances of dimension $n$ are partitioned into $n+1$ levels: for each $i \in \{0, \ldots, n\}$, level $i$ consists of all partial instances with exactly $i$ defined features. Given partial instances $\es$, $\es'$, we write $\es \lel \es'$ and say that $\es$ is on a \emph{less or equal level} than $\es'$ if and only if $|\es_\bot| \geq |\es'_\bot|$. In other words, $\es'$ has at least as many defined features as $\es$. 

	We will also use the symbols $\subset$ and $\lnel$ for the corresponding strict relations.
	
	In several proofs, we write $\es \cdot \es'$ for the concatenation of both instances. Moreover, for a value $s \in \{0,1,\bot\}$, we denote by $\{s\}^{n}$ the partial instance of dimension $n$ whose entries are all equal to $s$.
	
	We next introduce several classes of Boolean functions that will be used throughout the paper.
	
	\paragraph{Boolean Circuits.} A \emph{Boolean circuit} of dimension $n$ is a directed acyclic graph over a set of variables~$\{x_1, ..., x_n\}$ such that: 
	\begin{enumerate}
		\item[(i)] Every node without incoming edges is either a {\em variable gate} or a {\em constant gate}. A variable gate is labeled with a
		variable, and a constant gate is labeled with either~$0$ or~$1$;
		
		\item[(ii)] Every node with incoming edges is a {\em
		logic gate}, and is labeled with a symbol~$\land$,~$\lor$ or~$\lnot$. If it is labeled with the symbol~$\lnot$, then it has exactly one incoming edge;
		
		\item[(iii)] Exactly one node does not have any outgoing edges, and
		this node is called the {\em output gate}.
	\end{enumerate}

	Given a Boolean circuit $C$ and an instance $\es$ of dimension $n$, the value $C(\es)$ is defined as the value of the output gate of~$C$ when we evaluate~$C$ on input~$\es$. Note that we are identifying inputs of the circuit as instances of the Boolean model.
	
	Several restrictions of Boolean circuits with good computational properties have been studied.

\paragraph{Negation Normal Form.} A \emph{negation normal form} ($\nnf$) circuit of dimension $n$ is a Boolean circuit of dimension $n$ such that the incoming edge of every negation gate comes from a variable gate.

\paragraph{Determinism and decomposability.} Let $X$ be a set of variables, let $C$ be a circuit over $X$, and let ~$g$ be a gate of~$C$. We define $C_g$ to be the Boolean circuit over~$X$ induced by
the set of gates~$g'$ of~$C$ for which there exists a directed path from~$g'$ to~$g$ in~$C$. Note that~$g$ is the output gate of~$C_g$. An~$\lor$-gate~$g$ of~$C$ is said to be
\emph{deterministic} if, for every pair~$g_1$,~$g_2$ of distinct input gates of~$g$, there is no instance~$\es$ such that~$C_{g_1}(\es) = C_{g_2}(\es) = 1$. The circuit~$C$ is called \emph{deterministic} if every~$\lor$-gate of~$C$ is deterministic. For every gate $g$ of $C$, define $\var(g)$ as the set of variables~$x \in X$ such that there exists a variable gate labeled by~$x$ in~$C_g$.  An~$\land$-gate~$g$ of~$C$ is said to be \emph{decomposable} if for every pair~$g_1$,~$g_2$ of distinct input gates of~$g$, we have~$\var(g_1) \cap \var(g_2) = \emptyset$. The circuit~$C$ is called \emph{decomposable} if every
$\land$-gate of~$C$ is decomposable.

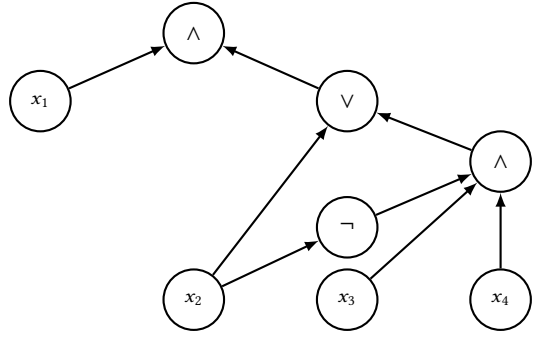
\begin{figure}[ht]
	\begin{center}
		\begin{center}
			\begin{tikzpicture}
				\node[circ, minimum size=8mm, inner  sep=-2] (n1) {\feat{x_2}};
				\node[circ, right=12mm of n1, minimum size=8mm] (n2) {\feat{x_3}};
				\node[circ, above=1.3mm of n2, minimum size=8mm] (nneg) {$\neg$}
				edge[arrin] (n1);
				\node[circ, right=12mm of n2, minimum size=8mm] (n3) {\feat{x_4}};
				\node[circ, above=10mm of n3, minimum size=8mm] (n4) {$\land$}
				edge[arrin] (nneg)
				edge[arrin] (n2)
				edge[arrin] (n3);
				\node[circ, above=18mm of n2, minimum size=8mm] (n5) {$\lor$}
				edge[arrin] (n4)
				edge[arrin] (n1);
				\node[circ, above=27mm of n1, minimum size=8mm] (n6) {$\land$}
				edge[arrin] (n5);
				\node[circw, left=12mm of n5, minimum size=8mm] (n5a) {};
				\node[circ, left=12mm of n5a, minimum size=8mm, inner sep=-2] (n0) {\feat{x_1}}
				edge[arrout] (n6);
			\end{tikzpicture}
		\end{center}
		\caption{A $\ddnnf$ circuit of dimension 4. \label{fig:ddnnf}}
		\Description[Example of a four-variable d-DNNF circuit.]
		{The figure shows a four-variable circuit. The output is an and gate with two inputs: variable x1 and an or gate. The or gate has two inputs. One input is variable x2. The other is an and gate whose three inputs are not x2, x3, and x4.}
	\end{center}
\end{figure}

\paragraph{Deterministic Decomposable Negation Normal Form.} A \emph{deterministic decomposable negation normal form} ($\ddnnf$) circuit of dimension $n$ is an $\nnf$ circuit of dimension $n$ that is both deterministic and decomposable. An example is shown in \autoref{fig:ddnnf}.

\paragraph{Binary Decision Diagram.} A \emph{binary decision diagram} ($\bdd$) of dimension $n$ is a directed acyclic graph with a unique root, and whose nodes and edges are labeled as follows: (i) every leaf is labeled by $\true$ or $\false$ and (ii) every non-leaf node is labeled by a feature in $\{1, ..., n\}$ and has exactly two outgoing edges, one labeled by $0$ and the other by $1$.

Let $B$ be a binary decision diagram and let $\es$ be an instance of dimension $n$. The value $B(\es)$ is defined as the Boolean value of the leaf obtained by starting at the root and following the path such that, at each non-leaf node labeled by $i$, the outgoing edge labeled by $\es[i]$ is chosen. 

\paragraph{Decision Trees.}  A \emph{decision tree} ($\dt$) over instances of dimension $n$ is a binary decision diagram of dimension $n$ such that (i) its underlying graph is a tree and (ii) no feature appears more than once on any root-to-leaf path. 
An example is shown in \autoref{fig:dtree}.

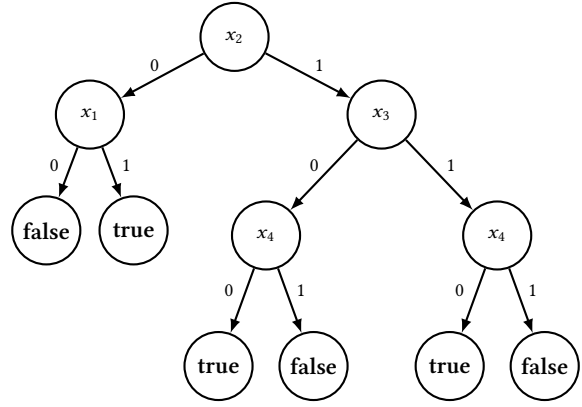
\begin{figure}[ht]
	\begin{center}
		\begin{center}
			\begin{tikzpicture}
				\node[circw, minimum size=9mm, inner  sep=-4] (n1) {};
				\node[circw, right=6mm of n1, minimum size=9mm, inner  sep=-4] (n2) {};
				
				\node[circ, left=-3mm of n1, minimum size=9mm, inner  sep=-2] (n10) {$\true$};
				\node[circ, right=-3mm of n1, minimum size=9mm, inner  sep=-2] (n11) {$\false$};

				\node[circw, right=6mm of n2, minimum size=9mm, inner  sep=-4] (n3) {};
				\node[circ, left=-3mm of n3, minimum size=9mm, inner  sep=-2] (n12) {$\true$};
				\node[circ, right=-3mm of n3, minimum size=9mm, inner  sep=-2] (n13) {$\false$};
				\node[circ, above=8mm of n3, minimum size=9mm] (n4) {\feat{x_4}}
				edge[arrout] (n12)
				edge[arrout] (n13);

				\node[circ, above=8mm of n1, minimum size=9mm] (n7) {\feat{x_4}}
				edge[arrout] (n10)
				edge[arrout] (n11);
				\node[circ, above=24mm of n2, minimum size=9mm] (n5) {\feat{x_3}}
				edge[arrout] (n7)
				edge[arrout] (n4);
				\node[circ, above left=37mm and -2.4mm of n1, minimum size=9mm] (n6) {\feat{x_2}}
				edge[arrout] (n5);
				\node[circw, left=10mm of n5, minimum size=9mm] (n5a) {};
				\node[circ, left=10mm of n5a, minimum size=9mm, inner sep=-2] (n0) {\feat{x_1}}
				edge[arrin] (n6);
				
				\node[circw, below=8.56mm of n0, minimum size=2mm, inner sep=-4] (n14) {};
				\node[circ, left=-1mm of n14, minimum size=9mm, inner  sep=-2] (n15) {$\false$}
				edge[arrin] (n0);
				\node[circ, right=-1mm of n14, minimum size=9mm, inner  sep=-2] (n16) {$\true$}
				edge[arrin] (n0);
				\node[text width=0.2cm] at (-0.45,1.0) {\feat{0}};
				\node[text width=0.2cm] at (0.51,1.0) {\feat{1}};
				\node[text width=0.2cm] at (2.6,1.0) {\feat{0}};
				\node[text width=0.2cm] at (3.57,1.0) {\feat{1}};
				\node[text width=0.2cm] at (-2.76,2.66) {\feat{0}};
				\node[text width=0.2cm] at (-1.80,2.66) {\feat{1}};
				\node[text width=0.2cm] at (0.68,2.66) {\feat{0}};
				\node[text width=0.2cm] at (2.49,2.66) {\feat{1}};
				\node[text width=0.2cm] at (-1.43,3.98) {\feat{0}};
				\node[text width=0.2cm] at (0.75,3.98) {\feat{1}};
			\end{tikzpicture}
		\end{center}
		\caption{A decision tree of dimension 4. For the sake of readability, we label nodes by $x_i$ rather than by indices. \label{fig:dtree}}
		\Description[Example of a decision tree of dimension 4.]
		{The figure shows a decision tree with root x2. If x2 equals 0, the tree moves to x1. From x1, edge 0 leads to a false leaf and edge 1 leads to a true leaf. If x2 equals 1, the tree moves to x3. Both outgoing edges of x3 lead to a node labeled x4. In each x4 node, edge 0 leads to a true leaf and edge 1 leads to a false leaf.}
	\end{center}
\end{figure}

	\subsection{Explainability queries}\label{sec:queries}
	
	We now define the explainability queries studied in this work.
	
	\paragraph{Weak Abductive Explanation.} Given an instance $\es$ and a model $\M$, a partial instance $\es_1$ is a \emph{weak abductive explanation} ($\sr$) for $\es$ on $\M$ if $\es_1 \subseteq \es$ and, for every $\es_2 \in \fullset(\es_1)$, the condition $\M(\es)=\M(\es_2)$ holds \citep{Huang_Cooper_Morgado_Planes_Marques-Silva_2023}. This notion is also known as \emph{sufficient reason} in the literature \citep{arenas2024uniform}. For example, in \autoref{fig:ddnnf}, $(1,1,1,\bot)$ is a weak abductive explanation for the instance $(1,1,1,1)$.
	
	\paragraph{Abductive Explanation.} Given a pair $(\es,\M)$, a partial instance $\es_1$ is an \emph{abductive explanation} ($\minsr$) for $\es$ on $\M$ if it is a weak abductive explanation for $\es$ on $\M$ and there is no weak abductive explanation $\es_2$ such that $\es_2 \subset \es_1$ \citep{DBLP:conf/aaai/IgnatievNM19}. This notion of explanation has been extensively studied and it can be found in the literature under names such as \emph{sufficient reason} \citep{inbook2019,Darwiche_Hirth_2020}, \emph{prime implicant} \citep{shih2018symbolic}, and \emph{minimal sufficient reason} \citep{arenas2024uniform}. Modeling abduction using propositional logic or first-order logic, and the complexity of computing such an explanation has been studied for many decades \citep{Marquis91}. In \autoref{fig:ddnnf}, $(1,1,1,\bot)$ is a weak abductive explanation for $(1,1,1,1)$, but not an abductive explanation. By contrast, one can check that $(1,1,\bot,\bot)$ is indeed an abductive explanation.

	\paragraph{Minimum Abductive Explanation.} Given a pair $(\es,\M)$, a partial instance $\es_1$ is a \emph{minimum abductive explanation} ($\msr$) for $\es$ on $\M$ if it is a weak abductive explanation for $\es$ on $\M$ and there is no weak abductive explanation $\es_2$ such that $\es_2 \lnel \es_1$. Our definition is based on the \emph{minimum sufficient reason} explainability query by \cite{NEURIPS2020_b1adda14} and \cite{arenas2024uniform}. In \autoref{fig:ddnnf}, $(1,1,\bot,\bot)$ is also a minimum abductive explanation for $(1,1,1,1)$.
	
	\paragraph{Weak Contrastive Explanation.} Given a pair $(\es,\M)$, the partial instance $\es_1$ is a \emph{weak contrastive explanation} ($\wcx$) for $\es$ on $\M$ if $\es_1 \subseteq \es$ and there is an instance $\es_2 \in \fullset(\es_1)$ such that the condition $\M(\es)\not =\M(\es_2)$ holds \cite{Marques-Silva2024}. 

	\paragraph{Contrastive Explanation.} Given a pair $(\es,\M)$, the partial instance $\es_1$ is a \emph{contrastive explanation} ($\cx$) for $\es$ on $\M$ if it is a weak contrastive explanation for $\es$ on $\M$ such that there is no weak contrastive explanation $\es_2$ satisfying $\es_1 \subset \es_2$. It can be shown that this definition is equivalent to the one by \cite{Marques-Silva2024}.
	
	\paragraph{Maximum Contrastive Explanation.} Given a pair $(\es,\M)$, the partial instance $\es_1$ is a \emph{maximum contrastive explanation} ($\mcx$) for $\es$ on $\M$ if it is a weak contrastive explanation for $\es$ on $\M$ such that there is no weak contrastive explanation $\es_2$ satisfying $\es_1 \lnel \es_2$. In \autoref{fig:ddnnf}, the partial instance $(\bot,1,1,1)$ is a maximum contrastive explanation for $(1,1,1,1)$, thus also a $\cx$ and a $\wcx$.

	\paragraph{Minimum Change Required.} Given a pair $(\es,\M)$, an instance $\es_1$ is a solution to the \emph{minimum change required} query ($\mcr$) for $\es$ on $\M$ if $\M(\es) \not = \M(\es_1)$ and $\M(\es) = \M(\es_2)$, for every instance $\es_2$ at smaller Hamming distance (meaning the number of flipped features between two instances) from $\es$ than $\es_1$. The instance $\es_1$ represents the minimum distance required to change the value on the model. This notion is based on the query introduced by \cite{NEURIPS2020_b1adda14}. Considering the instance $(1,1,1,1)$ in \autoref{fig:ddnnf}, one possible explanation for the query is $(0,1,1,1)$.
	
	\paragraph{Maximum Change Allowed.} Given a pair $(\es,\M)$, an instance $\es_1$ is a solution to the \emph{maximum change allowed} query ($\mca$) for $\es$ on $\M$ if $\M(\es) = \M(\es_1)$ and $\M(\es) \not = \M(\es_2)$, for every instance $\es_2$ at greater Hamming distance from $\es$ than $\es_1$. This notion is based on the query introduced by \cite{alfano2024evenif}. As the authors argue, only studying counterfactual queries like minimum change required may not capture the whole picture for explaining certain situations. That is why we include their semifactual version of the problem. In \autoref{fig:ddnnf}, the maximum change allowed for the negative input $(0,1,1,1)$ is the instance $(0,0,0,0)$.
	
	The original versions of $\mcr$ and $\mca$ \citep{NEURIPS2020_b1adda14, alfano2024evenif} are defined similarly with respect to each other, but have very different interpretations. Given a distance $k$ for $\mcr$ it is not trivially easier to decide any distance $k'$ for $\mca$; and vice versa. We only know that the inequality $k \leq k' + 1$ holds. Thus, presenting both queries has additional value and lets us present another useful case of maximization.
	
	\paragraph{Necessary Feature.} Given a pair $(\es,\M)$, a partial instance $\es_1$ with exactly one defined feature is a \emph{necessary feature} ($\nfeat$) for $\es$ on $\M$ if, for every weak abductive explanation $\es_2$ for $\es$ on $\M$, $\es_1 \subseteq \es_2$ holds. An equivalent formulation appears in \cite{Huang_Cooper_Morgado_Planes_Marques-Silva_2023}, where the authors define the property \emph{feature necessity} on their own framework using abductive explanations instead of weak abductive explanations. It is easy to see that both definitions are equivalent. Considering $(1,1,1,1)$ in \autoref{fig:ddnnf}, the instance has two abductive explanations: $(1,1,\bot,\bot)$ and $(1,\bot,1,1)$. Therefore, the feature $x_1=1$, represented by the partial instance $(1,\bot,\bot,\bot)$, is a necessary feature.
	
	\paragraph{Relevant Feature.} Given a pair $(\es,\M)$, a partial instance $\es_1$ with exactly one defined feature is a \emph{relevant feature} ($\rfeat$) for $\es$ on $\M$ if there exists an abductive explanation $\es_2$ for $\es$ on $\M$ such that $\es_1 \subseteq \es_2$. This notion appears in the literature as the \emph{AXp membership problem} \cite{DBLP:conf/kr/HuangII021} and as \emph{feature relevancy} \cite{Huang_Cooper_Morgado_Planes_Marques-Silva_2023}. Considering $(1,1,1,1)$ and its abductive explanations $(1,1,\bot,\bot)$ and $(1,\bot,1,1)$ in \autoref{fig:ddnnf}, $x_2=1$ is one of the relevant features, and can be represented by the partial instance $(\bot,1,\bot,\bot)$.

%% file: new-sections/foil.tex
In this section, we introduce an initial interpretability logic designed for expressing queries. We demonstrate that it faces limitations in expressive power and exhibits high computational complexity for its evaluation problem.


Our work is inspired by the {\em first-order interpretability logic} ($\foil$)
\citep{DBLP:conf/nips/ArenasBBPS21}, which is a simple explainability language rooted in first-order logic. $\foil$ is simply first-order
logic over two relations on the set of partial instances of a given
dimension: a unary relation $\pos$ whose interpretation is the set of instances that are positively classified by the model, and a binary relation $\subseteq$ that represents
the subsumption relation among partial instances.

Given a vocabulary $\sigma$ consisting of relations $R_1$, $\ldots$, $R_\ell$, recall that a structure $\astruct$ over $\sigma$ consists of a domain over which quantifiers range, and an interpretation for each relation $R_i$. Moreover, given a first-order formula $\varphi$ defined over the vocabulary $\sigma$, we write $\varphi(x_1, \ldots, x_k)$ to indicate that the free variables of $\varphi$ are among $\{x_1,\ldots,x_k\}$. Finally, given a structure $\astruct$ over the vocabulary $\sigma$ and elements $a_1$, $\ldots$, $a_k$ in the domain of $\astruct$, we use $\astruct \models \varphi(a_1, \ldots, a_k)$ to indicate that the formula $\varphi$ is satisfied by $\astruct$ when each variable $x_i$ is replaced by element $a_i$ ($1 \leq i \leq k$).


Consider a model $\M$ with $\dm(\M) = n$. The structure $\astruct_\M$ representing
$\M$ over the vocabulary formed by $\pos$ and $\subseteq$ is defined as follows.
The domain of $\astruct_\M$ is the set $\{0,1, \bot\}^n$ of all
partial instances of dimension $n$.
A partial instance $\es \in \{0,1, \bot\}^n$ belongs to the interpretation of
$\pos$ in $\astruct_\M$ if and only if $\es \in \{0,1\}^n$ and $\M(\es) = 1$.
Moreover, a pair $(\es_1,\es_2)$ is in the interpretation of
relation $\subseteq$ in $\astruct_\M$ if and only if $\es_1$ is
subsumed by $\es_2$.
%
Finally, given a formula $\varphi(x_1, \ldots, x_k)$ in $\foil$ and
partial instances $\es_1$, $\ldots$, $\es_k$ of dimension $n$, the model
$\M$ is said to {\em satisfy} $\varphi(\es_1, \ldots, \es_k)$, denoted
by $\M \models \varphi(\es_1, \ldots, \es_k)$, if
$\astruct_\M \models \varphi(\es_1, \ldots, \es_k)$.

Notice that for a succinctly-represented model $\M$, the structure $\astruct_\M$ can be exponentially larger than the representation of $\M$. Hence, $\astruct_\M$ is a theoretical construction needed to formally define the semantics of $\foil$, but it should not be constructed explicitly when checking in practice if a formula $\varphi$ is satisfied by $\M$.

\subsection{Expressing interpretability queries in \foil}\label{sec:expressing-in-foil}

It will be instructive for the rest of our presentation to see a few examples of how $\foil$ can be used to express some natural explainability queries on models.
In these examples, we make use of the following $\foil$ formula:
$$\full(x) \ := \ \forall y \, (x \subseteq y \, \rightarrow \, y \subseteq x).$$%
Notice that if $\M$ is a model and $\es$ is a partial instance, then
$\M \models \full(\es)$ if and only if $\es$ is also an instance (i.e., it has no undefined features).
We also use the formula
$$\allpos(x) \ := \ \forall y \, \big((x \subseteq y \wedge \full(y)) \, \rightarrow \, \pos(y)\big),$$
such that $\M \models \allpos(\es)$ if and only if every instance in $\fullset(\es)$ is classified as positive by $\M$. Analogously, we define a formula $\allneg(x)$. The definitions of both predicates are inspired by important knowledge compilation queries such as a conditioning transformation of the input and a consistency or validity check on the conditioned model \citep{DarwicheCompilation}. $\allpos(x)$ and $\allneg(x)$ will be important components of a new logic for explainability defined in \Cref{sec:dt-foil} since they enable the expression of a wide range of explainability queries. For example, we can now define weak abductive explanations (refer to \Cref{sec:queries}) in $\foil$ as follows:
$$
\sr(x,y) := \full(x) \wedge y \subseteq x \ \wedge \\  (\pos(x) \to \allpos(y)) \wedge (\neg \pos(x) \to \allneg(y)).
$$

In fact, it is easy to see that $\M \models \sr(\es,\es')$ if and only if $\es'$ is a weak abductive explanation for $\es$ over $\M$.
Notice that $\es$ is always a weak abductive explanation for itself. However, we are typically interested in explanations that satisfy some optimality criterion. A common such criterion is that of being {\em minimal} \citep{shih2018symbolic,DBLP:journals/corr/abs-2010-11034,NEURIPS2020_b1adda14}.
Let us write $x \subset y$ for $x \subseteq y \wedge \neg (y \subseteq x)$. Then, for
\[
\minsr(x,y) :=  \sr(x,y) \land
\forall z \,(z \subset y \, \rightarrow \neg \sr(x,z)),
\]%
we have that $\M \models \minsr(\es,\es')$ if and only if $\es'$ is an abductive explanation for $\es$ over $\M$. We could similarly express local explainability queries like contrastive explanations through the same approach using $\foil$.

As we show below, $\foil$ fails to meet either of the two criteria we are looking for in a practical language that provides explanations.  The first issue is its limited expressiveness: there are important notions of explanations that cannot be expressed in this language, even when restricted to decision trees, which are traditionally deemed to be easily interpretable.
The second issue is its high computational complexity: there are queries in $\foil$ that cannot be evaluated with a polynomial number of calls to an NP oracle. Both facts firmly establish the inadequacy of $\foil$ as a practical language.

\subsection{\foil~presents limited expressiveness}
\label{sec-lim-foil}
In some scenarios we want to express a stronger condition for abductive and contrastive explanations: not only that they are minimal, but also that they are {\em minimum} (see \Cref{sec:queries}).
In the case of abductive explanations, they can be minimal without being minimum. The following theorem shows that $\foil$ cannot express the query that asks whether a partial instance $\es'$ is a minimum abductive explanation for a given instance $\es$ over decision trees.


	\begin{theorem}\label{thm:ne-foil}
		There is no formula $\msr(x,y)$ in \textnormal{\foil} such that, for every decision tree $\T$, instance $\es$ and partial instance $\es'$, we have that
		$\T \models \msr(\es,\es') \Leftrightarrow \es'$ is a minimum abductive explanation for $\es$ over $\T$.
	\end{theorem}

	\begin{proof}
		The proof extends techniques from \citep{fmt-book} such as the games for \FO\ distinguishability. We now present notions that will be used in this and the following arguments throughout this paper.

		The {\em quantifier rank} of an $\FO$ formula $\varphi$, denoted by $\qr(\varphi)$,
		is the maximum depth of
		quantifier nesting in it. For a structure $\astruct$, we write $\dom(\astruct)$ to denote its domain. An {\em \EF} (EF) game is played in two structures, $\astruct_1$ and $\astruct_2$,
		of the same schema, by two players, the {\em spoiler} and
		the {\em duplicator}. In round $i$ the spoiler selects a structure,
		say $\astruct_1$, and an element $c_i$ in $\dom(\astruct_1)$; the duplicator responds
		by selecting an element $e_i$ in $\dom(\astruct_2)$. The duplicator {\em wins} in
		$k$ rounds, for $k \geq 0$, if $\{(c_i,e_i) \mid i \leq k\}$ defines a
		partial isomorphism between $\astruct_1$ and $\astruct_2$. If the duplicator wins no
		matter how the spoiler plays, we write $\astruct_1 \equiv_k \astruct_2$. A classical
		result states that $\astruct_1 \equiv_k \astruct_2$ iff $\astruct_1$ and $\astruct_2$ agree on all
		$\FO$ sentences of quantifier rank $\leq k$ (cf. \citep{fmt-book}).

		Also, if $\bar a$ is an $m$-tuple in $\dom(\astruct_1)$ and $\bar b$ is an $m$-tuple in
		$\dom(\astruct_2)$, where $m \geq 0$,
		we write $(\astruct_1,\bar a) \equiv_k (\astruct_2,\bar b)$ whenever the duplicator
		wins in $k$ rounds no matter how the spoiler plays, but
		starting from position $(\bar a,\bar b)$. In the same way, $(\astruct_1,\bar a) \equiv_k (\astruct_2,\bar b)$ iff for every $\FO$ formula $\varphi(\bar x)$
		of quantifier rank $\leq k$, it holds that $\astruct_1 \models \varphi(\bar a) \Leftrightarrow \astruct_2
		\models \varphi(\bar b)$.

		It is well-known (cf.~\citep{fmt-book}) that
		there are only finitely many $\FO$ formulae of quantifier rank
		$k$, up to logical equivalence. The {\em rank-$k$ type}
		of an $m$-tuple $\bar a$ in a structure $\astruct$ is the set of all formulae
		$\varphi(\bar x)$ of quantifier rank $\leq k$ such that $\astruct \models
		\varphi(\bar a)$. Given the above, there are only finitely many rank-$k$
		types, and each one of them is definable by an $\FO$ formula
		$\tau_k^{(\astruct,\bar a)}(\bar x)$ of
		quantifier rank $k$.

		We now introduce some terminology necessary for the proof.

		Let $\M$, $\M'$ be models of dimension $n$ and $p$, respectively,
		and consider the structures $\astruct_\M =  \langle \{0,1,\bot\}^n,\subseteq^{\astruct_\M},\pos^{\astruct_\M} \rangle$ and
		$\astruct_{\M'} =  \langle \{0,1,\bot\}^p,\subseteq^{\astruct_{\M'}},\pos^{\astruct_{\M'}} \rangle$. We write $\astruct_\M \oplus \astruct_{\M'}$ for the structure over the same vocabulary that satisfies the following:
		\begin{itemize}
			\item The domain of $\astruct_\M \oplus \astruct_{\M'}$ is $\{0,1,\bot\}^{n+p}$.
			\item The interpretation of $\subseteq$ on $\astruct_\M \oplus \astruct_{\M'}$ is the usual subsumption relation on $\{0,1,\bot\}^{n+p}$.
			\item The interpretation of $\pos$ on $\astruct_\M \oplus \astruct_{\M'}$ is the set of instances $\es \in \{0,1\}^{n+p}$ such that
			$(\es[1],\cdots,\es[n]) \in \pos^{\astruct_\M}$ or $(\es[n+1],\cdots,\es[n+p]) \in \pos^{\astruct_{\M'}}$.
		\end{itemize}

		We will also consider structures of the form $\astruct_n = \langle \{0,1,\bot\}^n,\subseteq^\astruct \rangle$, where $\subseteq$ is interpreted as the subsumption relation
		over $\{0,1,\bot\}^n$. For any such structure, we
		write $\astruct_n^+$ for the structure over the vocabulary $\{\subseteq,\pos\}$ that extends $\astruct$ by adding only the tuple $\{1\}^n$
		to the interpretation of $\pos$.

		We now state two crucial lemmas, whose proofs can be found in the appendix of this work (refer to Sections \ref{proof:pointed} and \ref{proof:comp}).



		\begin{lemma}
			\label{lemma:pointed}
			If $n,p \geq 3^k$, then $(\astruct_n,\{1\}^n) \ \equiv_k \  (\astruct_p,\{1\}^p)$. In particular, $\astruct_n^+ \equiv_k \astruct_p^+$.
		\end{lemma}



		\begin{lemma}
			\label{lemma:comp}
			Consider models $\M$, $\M_1$, and $\M_2$ of dimension $n$, $p$, and $q$, respectively,
			and assume that $(\astruct_{\M_1},\{1\}^p) \equiv_k (\astruct_{\M_2},\{1\}^q)$. Then it is the case that
			\begin{align*}
				\big(\astruct_\M \oplus \astruct_{\M_1},\{1\}^{n+p},\{\bot\}^n \cdot \{1\}^p\big) \ \equiv_k \
				\big(\astruct_\M \oplus \astruct_{\M_2},\{1\}^{n+q},\{\bot\}^n \cdot \{1\}^q\big).
			\end{align*}
		\end{lemma}

		We now proceed with the proof of Theorem \ref{thm:ne-foil}. Assume, for the sake of contradiction, that there is in fact a formula
		$\msr(x,y)$ in $\foil$ such that, for every decision tree $\M$, instance $\es$, and partial instance $\es'$, we have that
		$\astruct_\M \models \msr(\es,\es')$ iff $\es'$ is a minimum abductive explanation for $\es$ over $\M$. Let $k \geq 0$ be the quantifier rank of this formula.
		We show that there exist decision trees $\M_1$ and $\M_2$, instances $\es_1$ and $\es_2$ over $\M_1$ and $\M_2$, respectively,
		and partial instances $\es'_1$ and $\es'_2$ over $\M_1$ and $\M_2$, respectively, for which the following holds:
		\begin{itemize}
			\item $(\M_1,\es_1,\es'_1) \equiv_k (\M_2,\es_2,\es'_2)$, and hence
			\begin{align*}
				\M_1 \models \msr(\es_1,\es'_1) \ \Leftrightarrow \
				\M_2 \models \msr(\es_2,\es'_2).
			\end{align*}
			\item It is the case that $\es'_1$ is a minimum abductive explanation for $\es_1$ under $\M_1$, but
			$\es'_2$ is not a minimum abductive explanation for $\es_2$ under $\M_2$.
		\end{itemize}
		This is our desired contradiction.

		Let $\M_{n,p}$ be a decision tree of dimension $n+p$ such that, for every instance $\es \in \{0,1\}^{n+p}$, we have that $\M_{n,p}(\es) = 1$ iff
		$\es$ is of the form $\{1\}^n \cdot \{0,1\}^p$, i.e., the first $n$ features of $\es$ are set to 1, or $\es$ is
		of the form $\{0,1\}^n \cdot \{1\}^p$, i.e., the last $p$ features of $\es$ are set to 1.
		Take the instance $\es = \{1\}^{n+p}$. It is easy to see that $\es$ only has two abductive explanations in $\M_{n,p}$; namely, $\es_1 = \{1\}^n \cdot \{\bot\}^p$
		and $\es_2 = \{\bot\}^n \cdot \{1\}^p$.

		We define the following:
		\begin{itemize}
			\item $\M_1 := \M_{2^{k},2^{k}}$ and $\M_2 := \M_{2^k,2^k+1}$.
			\item $\es_1 := \{1\}^{2^k + 2^k}$ and $\es_2 := \{1\}^{2^k+2^k+1}$.
			\item $\es'_1 :=  \{\bot\}^{2^k} \cdot \{1\}^{2^k}$ and $\es'_2 :=  \{\bot\}^{2^k} \cdot \{1\}^{2^k+1}$.
		\end{itemize}
		From our previous observation, $\es'_1$ is an abductive explanation for $\es_1$ over $\M_1$
		and $\es'_2$ is an abductive explanation for $\es_2$ over $\M_2$.

		We show first that $(\astruct_{\M_1},\es_1,\es'_1) \equiv_k (\astruct_{\M_2},\es_2,\es'_2)$.  It can be observed that $\astruct_{\M_1}$ is of the form $\astruct_{N} \oplus \astruct_{N_1}$, where $N$ is a model of dimension $2^k$ that only accepts the tuple $\{1\}^{2^k}$ and the same holds for $N_1$. Analogously, $\astruct_{\M_2}$ is of the form $\astruct_{N} \oplus \astruct_{N_2}$, where $N_2$ is a model of dimension $2^k + 1$ that only accepts the tuple $\{1\}^{2^k + 1}$.
		From Lemma \ref{lemma:pointed}, we have that
		$$(\astruct_{N_1},\{1\}^{2^k}) \ \equiv_k \
		(\astruct_{N_2},\{1\}^{2^k+1}).$$
		Notice that indeed any winning strategy for the Duplicator on this game must map
		the tuples $\{1\}^{2^k}$ in $\astruct_{N_1}$ and $\{1\}^{2^k+1}$ into each other.

		Now, from Lemma \ref{lemma:comp}, we obtain that
		\begin{align*}
			\big(\astruct_N \oplus \astruct_{N_1},\{1\}^{2^k+2^k},\{\bot\}^{2^k} \cdot \{1\}^{2^k}\big) \ \equiv_k \
			\big(\astruct_N \oplus \astruct_{N_2},\{1\}^{2^k+2^k+1},\{\bot\}^{2^k} \cdot \{1\}^{2^k+1}\big).
		\end{align*}
		We can then conclude that $(\astruct_{\M_1},\es_1,\es'_1) \equiv_k (\astruct_{\M_2},\es_2,\es'_2)$, as desired.

		Notice now that $\es'_1$ is a minimum abductive explanation for $\es_1$ over $\M_1$. In fact, by our previous observations, the only other abductive explanation
		for $\es_1$ over $\M_1$ is $\es''_1 = \{1\}^{2^k} \cdot \{\bot\}^{2^k}$, which has the same number of undefined features as $\es'_1$.
		In turn, $\es'_2$ is not a minimum abductive explanation for $\es_2$ over $\M_2$. This is because $\es''_2 = \{1\}^{2^k} \cdot \{\bot\}^{2^k+1}$ is also an abductive explanation for $\es_2$ over $\M_2$, and $\es''_2$ has more undefined features than $\es'_2$.
	\end{proof}

\subsection{Evaluating \foil~is intractable}
\label{sec-hc-foil}
For each query $\varphi(x_1, \ldots, x_k)$ in $\foil$ and $\C$ a class of models, we define its associated problem \logicEvaluation$(\varphi, \C)$ as follows (we assume models and instances have the same dimension):
%
	%
	\begin{center}
		\fbox{\begin{tabular}{rl}
				{\sc Problem:} & \logicEvaluation$(\varphi, \C)$\\
				{\sc Input:} & A model $\M \in \C$ and partial instances $\es_1, \ldots, \es_k$\\
				{\sc Output:} & \textsc{Yes}, if $\M \models \varphi(\es_1, \ldots, \es_k)$, and \textsc{No} otherwise
		\end{tabular}}
	\end{center}

	It is known that there exists a formula $\varphi(x)$ in $\foil$ for which its evaluation problem over the class of decision trees is $\np$-hard~\citep{DBLP:conf/nips/ArenasBBPS21}. We want to determine whether the language $\foil$ is appropriate for implementation using SAT encodings.
	Thus, it is natural to ask whether the evaluation problem for formulas in this logic can always be decided in polynomial time by using a $\np$ oracle.
	However, we prove that this is not always the case. Although the evaluation of $\foil$ formulas is always in the polynomial hierarchy (PH), there exist formulas in $\foil$ for which their corresponding evaluation problems are hard for every level of PH. Based on widely held complexity assumptions, we can conclude that $\foil$ contains formulas whose evaluations cannot be decided in polynomial time by using a $\np$ oracle even on decision trees ($\dt$).


	\begin{theorem}
		\label{thm:eval-folistar}
		The following statements hold:
		\begin{enumerate}
		\item
		Let $\varphi$ be a \textnormal{\foil} formula. Then, there exists $k \geq 0$ such that
			\textnormal{\logicEvaluation$(\varphi, \nnf)$} is in the $\Sigma_k^{\rm{P}}$ complexity class.
			\item
			For every $k \geq 0$, there is an \textnormal{\foil}-formula $\varphi_k$ such that \textnormal{\logicEvaluation$(\varphi_k, \dt)$}
			is $\Sigma_k^{\rm{P}}$-hard.
			\end{enumerate}
	\end{theorem}

\begin{proof} 
	For the first item, consider a fixed $\foil$ formula $\varphi(x_1,\dots,x_m)$. We assume without loss of generality that $\varphi$ is in prenex normal form, i.e.,
	it is of the form $$\exists \bar y_1 \forall \bar y_2 \cdots Q_k \bar y_k \, \psi(x_1,\dots,x_m,\bar y_1,\dots,\bar y_k), \quad \quad (k \geq 0)$$
	where $Q_k = \exists$ if $k$ is odd and $Q_k = \forall$ otherwise, and $\psi$ is a quantifier-free formula.
	A FOIL formula of this form is called a $\Sigma_k$-$\foil$ {\em formula}.
	Consider that $\M$ is a negation normal form of dimension $n$,
	and assume that we want to check whether $\M \models \varphi(\es_1,\dots,\es_m)$, for $\es_1,\dots,\es_m$ given partial instances of dimension $n$. We know that the predicates $\pos$ and $\subseteq$ can be decided in polynomial time on $\M$. Additionally, the formula $\varphi$ is fixed, and thus the length of each tuple $\bar y_i$, for $i \leq k$, is constant. Therefore, we can decide this problem in polynomial time by
	using a $\Sigma_k$-alternating Turing machine (as the fixed size quantifier-free formula $\psi$ can be evaluated in polynomial time over $\M$).

	We now deal with the second item.
	We start by studying the complexity of the well-known {\em quantified Boolean formula} (QBF) problem for the case when the underlying formula (or,
	more precisely, the underlying Boolean function) is defined by a decision tree. More precisely, suppose that $\M$ is a
	decision tree
	over
	instances of dimension $n$. A $\Sigma_k$-QBF over $\M$, for $k > 1$, is an expression
	$$\exists P_1 \forall P_2 \cdots Q_k P_k \, \M,$$
	where $Q_k = \exists$ if $k$ is odd and $Q_k = \forall$ otherwise, and $P_1,\dots,P_k$ is a partition of $\{1,\dots,n\}$ into $k$ equivalence classes.
	As an example, if $\M$ is of dimension 3 then $\exists \{2,1\} \forall \{3\} \, \M$ is a $\Sigma_2$-QBF over $\M$. The semantics of these expressions is standard.
	For instance, $\exists \{1,2\} \forall \{3\} \, \M$ holds if there exists a partial instance $(b_1,b_2,\bot) \in \{0,1\} \times \{0,1\} \times \{\bot\}$ such that both
	$\M(b_1,b_2,0) = 1$ and $\M(b_1,b_2,1) = 1$.

	For a fixed $k > 1$, we introduce then the problem $\Sigma_k$-\textsc{QBF}$(\dt)$. It takes as input a $\Sigma_k$-QBF $\alpha$ over $\M$, for $\M$ a decision tree, and asks whether $\alpha$ holds.
	We establish the following result, which we believe of independent interest, as (to the best of our knowledge) the complexity of the
	QBF problem over decision trees has not been studied in the literature (refer to Section \ref{proof:qbf} for the proof).

	\begin{lemma} \label{lemma:qbf}
		For every odd $k \geq 1$, the problem \textnormal{$\Sigma_{k+1}$-\textsc{QBF}$(\dt)$} is $\Sigma_{k}^{\text{P}}$-complete.
	\end{lemma}

	For the second item, we can now finish the proof of the theorem with the help of Lemma \ref{lemma:qbf} and a reduction from $\Sigma_{k+1}$-\textsc{QBF}$(\dt)$. We can assume that $k$ is odd because for proving that there are \textnormal{\foil}-formulas $\varphi_{k'}$ such that \textnormal{\logicEvaluation$(\varphi_{k'}, \dt)$} is $\Sigma_{k'}^{\rm{P}}$-hard it is enough to show that there are \textnormal{\foil}-formulas $\varphi_{k}$ such that \textnormal{\logicEvaluation$(\varphi_{k}, \dt)$} is $\Sigma_{k'}^{\rm{P}}$-hard for some $k \geq k'$.
	The input to $\Sigma_{k+1}$-\textsc{QBF}$(\dt)$ is given by an expression $\alpha$ of the form
	$$\exists P_1 \forall P_2 \cdots \exists P_k \forall P_{k+1} \, \M,$$
	for $\M$ a decision tree of dimension $n$ and $P_1,\dots,P_{k+1}$ a partition of $\{1,\dots,n\}$.
	We explain next how the formula $\varphi_k(x_1,\dots,x_{k+1})$ is defined.

	We start by defining some auxiliary terminology.
	We use $x[i]$ to denote the $i$-th feature of
	the partial instance that is assigned to variable $x$. We define the following formulas.
	\begin{itemize}
		\item ${\sf Undef}(x) := \neg \exists y (y \subset x)$. That is, ${\sf Undef}$ defines the set that only consists of the partial instance $\{\bot\}^{n}$ in which all components are undefined.

		\item ${\sf Single}(x) := \exists y (y \subset x) \wedge \forall y (y \subset x \, \rightarrow \, {\sf Undef}(y))$.
		That is, ${\sf Single}$ defines the set that consists precisely of those partial instances in $\{0,1,\bot\}^{n}$ which have exactly one defined component.

		\item $(x \sqcup y = z) := (x \subseteq z) \wedge (y \subseteq z) \wedge \neg \exists w \big((x \subseteq w) \wedge (y \subseteq w) \wedge (w \subset z)\big)$. That is, $z$, if it exists, is the {\em join} of $x$ and $y$. In other words, $z$ is defined if every feature that is defined over $x$ and $y$ takes the same value in both partial instances,
		and, in such case, for each $1 \leq i \leq n$ we have that $z[i] = x[i] \sqcup y[i]$, where $\sqcup$ is the commutative and idempotent
		binary operation that satisfies
		$\bot \sqcup 0 = 0$ and $\bot \sqcup 1 = 1$.

		As an example, $(1,0,\bot,\bot) \sqcup (1,\bot,\bot,1) = (1,0,\bot,1)$, while $(1,\bot) \sqcup (0,0)$ is undefined.

		\item $(x \sqcap y = z) := (z \subseteq x) \wedge (z \subseteq y) \wedge \neg \exists w \big((w \subseteq x) \wedge (w \subseteq y) \wedge (z \subset w)\big)$. That is, $z$
		is the {\em meet} of $x$ and $y$ (which always exists). In other words, for each $1 \leq i \leq n$ we have that $z[i] = x[i] \sqcap y[i]$,
		where $\sqcap$ is the commutative and idempotent
		binary operation that satisfies
		$\bot \sqcap 0 = \bot \sqcap 1 = 0 \sqcap 1 = \bot$.

		As an example, $(1,0,\bot,\bot) \sqcap (1,\bot,\bot,1) = (1,\bot,\bot,\bot)$, while $(1,\bot) \sqcap (0,0) = (\bot,\bot)$.

		\item ${\sf Comp}(x,y) := \exists w \exists z ({\sf Undef}(z) \, \wedge \, x \sqcup y = w \, \wedge \, x \sqcap y = z)$.
		That is, ${\sf Comp}$ defines the pairs $(\es_1,\es_2)$
		of partial instances in $\{0,1,\bot\}^n \times \{0,1,\bot\}^n$ such that no feature that is defined in $\es_1$ is also defined in $\es_2$, and vice versa. In fact, assume
		for the sake of contradiction that this is not the case. By symmetry, we only have to consider the following two cases.
		\begin{itemize}
			\item There is an $i \leq n$ with $\es_1[i] = 1$ and $\es_2[i] = 0$. Then the join of $\es_1$ and $\es_2$ does not exist.
			\item There is an $i \leq n$ with $\es_1[i]  = \es_2[i] = 1$. Then the $i$-th component of the meet of $\es_1$ and $\es_2$ takes value 1,
			and hence $\es_1 \sqcap \es_2 \neq \{\bot\}^n$.
		\end{itemize}
		\item ${\sf MaxComp}(x,y) := {\sf Comp}(x,y) \, \wedge \, \neg \exists z \big((y \subset z) \wedge {\sf Comp}(x,z)\big)$.
		That is, ${\sf MaxComp}$ defines the pairs $(\es_1,\es_2)$ such that the components that are defined in $\es_1$ are precisely the ones that are undefined in $\es_2$, and vice versa.
		\item ${\sf Rel}(x,y) := \neg \exists z \big((z \subseteq y)  \, \wedge \, {\sf Single}(z) \, \wedge\, {\sf Comp}(x,z)\big)$.
		That is, ${\sf Rel}$ defines the pairs $(\es_1,\es_2)$
		of partial instances in $\{0,1,\bot\}^n \times \{0,1,\bot\}^n$ such that every feature that is defined in $\es_1$ is also defined in $\es_2$.
		\item ${\sf MaxRel}(x,y) := {\sf Rel}(x,y) \, \wedge \, \neg \exists z \big((z \subset y) \wedge {\sf Rel}(x,z)\big)$.
		That is, ${\sf MaxRel}$ defines the pairs $(\es_1,\es_2)$ such that the features defined in $\es_1$ and in $\es_2$ are the same.
		%
	\end{itemize}

	For defining the formula $\varphi_k(x_1,\dots,x_{k+1})$ we will use guarded quantifiers.
	For each $i$ with $1 \leq i \leq k+1$ consider
	\begin{align*}
		\exists^{\mathcal{G}(x_i)} y_i \, \psi \ &= \ \exists y_i \, \big({\sf MaxRel}(x_i,y_i) \ \wedge \ \psi \big) \\
		\forall^{\mathcal{G}(x_i)} y_i \, \psi \ &= \ \forall y_i \, \big({\sf MaxRel}(x_i,y_i) \ \rightarrow \ \psi \big)
	\end{align*}

	We now define the formula $\varphi_k(x_1,\dots,x_{k+1})$ as
	\begin{align*}
		\exists^{\mathcal{G}(x_1)} y_1 \forall^{\mathcal{G}(x_2)} y_2 \cdots \exists^{\mathcal{G}(x_k)} y_k \forall^{\mathcal{G}(x_{k+1})} y_{k+1} \ \forall z \big(z = y_1 \sqcup y_2 \sqcup \cdots \sqcup y_{k+1} \, \rightarrow \, \pos(z)\big)
	\end{align*}


	For each $i$ with $1 \leq i \leq k+1$, let $\es_i$ be the partial instance of dimension $n$ such that
	$$\es_i[j] \ = \ \begin{cases}
		1 \quad \quad & \text{if $j \in P_i$,} \\
		\bot & \text{otherwise.}
	\end{cases}$$
	That is, $\es_i$ takes value 1 over the features in $P_i$ and it is undefined over all other features.
	We claim that $\alpha$ holds if, and only if, $\M \models \varphi_k(\es_1,\dots,\es_{k+1})$.
	The result then follows since $\M$ is a decision tree.

	For the sake of presentation we only prove the aforementioned equivalence
	for the case when $k = 1$, since the extension to $k > 1$ is standard (but cumbersome). That is, we consider the case when $\alpha = \exists P_1 \forall P_2 \M$ and, therefore,
	\begin{align*}
		\varphi_2(x_1,x_2) = \exists^{\mathcal{G}(x_1)} y_1 \forall^{\mathcal{G}(x_2)} y_2 \ \forall z \big(z = y_1 \sqcup y_2 \, \rightarrow \, \pos(z)\big).
	\end{align*}

	\begin{itemize}
		\item[$(\Leftarrow)$] Assume first that $\M \models \varphi_2(\es_1,\es_2)$. Hence, there exists a partial instance $\es'_1$ such that
		\begin{equation} \label{eq_case_one}
			\M \, \models \, \big({\sf MaxRel}(\es_1,\es'_1) \ \wedge \ \forall^{\mathcal{G}(\es_2)} y_2 \ \forall z \big(z = \es'_1 \sqcup y_2 \, \rightarrow \, \pos(z)\big) \big).
		\end{equation}
		This means that the features defined in $\es_1$ and $\es'_1$ are exactly the same, and hence $\es'_1$ is a partial instance that is defined precisely over the features
		in $P_1$. We claim that every instance $\es$ that is a completion of $\es'_1$ satisfies $\M(\es) = 1$, thus showing that $\alpha$ holds. In fact, take $\es$ to be an arbitrary completion. By definition, $\es$ can be written as $\es'_1 \sqcup \es'_2$, where $\es'_2$ is a partial instance that is defined precisely over those features not in $P_1$, i.e.,
		over the features in $P_2$. Thus in the formula \eqref{eq_case_one} we can assign the partial instance $\es'_2$ to the variable $y_2$ and the instance $\es$ to the variable $z$,
		which allows us to conclude that $\M \models \pos(\es)$. This tells us that $\M(\es) = 1$.

		\item[$(\Rightarrow)$] Assume in turn that $\alpha$ holds, and hence that there is a partial instance $\es'_1$ that is defined precisely over the features
		in $P_1$ such that every instance $\es$ that is a completion of $\es'_1$ satisfies $\M(\es) = 1$. We claim that
		\begin{align*}
			\M \, \models \, \big({\sf MaxRel}(\es_1,\es'_1) \ \wedge \ \forall^{\mathcal{G}(\es_2)} y_2 \ \forall z \big(z = \es'_1 \sqcup y_2 \, \rightarrow \, \pos(z)\big) \big),
		\end{align*}
		which implies that $\M \models \varphi_2(\es_1,\es_2)$. In fact, let $\es'_2$ be an arbitrary instance such that ${\sf MaxRel}(\es_2,\es'_2)$ holds. By definition, $\es'_2$ is defined precisely over the features in $P_2$. Let $\es = \es'_1 \sqcup \es'_2$. Notice that $\es$ is well-defined since the sets of features defined in $\es'_1$ and $\es'_2$, respectively, are disjoint. Moreover, $\es$ is a completion of $\es'_1$ as $P_1 \cup P_2 = \{1,\dots,n\}$. We then have that $\M(\es) = 1$ as $\alpha$ holds. This allows us to conclude that
		$\M \models \pos(\es)$, and hence that $\M \models \varphi_2(\es_1,\es_2)$.
	\end{itemize}
	This concludes the proof of the theorem.
\end{proof}

%% file: new-sections/f-foil.tex

In the previous section we identified two limitations of $\foil$ that must be addressed in order to build a practical logic for explanations. On one hand, we must extend $\foil$ to increase its expressive power, and on the other hand, we must constrain the resulting logic
to ensure that its evaluation complexity is appropriate. In this
section we define \ffoil, a logic that takes both criteria into account and in which explainability notions can be expressed naturally.




\subsection{The atomic layer of \ffoil}
\label{sec-dtfoil-c-v-1}
$\foil$ cannot express properties such as minimum abductive explanations
that involve comparing cardinalities of sets of features. As a first
step, we solve this issue by extending the vocabulary of $\foil$ with a
simple binary relation $\lel$ defined as:
$$\M \models \es \lel \es' \ \ \Longleftrightarrow \ \
|\es_\bot| \geq |\es'_\bot|.$$ As we will show later, the use of this
predicate indeed allows us to express many notions of
explanations. Note that we could not simply keep only one of $\subseteq$ and $\lel$ when defining the new logic, as we show that they cannot
be defined in terms of each other. First, we show that predicate $\lel$ cannot be defined in terms of predicate $\subseteq$.
\begin{proposition}
	There is no formula $\varphi(x,y)$ in \textnormal{\foil} defined over the vocabulary
	$\{\subseteq\}$ such that, for every decision tree $\T$ and pair
	of partial instances $\es$, $\es'$, we have that
	$$\T \models \varphi(\es,\es') \ \Longleftrightarrow \ |\es_\bot| \geq |\es'_\bot|.$$
\end{proposition}
\begin{proof}
	For the sake of contradiction, assume that $\varphi(x,y)$ is
	definable in $\foil$ over the vocabulary $\{\subseteq\}$. Then
	the following are formulas in $\foil$:
	\begin{eqnarray*}
		\sr(x,y) &:=& \full(x) \wedge \, y \subseteq x \wedge \forall z \, \big(y \subseteq z \wedge \full(z) \, \rightarrow \, (\pos(z) \leftrightarrow \pos(\ x))\big),\\
		\msr(x,y) &:=& \sr(x,y) \wedge \forall z \, \big(\sr(x,z) \to (\varphi(z,y) \to \varphi(y,z))\big).
	\end{eqnarray*}
	But the second formula verifies if a partial instance $y$ is a minimum abductive explanation for a given
	instance $x$, which contradicts the inexpressibility result of
	Theorem \ref{thm:ne-foil}, and hence concludes the proof of the proposition.
\end{proof}
Second, we show that predicate $\subseteq$ cannot be defined in terms of predicate $\lel$.
\begin{proposition}
	There is no formula $\psi(x,y)$ in \textnormal{\foil} defined over the vocabulary
	$\{\lel\}$ such that, for every decision tree $\T$\footnote{Naturally, this statement does not rely on decision trees at all since it concerns only $\subseteq$ and $\lel$; we only state it in these terms for consistency.} and pair
	of partial instances $\es$, $\es'$, we have that
	$$\T \models \psi(\es,\es') \ \Longleftrightarrow \ \es \text{ is subsumed by } \es'.$$
\end{proposition}

\begin{proof}
	Intuitively, $\lel$ is invariant under any bijection of partial instances that preserves the number of $\bot$'s, whereas subsumption is not. We formalize this as follows.

	For the sake of contradiction, assume that $\psi(x,y)$ is
	definable in $\foil$ over the vocabulary $\{\lel\}$, and let $n \geq
	3$. Moreover, for every $k \in \{0, \ldots, n\}$, define $L_k$ as the
	following set of partial instances:
	\begin{eqnarray*}
		L_k &=& \{ \es \in \{0,1,\bot\}^n \mid |\es_\bot| = k\},
	\end{eqnarray*}
	and let $f_k : L_k \to L_k$ be an arbitrary bijection from $L_k$ to
	itself. Finally, let $f : \{0, 1, \bot\}^n \to \{0, 1, \bot\}^n$ be
	defined as $f(\es) = f_i(\es)$ if $\es \in L_i$. Clearly, $f$ is a
	bijection from $\{0, 1, \bot\}^n$ to $\{0, 1, \bot\}^n$.

	For a decision tree $\T$ of dimension $n$, define $\mathfrak{A}'_\T$
	as the restriction of $\mathfrak{A}_\T$ to the vocabulary
	$\{\lel\}$. Then function $f$ is an automorphism of
	$\mathfrak{A}'_\T$ since $f$ is a bijection from $\{0, 1, \bot\}^n$ to
	$\{0, 1, \bot\}^n$, and for every pair of partial instances $\es_1,
	\es_2$:
	\begin{align*}
		\mathfrak{A}'_\T \models \es_1 \lel \es_2 \quad \text{ if and only if } \quad
		\mathfrak{A}'_\T \models f(\es_1) \lel f(\es_2).
	\end{align*}
	Then given that $\psi(x,y)$ is definable in first-order logic over the
	vocabulary $\{\lel\}$, we have that for every pair of partial
	instances $\es_1, \es_2$:
	\begin{align}\label{eq-subsumed-aut}
		\mathfrak{A}'_\T \models \psi(\es_1, \es_2) \quad \text{ if and only if } \quad
		\mathfrak{A}'_\T \models \psi(f(\es_1), f(\es_2)).
	\end{align}
	But now assume that $g_k : L_k \to L_k$ is defined as the identity
	function for every $k \in \{0, \ldots, n\} \setminus \{1\}$, and
	assume that $g_1$ is defined as follows for every partial instance
	$\es$:
	\begin{eqnarray*}
		g_1(\es) &=&
		\begin{cases}
			(\bot, 0, \ldots, 0) & \text{if } \es = (0, \ldots, 0, \bot)\\
			(0, \ldots, 0, \bot) & \text{if } \es = (\bot, 0, \ldots, 0)\\
			\es & \text{otherwise}
		\end{cases}
	\end{eqnarray*}
	Clearly, each function $g_i$ is a bijection. Moreover, let $g : \{0,
	1, \bot\}^n \to \{0, 1, \bot\}^n$ be defined as $g(\es) = g_i(\es)$ if
	$\es \in L_i$. Then we have by \eqref{eq-subsumed-aut} that for every
	pair of partial instances $\es_1$, $\es_2$:
	\begin{align*}
		\mathfrak{A}'_\T \models \psi(\es_1, \es_2) \quad \text{ if and only if } \quad
		\mathfrak{A}'_\T \models \psi(g(\es_1), g(\es_2)).
	\end{align*}
	Hence, taking $\es_1 = (\bot, \bot, 0, \ldots, 0)$ and $\es_2
	= (\bot, 0, \ldots, 0)$, given that $g(\es_1) = (\bot, \bot, 0, \ldots, 0)$ and ${g(\es_2)
	= (0, \ldots, 0, \bot)}$, we conclude that:
	\begin{align*}
		\begin{gathered}
			\mathfrak{A}'_\T \models \psi((\bot, \bot, 0, \ldots, 0),\, (\bot, 0, \ldots, 0)) \\
			\text{ if and only if }\\
			\mathfrak{A}'_\T \models \psi((\bot, \bot, 0, \ldots, 0),\, (0, \ldots, 0, \bot)).
		\end{gathered}
	\end{align*}
	But this leads to a contradiction, since $(\bot, \bot, 0, \ldots, 0)$
	is subsumed by $(\bot, 0, \ldots, 0)$, but $(\bot, \bot, 0, \ldots,
	0)$ is not subsumed by $(0, \ldots, 0, \bot)$. This concludes the
	proof of the proposition.
\end{proof}

However, adding $\lel$ to $\foil$ can only add extra complexity. Therefore, our second step is to define the logic $\ffoil$ expressive enough to capture important notions, but keeping the evaluation tractable using $\sat$ solvers. Our logic $\ffoil$ consists of three hierarchical layers, where the first layer does not depend on the structure of the model.

Predicates $\subseteq$ and $\lel$, as well as predicate $\full$ used in Section \ref{sec:expressing-in-foil}, can be regarded as {\em
syntactic} in the sense that they refer to the values of the features
of partial instances, and they do not make reference to classification
models. It turns out that all the syntactic predicates needed in our
logical formalism can be expressed as first-order formulas over the
predicates $\subseteq$ and $\lel$. The {\em atomic formulas} of $\ffoil$ are defined
as first-order formulas over the vocabulary $\{\subseteq, \lel\}$. We now prove that such formulas can be evaluated in polynomial time. We also prove that in the case of \emph{sentences}, that is, formulas without free variables, it is decidable whether a sentence is true in every structure $\mathfrak{B}_n$.

Given $n \geq 0$ and a model $\M$ of dimension $n$, define $\mathfrak{B}_\M$ as a structure over the vocabulary $\{\subseteq, \lel\}$ generated from $\mathfrak{A}_\M$ by removing the interpretation of predicate $\pos$, and adding the interpretation of predicate $\lel$. Notice that, given two models $\M_1$ and $\M_2$ of dimension $n$, we have that $\mathfrak{B}_{\M_1} =
\mathfrak{B}_{\M_2}$, so we define simply $\mathfrak{B}_n$ as $\mathfrak{B}_\M$ for an arbitrary model of dimension $n$. Therefore, when measuring the complexity of evaluating formulas in the atomic layer, we take $n$ in unary as part of the input, since $n$ is the size of the partial instances. Hence, for each formula $\varphi(x_1, \ldots, x_k)$ in the atomic layer of $\ffoil$, we define its associated problem \logicEvaluation$(\varphi)$ as follows:

\begin{center}
	\fbox{\begin{tabular}{rl}
			{\sc Problem:} & \logicEvaluation$(\varphi)$\\
			{\sc Input:} & An integer $n \in \mathbb{N}$ given in unary and partial instances $\es_1, \ldots, \es_k$\\
			{\sc Output:} & \textsc{Yes}, if $\mathfrak{B}_n \models \varphi(\es_1, \ldots, \es_k)$, and \textsc{No} otherwise
	\end{tabular}}
\end{center}

Recall also that the \emph{width} of a first-order formula $\varphi$, denoted $\rm{wd}(\varphi)$, is defined as the maximum number of free variables among all subformulas of $\varphi$ (see \cite{MR2374364} for a reference).

\begin{theorem}\label{theo:ptime-atomic}
	The following statements hold:
	\begin{enumerate}
		\item Let $\varphi$ be a first-order formula defined over the vocabulary $\{\subseteq, \lel\}$. Then \textnormal{\logicEvaluation$(\varphi) \in \ptime$}.
		\item It is decidable whether a given first-order sentence $\varphi$ defined over the vocabulary $\{\subseteq, \lel\}$ is true in every structure $\mathfrak{B}_n$. In particular, it can be solved in $2^{2^{\rm{poly}\big(|\varphi| \cdot 3^{\rm{wd}(\varphi)} \big)}}$ space, and, hence, in $2^{2^{2^{\rm{poly}\big(|\varphi| \cdot 3^{\rm{wd}(\varphi)} \big)}}}$ time.
	\end{enumerate}
\end{theorem}%

\begin{proof}
	We will prove both claims by a reduction to Presburger arithmetic. One standard presentation of Presburger arithmetic consists of two constants, $0$ and $1$, a binary relation $<$ and a binary function $+$. We consider the model $\mathbb{N}$ of the non-negative integers with the usual interpretations. We will use the following two well-known facts about Presburger arithmetic:
	\begin{enumerate}
		\item[I.] Presburger arithmetic admits quantifier elimination, that is, for every Presburger formula $\varphi(y_1, \dots, y_m)$ there exists a quantifier-free formula $\psi(y_1, \dots, y_m)$ such that $\varphi$ and $\psi$ are logically equivalent \cite{MR1111343}.
		\item[II.] The problem of determining the truth of sentences in Presburger Arithmetic with respect to the model $\mathbb{N}$ can be solved in double exponential space with respect to the size of the sentence \cite{MR566694}.
	\end{enumerate}

	First, we introduce some terminology. Let $n \in \mathbb{N}$ be the dimension. Then,
	for a tuple $\Gamma=(x_1,\dots,x_k)$ of (distinct) variables, an assignment $s
	\colon \Gamma \to \{0,1,\bot\}^n$, and a coordinate $i \in \{1,\dots,n\}$, we define the ``pattern'' $\pat^{\Gamma, s}_i$ as the tuple
	\(
	  \bigl(s(x_1)_i,\,\dots,\,s(x_k)_i\bigr)\in \{0,1,\bot\}^k.
	\)
	Given a pattern $\rho \in \{0,1,\bot\}^k$, we define its \emph{pattern count} in $(\Gamma,s)$ by
	\[
	c^{\Gamma,s}_\rho
	:=
	\bigl|\{\, i\in \{1,\dots,n\} \, : \, \pat^{\Gamma,s}_i=\rho \,\}\bigr|.
	\]
	Clearly,
	\(
	\displaystyle\sum_{\rho\in \{0,1,\bot \}^{k}} c^{\Gamma,s}_\rho = n.
	\)
	For example, let $\Gamma=(x_1,x_2)$ and let $n=4$. Suppose that $s(x_1)=(1,\bot,0,0)$ and $s(x_2)=(1,1,\bot,\bot)$. Then the four coordinates have patterns
	\[
	\operatorname{pat}^{\Gamma,s}_1=(1,1),\qquad
	\operatorname{pat}^{\Gamma,s}_2=(\bot,1),\qquad
	\operatorname{pat}^{\Gamma,s}_3=(0,\bot),\qquad
	\operatorname{pat}^{\Gamma,s}_4=(0,\bot).
	\]
	Hence
	\[
	c^{\Gamma,s}_{(1,1)}=
	c^{\Gamma,s}_{(\bot,1)}=1,\qquad
	c^{\Gamma,s}_{(0,\bot)}=2,
	\]
	and every other pattern in $\{0,1,\bot \}^2$ has count $0$.

	We are now ready to state the reduction lemma (refer to Section \ref{proof:presburger} for the proof):

	\begin{lemma}\label{presburger}
	Let $\varphi(x_1, \dots, x_\ell)$ be a first-order formula defined over the vocabulary
	$\{\subseteq,\lel\}$, and let $\Gamma \coloneq (x_1, \dots, x_k)$ be a tuple of distinct variables that contains all free variables of $\varphi$. Then there exists a Presburger formula
	\[
	\operatorname{T}_\Gamma(\varphi)\bigl((z_\rho)_{\rho\in\{0,1,\bot\}^{k}}\bigr)
	\]
	such that for every $n\in\mathbb{N}$ and every assignment $s:\Gamma\to \{0,1,\bot \}^n$,
	\[
	\mathfrak{B}_n\models \varphi(s(x_1), \dots, s(x_\ell))
	\quad\Longleftrightarrow\quad
	\mathbb{N} \models
	\operatorname{T}_\Gamma(\varphi)\bigl((c^{\Gamma,s}_\rho)_\rho\bigr).
	\]
	Moreover, if $k = \ell$ (that is, if $\Gamma$ contains exactly the free variables of $\varphi$), then we have that
	$$\bigg|\operatorname{T}_\Gamma(\varphi)\bigl((z_\rho)_{\rho\in\{0,1,\bot\}^{k}}\bigr)\bigg| = O\bigg(|\varphi| \cdot 3^{\rm{wd}(\varphi)} \cdot \mathrm{poly}(\rm{wd}(\varphi))\bigg),$$
	and the reduction can be computed using the same space.
	\end{lemma}

	We now show how Lemma~\ref{presburger} implies both statements of the theorem.

	For the first part, let $\varphi(x_1, \dots, x_k)$ be a first-order formula defined over the vocabulary $\{\subseteq, \lel\}$, and let $\Gamma = (x_1, \dots, x_k)$. Because Presburger arithmetic admits quantifier elimination, we know that there exists a quantifier-free Presburger formula $\psi\bigl((z_\rho)_{\rho\in\{0,1,\bot\}^{k}}\bigr)$ that is logically equivalent to $\operatorname{T}_\Gamma(\varphi)\bigl((z_\rho)_{\rho\in\{0,1,\bot\}^{k}}\bigr)$. From Lemma~\ref{presburger} we know that, for every $n\in\mathbb{N}$ and every assignment $s:\Gamma\to \{0,1,\bot\}^{n}$,
	\[
	\mathfrak{B}_n \models \varphi(s(x_1), \dots, s(x_\ell))
	\quad\Longleftrightarrow\quad
	\mathbb{N} \models
	\psi\bigl((c^{\Gamma,s}_\rho)_\rho\bigr).
	\]
	Now suppose we are given an integer $n \in \mathbb{N}$ in unary and partial instances $\es_1, \ldots, \es_k$. Let $s:\Gamma\to \{0,1,\bot\}^{n}$ be the assignment such that $s(x_i) = \es_i$ for every $i \in \{1, \dots, k\}$. Notice that we can compute in linear time all the values $\{c^{\Gamma,s}_\rho\}_{\rho\in\{0,1,\bot\}^{k}}$, as $3^k$ is constant with respect to the input size of \logicEvaluation$(\varphi)$. Notice that each variable $c^{\Gamma,s}_\rho$ has a value less than or equal to $n$. Because $\psi$ is a fixed, quantifier-free formula, we can evaluate it onto the values $\{c^{\Gamma,s}_\rho\}_{\rho\in\{0,1,\bot\}^{k}}$ in polynomial time. This shows that \textnormal{\logicEvaluation$(\varphi) \in \ptime$}.

	For the second part, let $\varphi$ be a first-order sentence defined over the vocabulary $\{\subseteq, \lel\}$. Because we can take $\Gamma$ as an empty context, there is only one possible pattern. Therefore, the formula $\operatorname{T}_\Gamma(\varphi)\bigl((z)\bigr)$ has exactly one free variable. Recall that, in general, given an assignment for $\Gamma$ in $\mathfrak{B}_n$, the sum over all pattern counts must equal $n$. Hence, in this case, $\varphi$ is true in every structure $\mathfrak{B}_n$ if and only if $\forall m \operatorname{T}_\Gamma(\varphi)\bigl((m)\bigr)$ is a true Presburger sentence. We know from Lemma~\ref{presburger} that the sentence $\forall m \operatorname{T}_\Gamma(\varphi)\bigl((m)\bigr)$ has size $O\bigg(|\varphi| \cdot 3^{\rm{wd}(\varphi)} \cdot \mathrm{poly}(\rm{wd}(\varphi))\bigg)$ and it can be constructed in at most the same space. Because the problem of determining the truth value of a Presburger sentence can be solved in double exponential space, we conclude that we can determine if $\varphi$ is true in every structure $\mathfrak{B}_n$ in $2^{2^{\rm{poly}\big(|\varphi| \cdot 3^{\rm{wd}(\varphi)} \big)}}$ space. This concludes the proof of the theorem.

\end{proof}

\subsection{The quantified layer of \ffoil}
\label{sec-quant}

In this layer we introduce predicates whose interpretation does depend on the model. The vocabulary of this layer is $\{\subseteq, \lel, \allpos, \allneg\}$, where $\M \models \allpos(\es)$ if and only if all instances in $\fullset(\es)$ are classified positively by $\M$, and $\allneg$ is defined analogously. As we will show, at this point we will already be able to express properties over polynomial-size sets of instances and minimality/minimum conditions.

We will need two auxiliary formulas from the atomic layer that we already defined during the proof of Theorem~\ref{thm:eval-folistar}. The first is $${\sf Undef}(x) := \neg \exists y (y \subset x),$$ which defines the set that only contains the partial instance $\{\bot\}^{n}$. The second is $${\sf Single}(x) := \exists y (y \subset x) \wedge \forall y (y \subset x \, \rightarrow \, {\sf Undef}(y)),$$ which defines the set of partial instances with exactly one defined feature.

The \emph{quantified layer} is recursively defined as follows:
\begin{enumerate}
	\item Boolean combinations of formulas from the atomic layer, together with $\allpos(x)$ and $\allneg(x)$, are formulas from the quantified layer.

	\item If $\varphi$ is a formula from the quantified layer, then $\exists x\ \varphi$ is a formula from the quantified layer.

	\item If $\varphi$ is a formula from the quantified layer, then \textnormal{$\forall x\ \left( {\sf Single}(x) \to \varphi\right)$} is a formula from the quantified layer.
\end{enumerate}

Using only the first rule we can already express some basic explainability properties. For example, we can express the query for weak abductive explanations as follows:
$$\sr(x, y) := \full(x) \wedge y \subseteq x \wedge (\allpos(x) \to \allpos(y)) \wedge (\allneg(x) \to \allneg(y)).$$

The third rule involves the concept of {\em guarded} quantification. In that case we only quantify over partial instances with exactly one defined feature, which naturally correspond to assignments of a value to a single feature. On any class of models $\C$, the number of partial instances with one defined feature is at most twice the dimension of the model, so we cannot express universal properties over superpolynomial-size sets of partial instances. Notice that our rules do not allow us to define unguarded universal quantifiers because the first rule only allows us to take Boolean combinations of unquantified formulas. In particular, in this layer we are not allowed to negate formulas that were produced using the second (or third) rule.

\begin{theorem}
	\label{theo:np-guarded}
	The following statements hold:
	\begin{enumerate}
		\item Let $\C$ be a class of models such that \textnormal{\logicEvaluation}$(\allpos(x), \C) \in \ptime$ and \textnormal{\logicEvaluation}$(\allneg(x), \C) \in \ptime$. Then \textnormal{\logicEvaluation}$(\varphi, \C) \in \np$ for every formula $\varphi$ from the quantified layer of \textnormal{\ffoil}.
		\item
		There exists a formula $\varphi$ from the quantified layer of \textnormal{\ffoil} such that \textnormal{\logicEvaluation}$(\varphi, \dt)$ is $\np$-hard.
	\end{enumerate}
\end{theorem}

\begin{proof}
	For the first item, let $\C$ be a class of models such that \textnormal{\logicEvaluation}$(\allpos(x), \C) \in \ptime$ and \textnormal{\logicEvaluation}$(\allneg(x), \C) \in \ptime$, and let $\varphi$ be a fixed formula from the quantified layer of \textnormal{\ffoil}. The algorithm is the following. For each existential quantifier we nondeterministically guess a partial instance as a polynomial-size witness. Each guarded universal quantifier ranges only over the set of partial instances with exactly one defined feature, whose size is linear in the dimension $n$. Since the formula is fixed, unfolding all guarded universal quantifiers yields only polynomially many cases (to be more precise, at most $n^c$ cases, where $c$ is the quantifier rank of $\varphi$). At every computation path of this process, we are left with a Boolean combination of formulas from the atomic layer, together with $\allpos(x)$ and $\allneg(x)$. Thanks to Theorem~\ref{theo:ptime-atomic} and to the hypothesis that \textnormal{\logicEvaluation}$(\allpos(x), \C) \in \ptime$ and \textnormal{\logicEvaluation}$(\allneg(x), \C) \in \ptime$, and considering that the formula $\varphi$ is fixed, we can do that evaluation in polynomial time.

	For the second item, we consider the following formula from the quantified layer of \textnormal{\ffoil}:
	$$\nmsr(x, y) \ := \ \exists z\ \big( \neg \sr(x, y) \vee [z \lnel y \wedge \sr(x, z)]  \big).$$
	$\nmsr$ defines the pairs $(\es_1,\es_2)$ such that the partial instance $\es_2$ is not a minimum abductive explanation for the instance $\es_1$, that is, it is logically equivalent to $\neg \msr(x, y)$. We conclude the proof using the following intermediate result, whose proof can be found in the appendix of this work (refer to Section~\ref{proof:maxp-hard}).

	\begin{lemma} \label{lemma:maxp-hard}
		\textnormal{\logicEvaluation}$(\msr(x,y), \dt)$ is $\conp$-hard.
	\end{lemma}


\end{proof}




\subsection{The \ffoil~logic}
\label{sec-fflogic}

We define $\ffoil$ as the logic obtained by taking Boolean combinations of formulas from the quantified layer. In particular, since we can negate quantified formulas, in this third layer we are allowed to use unguarded universal quantifiers. Nevertheless, notice that a necessary condition for a formula to have a valid syntax according to the $\ffoil$ logic is that alternations between unguarded quantifiers cannot occur.

We now provide a precise characterization of the complexity of the evaluation problem for $\ffoil$. More specifically, we establish that this problem can always be solved in the {\em Boolean Hierarchy over $\np$}
\citep{DBLP:conf/fct/Wechsung85,DBLP:journals/siamcomp/CaiGHHSWW88}, i.e., in the class consisting of Boolean combinations of $\np$ languages. In fact, we will show that the $\ffoil$ logic captures the entire Boolean Hierarchy.

For the following theorem, we denote the levels of the Boolean Hierarchy by $\bh_k$, and we denote by $\bh$ the Boolean Hierarchy consisting of all these levels.

\begin{theorem}
	\label{theo:thm:eval-ffoil}
	The following statements hold:
	\begin{enumerate}
		\item Let $\varphi$ be an \textnormal{\ffoil} formula. Then there exists a $k \geq 1$ such that, for every class of models $\C$ such that \textnormal{\logicEvaluation}$(\allpos(x), \C) \in \ptime$ and \textnormal{\logicEvaluation}$(\allneg(x), \C) \in \ptime$, it holds that \textnormal{\logicEvaluation}$(\varphi, \C) \in \bh_k$.
		\item For every $k \geq 1$, there exists an \textnormal{\ffoil} formula $\varphi$ such that \textnormal{\logicEvaluation}$(\varphi, \dt)$ is $\bh_k$-hard.
	\end{enumerate}
\end{theorem}

This result tells us that $\ffoil$ meets one of the fundamental criteria for an interpretability logic, namely that we can evaluate an $\ffoil$ formula over a tuple of partial instances in polynomial time with a polynomial number of calls to an $\np$ oracle. In fact, by definition of the Boolean hierarchy, the evaluation of a fixed $\ffoil$ formula can be done with a constant number of calls to an $\np$ oracle. Thus, we argue that the technology of $\sat$ solvers will allow us to tractably evaluate $\ffoil$ over classes that support consistency and validity checks in polynomial time. More precisely, Theorem~\ref{theo:thm:eval-ffoil} requires the class of models $\C$ to satisfy that \textnormal{\logicEvaluation}$(\allpos(x), \C) \in \ptime$ and \textnormal{\logicEvaluation}$(\allneg(x), \C) \in \ptime$ in order for $\ffoil$ to be able to tractably solve its evaluation problem over that class. This includes decision trees, but also richer representation classes like $\ddnnf$ circuits. Moreover, this includes fragments of the class of circuits corresponding to propositional formulas in conjunctive normal form ($\cnf$) whose satisfiability can be decided in polynomial time, such as the class of circuits corresponding to CNF formulas in which each clause contains at most two literals ($\tcnf$), and the class of circuits corresponding to Horn CNF formulas ($\horn$). We formally state these results in the following corollary.

\begin{corollary}
	\begin{sloppypar}
Let $\varphi$ be an \textnormal{\ffoil} formula. Then \textnormal{\logicEvaluation}$(\varphi, \ddnnf) \in \bh$,
\textnormal{\logicEvaluation}$(\varphi, \tcnf) \in \bh$, and
\textnormal{\logicEvaluation}$(\varphi, \horn) \in \bh$.
	\end{sloppypar}
\end{corollary}
In what follows, we provide a proof of Theorem \ref{theo:thm:eval-ffoil}.

\begin{proof}[Proof of Theorem \ref{theo:thm:eval-ffoil}]
	We consider languages over a finite alphabet $\Sigma$. First, we introduce Boolean operations between complexity classes as follows \citep{DBLP:conf/fct/Wechsung85}:
	\begin{enumerate}
		\item $A \vee B = \{L_A \cup L_B \ | \ L_A \in A \text{ and } L_B \in B \}$;
		\item $A \wedge B = \{L_A \cap L_B \ | \ L_A \in A \text{ and } L_B \in B \}$;
		\item $\rm{co}A = \{\overline{L} \ | \ L \in A\}$.
	\end{enumerate}
	Then, the Boolean Hierarchy $\bh$ is defined as the union $\cup_{k \geq 1} \bh_{k}$ \citep{DBLP:journals/siamcomp/CaiGHHSWW88}, where:
	\begin{enumerate}
		\item $\bh_{1} = \np$;
		\item $\bh_{2i} =  \bh_{2i-1} \wedge \conp$;
		\item $\bh_{2i + 1} =  \bh_{2i} \vee \np$.
	\end{enumerate}
	Note that $\conp \subseteq \bh_{2}$. In fact, let $L \in \conp$ and note that $L = \Sigma^\ast \cap L$, where $\Sigma^\ast \in \bh_{1}$.

	With this definition, every Boolean combination of $\np$ and $\conp$ languages is contained in $\bh_{k}$ for some positive integer $k$.

	For the first item of the theorem, let $\varphi$ be an \textnormal{\ffoil} formula. We know that $\varphi$ is a fixed Boolean combination of formulas from the quantified layer of \textnormal{\ffoil}. Now let $\C$ be a class of models such that \textnormal{\logicEvaluation}$(\allpos(x), \C) \in \ptime$ and \textnormal{\logicEvaluation}$(\allneg(x), \C) \in \ptime$. Thanks to the first part of Theorem~\ref{theo:np-guarded} we know that \logicEvaluation$(\psi, \C) \in \np$ for every formula $\psi$ from the quantified layer of \textnormal{\ffoil} that appears as a subformula of $\varphi$. This means that the evaluation problem \logicEvaluation$(\varphi, \C)$ corresponds to a fixed Boolean combination of languages in $\np$, and so it must be contained in $\bh_k$ for some $k \geq 1$. Notice that such a $k$ does depend on $\varphi$ but not on $\C$.

	We now turn our attention to the second item of the theorem. We will first describe a family of decision problems known to be complete for every level of the Boolean hierarchy. As usual, let us denote by $\sat$ the language of propositional formulas that are satisfiable, and by $\unsat$ the language of propositional formulas that are not satisfiable. For each $k \geq 1$ we define the language $\sat(k)$ recursively as follows:
	\begin{enumerate}
		\item $\sat(1) \ := \ \{(\varphi_1) \ \mid \ \varphi_1 \in \sat \}$;
		\item $\sat(2i) \ := \ \{(\varphi_1, \dots \varphi_{2i}) \ \mid \ (\varphi_1, \dots \varphi_{2i-1}) \in \sat(2i-1) \, \wedge \, \varphi_{2i} \in \unsat\}$;
		\item $\sat(2i+1) \ := \ \{(\varphi_1, \dots \varphi_{2i+1}) \ \mid \ (\varphi_1, \dots \varphi_{2i}) \in \sat(2i) \, \vee \, \varphi_{2i+1} \in \sat\}$.
	\end{enumerate}
	It is known that, for every $k \geq 1$, $\sat(k)$ is $\bh_k$-complete
	\citep{DBLP:journals/siamcomp/CaiGHHSWW88}. We will now fix a $k \geq 1$ and prove that there exists an $\ffoil$ formula $\varphi_k$ such that $\sat(k)$ can be reduced in polynomial time to \logicEvaluation$(\varphi_k, \dt)$, thus concluding the hardness item of the theorem.

	We know from Lemma~\ref{lemma:maxp-hard} that the following language is $\np$-hard:
	$$L = \{ (\T, \es, \es') \ \mid \ \T \text{ is a decision tree, } \es \text{ and } \es'
	\text{ are partial instances and } \T \models \neg\msr(\es, \es')\}.$$
	Also, because the class of decision trees satisfies the hypothesis of the first item of Theorem~\ref{theo:np-guarded}, we actually know that $L$ is $\np$-complete. Hence, we have a polynomial-time algorithm that, given a propositional formula $\psi$, constructs a decision tree $\T_\psi$ and partial instances $\es_{\psi}$, $\es_{\psi}'$ such that:
	\begin{equation} \label{sat_reduction}
	\psi \in \sat \quad \iff \quad \T_\psi \models \neg \msr(\es_{\psi}, \es_{\psi}').
	\end{equation}

	Let $(\psi_1, \ldots, \psi_k)$ be a tuple of $k$ propositional formulas, and assume that, for each $i \in \{1,\dots,k\}$, the decision tree $\T_{\psi_i}$ has dimension $n_i$. Then a decision tree $\T$ of dimension $d = k + \sum_{\ell=1}^k{n_\ell}$ is defined as follows:
	\begin{center}
		\begin{tikzpicture}
			\node[circle,draw=black] (c1) {$1$};
			\node[below left = 6mm and 6mm of c1] (tc1) {$\T_{\psi_1}$};
			\node[circle,draw=black,below right = 6mm and 6mm of c1] (c2) {$2$};
			\node[below left = 6mm and 6mm of c2] (tc2) {$\T_{\psi_2}$};
			\node[circle,draw=black,below right = 6mm and 6mm of c2] (c3) {$3$};
			\node[below left = 6mm and 6mm of c3] (tc3) {$\T_{\psi_3}$};
			\node[below right = 6mm and 6mm of c3] (d) {$\cdots$};
			\node[circle,draw=black,below right = 6mm and 6mm of d] (cn) {$k$};
			\node[below left = 6mm and 6mm of cn] (tcn) {$\T_{\psi_k}$};
			\node[circle,draw=black,below right = 6mm and 6mm of cn, minimum size=8mm] (o) {$\true$};

			\path[arrout] (c1) edge node[above] {$1$} (tc1);
			\path[arrout] (c1) edge node[above] {$0$} (c2);
			\path[arrout] (c2) edge node[above] {$1$} (tc2);
			\path[arrout] (c2) edge node[above] {$0$} (c3);
			\path[arrout] (c3) edge node[above] {$1$} (tc3);
			\path[arrout] (c3) edge node[above] {$0$} (d);
			\path[arrout] (d) edge node[above] {$0$} (cn);
			\path[arrout] (cn) edge node[above] {$1$} (tcn);
			\path[arrout] (cn) edge node[above] {$0$} (o);
		\end{tikzpicture}
	\end{center}
	where, for each $i \in \{1, \dots, k\}$, $\T_{\psi_i}$ mentions the features:
	$$B_i \ := \ \{s_i +1,\, s_i+2,\, \dots s_i+n_i\} \qquad \text{where } s_i = k + \sum_{\ell=1}^{i-1}{n_\ell}.$$
	This way, we ensure that, for every $i \neq j$, $\T_{\psi_i}$ and
	$\T_{\psi_j}$ are defined over disjoint sets of features. Moreover, define the following partial instances of dimension $d$:
	\begin{itemize}
		\item For each $i \in \{1, \ldots, k\}$, the partial instance
		$\es_i$ is defined as $\{0\}^{i-1} \cdot \{1\} \cdot \{0\}^{k-i} \cdot \{\bot\}^{d-k}$.
		\item For each $i \in \{1, \ldots, k\}$, the partial instances
		$\es_{i,1}$ and $\es_{i,2}$ are defined as follows:
		\begin{itemize}
			\item $\es_{i,1}[i] = \es_{i,2}[i] = 1$;
			\item $\es_{i,1}[j] = \es_{i,2}[j] = 0$ for
			every $j \in \{1, \ldots, k\} \setminus \{i\}$;
			\item $\es_{i,1}[j] = \es_{\psi_i}[j - s_i]$ for every $j \in B_i$;
			\item $\es_{i,2}[j] = \es_{\psi_i}'[j - s_i]$ for every $j \in B_i$;
			\item $\es_{i,1}[j] = 0$ for every $j \in \{k+1, \ldots, d\} \setminus B_i$;
			\item $\es_{i,2}[j] = \bot$ for every $j \in \{k+1, \ldots, d\} \setminus B_i$.
		\end{itemize}
	\end{itemize}
	Moreover, consider the following $\ffoil$ formulas:
	\begin{eqnarray*}
		\rsr(x,y,w) &:=& w \subseteq x \wedge w \subseteq y \wedge \full(x) \wedge y \subseteq x \wedge  (\allpos(x) \to \allpos(y)) \wedge (\allneg(x) \to \allneg(y)); \\
		\rmsr(x,y,w) &:=& \rsr(x,y,w) \ \wedge \ \neg \exists z \, \big(w \subseteq z \, \wedge \, \rsr(x,z,w) \, \wedge \, z \lnel y\big).
	\end{eqnarray*}
	These formulas should be interpreted as the usual predicates, but relativized to one branch of $\T$. More concretely, the partial instance $w$ will serve to select in which of the branches $\T_{\psi_1}, \T_{\psi_2}, \dots, \T_{\psi_k}$ we will look at.

	Notice that the decision tree $\T$ and the partial instances
	$\es_1$, $\ldots$, $\es_k$, $\es_{1,1}$, $\es_{1,2}$, $\ldots$,
	$\es_{k,1}$, $\es_{k,2}$ can be constructed in polynomial time in
	the size of $(\psi_1, \ldots, \psi_k)$. Besides, from the
	definition of these elements, for every $i \in \{1, \ldots, k\}$ it holds that
	\begin{equation} \label{lem:bh-foilpm}
		\T_{\psi_i} \models\ \msr(\es_{\psi_i}, \es_{\psi_i}') \quad \iff \quad \T \models\ \rmsr(\es_{i,1}, \es_{i, 2}, \es_i).
	\end{equation}
	Finally, let $\varphi_k$ be the $\ffoil$ formula obtained by
	constructing the following sequences of formulas $\alpha_1, \dots, \alpha_k$, and then defining $\varphi_k(x_{1,1}, x_{1,2},
	x_1, \ldots, x_{k,1}, x_{k,2}, x_k) := \alpha_k(x_{1,1}, x_{1,2},
	x_1, \ldots, x_{k,1}, x_{k,2}, x_k)$:
	\begin{eqnarray*}
		\alpha_1 &:=& \neg \rmsr(x_{1,1}, x_{1, 2}, x_1);\\
		\alpha_{2 \ell} &:=& (\alpha_{2 \ell - 1} \wedge \rmsr(x_{2\ell,1}, x_{2\ell, 2}, x_{2\ell})); \\
		\alpha_{2 \ell + 1} &:=& (\alpha_{2 \ell} \vee \neg \rmsr(x_{2\ell+1,1}, x_{2\ell+1, 2}, x_{2\ell+1})).
	\end{eqnarray*}
	For example, we have that:
	\begin{eqnarray*}
		\alpha_{2} & = & (\neg \rmsr(x_{1,1}, x_{1, 2}, x_1) \wedge \rmsr(x_{2,1}, x_{2, 2}, x_{2})) \\
		\alpha_{3} & = & (\neg \rmsr(x_{1,1}, x_{1, 2}, x_1) \wedge \rmsr(x_{2,1}, x_{2, 2}, x_{2})) \vee \neg \rmsr(x_{3,1}, x_{3, 2}, x_{3})
	\end{eqnarray*}
	Combining conditions \eqref{sat_reduction} and \eqref{lem:bh-foilpm} with the definition $\varphi_k(x_{1,1}, x_{1,2}, x_1, \ldots, x_{k,1}, x_{k,2}, x_k)$, we conclude that $(\psi_1, \ldots, \psi_k) \in \sat(k)$ if and only if $\T \models \varphi_k(\es_{1,1}, \es_{1,2}, \es_1,
	\ldots, \es_{k,1}, \es_{k,2}, \es_k)$. Given that $\sat(k)$ is $\bh_k$-complete and that the
	decision tree $\T$ and the partial instances $\es_{1,1}$, $\es_{1,2}$,
	$\es_1$, $\ldots$, $\es_{k,1}$, $\es_{k,2}$, $\es_k$ can be
	constructed in polynomial time in the size of $(\psi_1, \ldots,
	\psi_k)$, we conclude that \logicEvaluation$(\varphi_k, \dt)$ is $\bh_k$-hard. This completes the proof of the theorem.
\end{proof}

\subsection{On the expressiveness of \texorpdfstring{$\ffoil$}{Lg}}
\label{expressive-f-foil}
$\ffoil$ allows us to express in a simple way the basic notions of explainability studied in this paper. Moreover, its evaluation problem is tractable given access to $\sat$ solvers.
\autoref{fig:formulas} shows how all the queries defined in Section~\ref{sec:queries} can be expressed in the $\ffoil$ logic. Just for clarity we use unguarded universal quantifiers, which are not allowed according to the $\ffoil$ syntax, because in these cases they can be rewritten as negations of unguarded existential quantifiers, as we discussed in Section~\ref{sec-fflogic}. We also make use of some auxiliary predicates defined in the appendix (refer to Section~\ref{app:def-aux}).

Probably the most complicated formula of this section is the one used to express the query of relevant feature $\rfeat(x, y)$. The idea there is to guess a weak abductive explanation $w$ containing the assigned feature under consideration and verify that undefining the feature makes $w$ lose the property of weak abductiveness. In fact, suppose first that $y$ is indeed a relevant feature for $x$. Then some abductive explanation $w$ contains $y$. Because $w$ is minimal, undefining $y$ from $w$ cannot produce another weak abductive explanation, so the witness pair $w, z$ exists. Conversely, if $y$ is not a relevant feature for $x$, then it cannot be contained in any minimal abductive explanation for $x$. Assume toward a contradiction that there exist witnesses $w$ and $z$ for the formula, where $w$ is a weak abductive explanation containing $y$, $z$ is obtained from $w$ by undefining $y$, and $z$ is not a weak abductive explanation. Now let $u \subseteq w$ be an abductive explanation. As $y$ is not relevant, $u$ cannot contain $y$, and therefore $u \subseteq z$. Since every completion of $z$ is also a completion of $u$, it would follow that $z$ is a weak abductive explanation, a contradiction. Hence no such witness pair $w, z$ can exist.


	\begin{figure}[ht!]
		\normalsize
			\begin{align*}
				\sr(x, y) = \ & \full(x) \wedge y \subseteq x \wedge (\allpos(x) \to \allpos(y)) \wedge (\allneg(x) \to \allneg(y))\\[\ecrowlen]
				\minsr(x, y) = \ & \sr(x, y) \wedge \forall z\ [z \subset y \to \neg \sr(x, z)]\\[\ecrowlen]
				\msr(x, y) = \ & \sr(x, y) \wedge \forall z\ [z \lnel y \to \neg \sr(x, z)]\\[\ecrowlen]
				\wcx(x, y) = \ & \full(x) \wedge y \subseteq x \wedge (\allpos(x) \to \neg \allpos(y)) \wedge (\allneg(x) \to \neg \allneg(y))\\[\ecrowlen]
				\cx(x, y) = \ & \wcx(x, y) \wedge \forall z\ [y \subset z \to \neg \wcx(x, z)]\\[\ecrowlen]
				\mcx(x, y) = \ & \wcx(x, y) \wedge \forall z\ [y \lnel z \to \neg \wcx(x, z)]\\[\ecrowlen]
				\mcr(x, y) = \ & \full(x) \wedge \full(y) \wedge \neg (\allpos(x) \leftrightarrow \allpos(y)) \ \wedge \\
				& \hspace{10.1em}  \forall z\ \big( [\full(z) \wedge \neg (\allpos(x) \leftrightarrow \allpos(z))] \to \leh(x, y, z) \big) \\[\ecrowlen]
				\mca(x, y) = \ & \full(x) \wedge \full(y) \wedge (\allpos(x) \leftrightarrow \allpos(y)) \ \wedge\\
				& \hspace{10.1em} \forall z\ \big( [\full(z) \wedge (\allpos(x) \leftrightarrow \allpos(z))] \to \leh(x, z, y) \big) \\[\ecrowlen]
				\nfeat(x, y) = \ &  {\sf Single}(y) \wedge \full(x) \wedge y \subseteq x \wedge \neg \exists z \ [\sr(x, z) \wedge \neg (y \subseteq z)]\\[\ecrowlen]
				\rfeat(x, y) = \ &  {\sf Single}(y) \wedge \full(x) \wedge y \subseteq x \wedge \exists w, z\ [\sr(x, w) \wedge \add(z, y, w) \wedge \neg \sr(x, z)]
			\end{align*}
		\caption[Formulas that express all queries in Subsection \ref{sec:queries} using $\ffoil$.]{Formulas that express all queries in Section \ref{sec:queries} using $\ffoil$.}
		\label{fig:formulas}
		\Description[ExplAIner formulas for Explainability queries.]
			{The figure lists ExplAIner formulas for the explainability queries defined in Section 2.2: weak abductive explanation, abductive explanation, minimum abductive explanation, weak contrastive explanation, contrastive explanation, maximum contrastive explanation, minimum change required, maximum change allowed, necessary feature, and relevant feature.}
	\end{figure}

It should be noted that if $\C$ is a class of models such that \textnormal{\logicEvaluation}$(\allpos(x), \C) \in \ptime$ but \textnormal{\logicEvaluation}$(\allneg(x), \C) \not\in \ptime$ (under standard complexity-theoretic assumptions),
then we still have that the problem \textnormal{\logicEvaluation}$(\psi, \C)$ is in the Boolean hierarchy for the $\ffoil$ formulas $\psi$ that do not mention the predicate $\allneg(x)$. Hence, the evaluation problem for this restricted fragment is still in the Boolean hierarchy, thus satisfying our criteria for an interpretability logic. This is the case, for example, for $\cnf$ formulas, for which validity checks can be done in polynomial time but checking unsatisfiability is $\conp$-complete. In this case, we can still express queries such as minimum change required and maximum change allowed, since they can be expressed in $\ffoil$ without mentioning the predicate $\allneg(x)$. Similarly, if $\C$ is a class of models such that \textnormal{\logicEvaluation}$(\allneg(x), \C) \in \ptime$ but \textnormal{\logicEvaluation}$(\allpos(x), \C) \not\in \ptime$ (for example, $\dnf$ formulas), then we still have that the evaluation problem is in the Boolean hierarchy for $\ffoil$ formulas $\psi$ that do not mention the predicate $\allpos(x)$.

Despite all the virtues of $\ffoil$, unfortunately it is in general not able to solve computation problems efficiently. In fact, note that the ability to tractably evaluate queries over concrete partial instances does not directly imply that positive answers to the query can be constructed efficiently. We devote the rest of the paper to addressing this problem and propose a third logic that resolves it.

%% file: new-sections/opt-foil.tex
Given an $\ffoil$ formula $\varphi(x, u_1, \ldots, u_k)$, we use the notation $\varphi[u_1, \ldots, u_k](x)$ to indicate that $x$ is a
distinguished variable and $u_1, \ldots, u_k$ are parameters that
define the possible values for $x$. In general, we use this syntax
when $x$ stores an explanation given an assignment for the
variables $u_1$, $\ldots$, $u_\ell$. For example, we write
$\varphi[u](x) = \axp(u, x)$ to indicate that $x$ is an abductive explanation given an assignment for the variable $u$ (that is, $x$ is an abductive explanation for $u$).

For each query $\varphi[x_1, \ldots, x_k](x)$ in $\ffoil$ and $\C$ a class of models, we define the computation problem $\comp(\varphi, \C)$ as follows:

\begin{center}
	\fbox{\begin{tabular}{rl}
			{\sc Problem:} & $\comp(\varphi[x_1, \ldots, x_k](x),\, \C)$\\
			{\sc Input:} & A model $\M \in \C$ and partial instances $\es_1, \ldots, \es_k$\\
			{\sc Output:} & Partial instance $\es$ such that $\M \models \varphi[\es_1, \ldots, \es_k](\es)$, \\
			& and \textsc{No} if no such partial instance exists
	\end{tabular}}
\end{center}

\subsection{The computational drawback of~\ffoil}
\label{sec-qdtfoil}

We proved in Section~\ref{sec-fflogic} that the $\ffoil$ logic admits tractable evaluations over adequate classes of models, which allows us to check if a partial instance is an answer for some explainability query. The next step in the study of $\ffoil$ is to establish the complexity of actually computing such answers. Unfortunately, the following result tells us that this problem cannot be solved with a polynomial number of calls to an $\np$ oracle, showing an important limitation of $\ffoil$.

\begin{theorem}
	\label{thm:comp-ffoil}
	There exists an \textnormal{\ffoil} formula $\varphi[y, z](x)$ such that \textnormal{$\comp(\varphi[y, z](x),\, \dt) \not \in \fptime^{\np}$} unless $\ph$ collapses to $\ptime^{\np}$.
\end{theorem}

\begin{proof}
	Consider the following \textnormal{\ffoil} formula:
	$$\varphi(x, y, z) \ := \ {\sf MaxRel}(x, y) \wedge \neg \exists w \ [{\sf MaxRel}(w, z) \wedge x \subseteq w \wedge \allneg(w)],$$
	where ${\sf MaxRel}(x,y)$, defined during the proof of Theorem~\ref{thm:eval-folistar}, is a formula from the atomic layer of \textnormal{\ffoil} such that $\M \models {\sf MaxRel}(\es,\es')$ if and only if $\es_\bot = \es'_\bot$, i.e., the sets of undefined features in $\es$ and $\es'$ are the same. Now consider the formula $\varphi(y, z) := \exists x\ \varphi(x, y, z)$, which is a \textnormal{\foil} formula with predicates $\{\subseteq, \lel, \allpos, \allneg\}$. We will show that \logicEvaluation$(\varphi, \dt)$ is $\np^{\np}$-hard by a Karp reduction from the following well-known $\np^{\np}$-complete problem (see \cite{arora2006computational} for a reference): given a propositional formula $\alpha(\textbf{p}, \textbf{q})$ in $\dnf$, where $\textbf{p}$ and $\textbf{q}$ are sets of variables, decide if the quantified propositional formula $\exists \textbf{p} \forall \textbf{q}\ \alpha(\textbf{p}, \textbf{q})$ is true. We will describe a polynomial-time reduction that constructs a decision tree $\T_{\alpha}$ and partial instances $\es_{\alpha}, \es_{\alpha}'$ such that $\exists \textbf{p} \forall \textbf{q}\ \alpha(\textbf{p}, \textbf{q})$ is true if and only if $\T_{\alpha} \models \varphi(\es_{\alpha}, \es_{\alpha}')$.

	Let $n$ be the number of terms in $\alpha$. We construct a decision tree $\T_{\alpha}$ of dimension $n + |\textbf{p}| + |\textbf{q}|$. We partition the features of $\T_{\alpha}$ into three consecutive blocks $H$, $E$ and $U$, where $|H| = n$, $|E| = |\textbf{p}|$, and $|U| = |\textbf{q}|$. The block $H$ will be used to select a term from $\alpha$, while $E$ and $U$ will encode truth assignments to the existential $\textbf{p}$ and to the universal variables $\textbf{q}$, respectively.

	\begin{figure}[ht!]
		\begin{center}
			\includegraphics[width=0.6\textwidth]{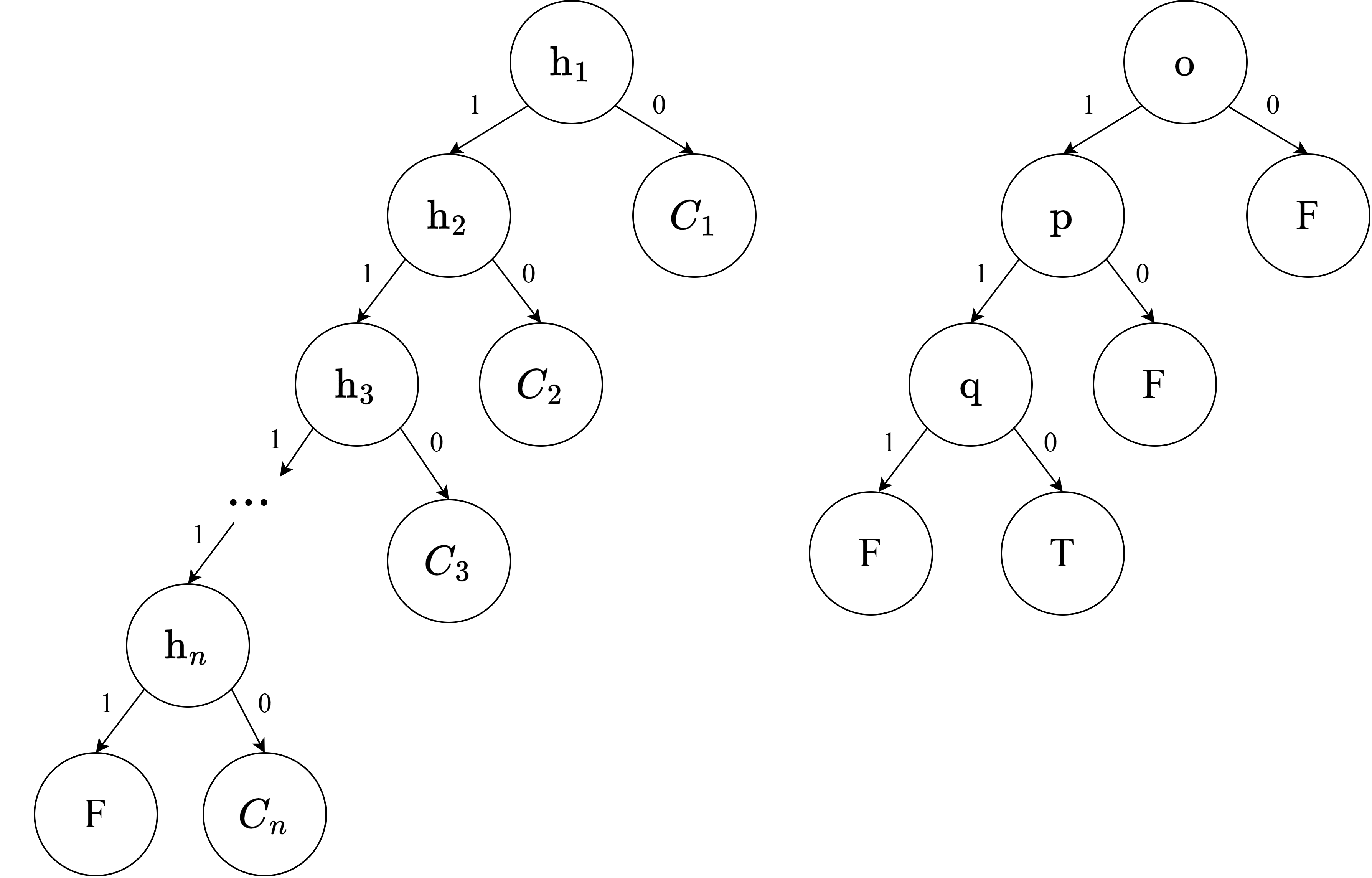}
			\caption{Construction used in the proof of Theorem~\ref{thm:comp-ffoil}. On the left, the decision tree $\T_{\alpha}$. On the right, the gadget for the term $\textnormal{o} \wedge \textnormal{p} \wedge \neg \textnormal{q}$.}
			\label{fig:diagramcomp}
			\Description[Illustration of the construction of the decision tree T_alpha from a DNF formula, together with the gadget for one term.]
			{The figure has two panels. In the left panel, the tree starts at node h_1. Edge 0 leads to the term gadget C_1, while edge 1 continues to h_2. The same pattern repeats: from h_2, edge 0 leads to C_2 and edge 1 continues to h_3, and so on until h_n, where edge 0 leads to C_n and edge 1 leads to a false leaf. Thus, the h-variables are used to select one term gadget. In the right panel, the gadget for the term o and p and not q is shown. The root is o. Edge 0 leads to a false leaf and edge 1 leads to p. From p, edge 0 leads to a false leaf and edge 1 leads to q. From q, edge 0 leads to a true leaf and edge 1 leads to a false leaf. Therefore, the gadget accepts exactly the assignments that satisfy the term o and p and not q.}
		\end{center}
	\end{figure}

	For every term $h_r$ ($1 \leq r \leq n$) we construct a decision tree $C_r$ over $E$ and $U$ features in such a way that an input to that tree encoding a truth assignment reaches a $\true$ leaf if and only if the term $h_r$ evaluated over that assignment is $\true$.
	We now explain how to construct $\T_{\alpha}$. We use the feature corresponding to term $h_{1}$ as the root. For every $i < n$, the outgoing edge of $h_i$ labeled by $1$ is connected to $h_{i+1}$, and the outgoing edge of $h_n$ labeled by $1$ is connected to a $\false$ leaf. Also, for every $i$, we connect the outgoing edge labeled by $0$ of $h_i$ to a copy of the tree $C_i$.
	An example is shown in \autoref{fig:diagramcomp}.

	Set $\es_{\alpha} = \{\bot\}^{n} \cdot \{0\}^{|\textbf{p}|} \cdot \{\bot\}^{|\textbf{q}|}$ and $\es_{\alpha}' = \{\bot\}^{n} \cdot \{0\}^{|\textbf{p}|+|\textbf{q}|}$. We now show that the reduction is correct.

	First suppose that $\exists \textbf{p} \forall \textbf{q}\ \alpha(\textbf{p}, \textbf{q})$ is true. Let $\nu$ be a truth assignment for the variables $\textbf{p}$ such that for every truth assignment $\sigma$ for the variables $\textbf{q}$ it holds that $\alpha(\nu(\textbf{p}), \sigma(\textbf{q}))$ is true. Let $\es_{\alpha}''$ be a partial instance with just its $E$ features defined according to $\nu$. Notice that $\es_{\alpha}''$ and $\es_{\alpha}$ have the same defined features. We claim that $\varphi(\es_{\alpha}'', \es_{\alpha}, \es_{\alpha}')$. In fact, let $\es_{\alpha}'''$ be a partial instance with the same defined features as $\es_{\alpha}'$ and such that $\es_{\alpha}'' \subseteq \es_{\alpha}'''$. Notice that $\es_{\alpha}'''$ naturally encodes a truth assignment $\sigma$ for the variables $\textbf{q}$ together with $\nu$. By taking a completion of $\es_{\alpha}'''$ that has a $0$ in the feature corresponding to a true term under the truth assignment $(\nu, \sigma)$ and a $1$ in the features corresponding to all previous terms, we can see that $\neg \allneg(\es_{\alpha}''')$. This shows that $\varphi(\es_{\alpha}'', \es_{\alpha}, \es_{\alpha}')$, and therefore $\T_{\alpha} \models \varphi(\es_{\alpha}, \es_{\alpha}')$.

	Now suppose that $\T_{\alpha} \models \varphi(\es_{\alpha}, \es_{\alpha}')$, so that there exists a partial instance $\es_{\alpha}''$ such that $\varphi(\es_{\alpha}'', \es_{\alpha}, \es_{\alpha}')$. Because $\es_{\alpha}''$ has the same defined features as $\es_{\alpha}$, we can define a truth assignment $\nu$ for the variables $\textbf{p}$ according to $\es_{\alpha}''$. Now let $\sigma$ be any truth assignment for the variables $\textbf{q}$. We claim that $\alpha(\nu(\textbf{p}), \sigma(\textbf{q}))$ is true. In fact, let $\es_{\alpha}'''$ be the partial instance with the same defined features as $\es_{\alpha}'$ and which corresponds to the pair $(\nu, \sigma)$. Because we have that $\varphi(\es_{\alpha}'', \es_{\alpha}, \es_{\alpha}')$, it must be the case that $\neg \allneg(\es_{\alpha}''')$. That means that there exists a completion of $\es_{\alpha}'''$ that is evaluated as $\true$ by $\T_{\alpha}$. That necessarily means that there is a term in $\alpha(\nu(\textbf{p}), \sigma(\textbf{q}))$ that is being satisfied. Hence $\exists \textbf{p} \forall \textbf{q}\ \alpha(\textbf{p}, \textbf{q})$ is true.

	This reduction shows that \logicEvaluation$(\varphi, \dt)$ is $\np^{\np}$-hard. To conclude the proof of the theorem, assume for the sake of contradiction that \textnormal{$\comp(\varphi[y, z](x),\, \dt) \in \fptime^{\np}$}. Then, it is clear that we would have \logicEvaluation$(\varphi, \dt) \in \ptime^{\np}$. Finally, $\ptime^{\np} = \np^{\np}$ implies that $\ph = \ptime^{\np}$.

\end{proof}

To solve the problem that Theorem~\ref{thm:comp-ffoil} signifies, we now propose $\optfoil$, a logic that is defined by introducing a minimality operator over a subset of $\ffoil$. As we will show in the next section, $\optfoil$ meets all the criteria for an appropriate interpretability logic.

\subsection{The~\optfoil~logic}
Our aim is to capture the right subset of $\ffoil$ that meets all the criteria for an interpretability logic. For this, we will define a third logic called $\optfoil$. We will show that the computation problem for this logic can be solved in polynomial time with a polynomial number of calls to an $\np$ oracle.

We will say that a formula $\rho(x, y, v_1, \ldots, v_\ell)$ from the atomic layer of $\ffoil$ represents a \emph{strict partial order} if, for every natural number $n$ and assignment of partial instances of dimension $n$ to the variables $v_1$, $\ldots$, $v_\ell$, the resulting binary relation over the variables $x$ and $y$ is a strict partial order over the partial instances of dimension $n$. Formally, $\rho(x, y, v_1, \ldots,
v_\ell)$ represents a strict partial order if, for every $n \in \mathbb{N}$,
$$\mathfrak{B}_n \models \ \forall v_1 \cdots \forall v_\ell \, \big[ \forall x \, \neg \rho(x, x, v_1, \ldots, v_\ell) \wedge
\forall x \forall y \forall z \,\big(
(\rho(x, y, v_1, \ldots, v_\ell) \wedge \rho(y, z, v_1, \ldots, v_\ell)) \to \rho(x, z, v_1, \ldots, v_\ell)\big)\big]. $$

The variables $v_1$, $\ldots$, $v_\ell$ in the formula $\rho(x, y, v_1, \ldots, v_\ell)$ are considered as parameters that define a strict partial order. In fact, different assignments for these variables can give rise to different orders. Hence, we use the notation $\rho[v_1, \ldots, v_\ell](x, y)$ to make explicit the distinction between the parameters $v_1$,
$\ldots$, $v_\ell$ and the variables $x$, $y$ that are instantiated with partial instances. For example, the strict partial order determined by the subsumption relation is defined by the formula $\rho_1(x, y) = x \subset y$.

As a second example, consider the case where a certain feature must be disregarded when defining an order on partial instances (for instance, it is often undesirable to use the feature {\em gender} for comparisons). Such an order can be defined as follows. Notice that, with the appropriate values for the variables $v_1$ and $v_2$, the following formula checks whether the $i$-th feature of $x$ is undefined:
$$\nf(x, v_1, v_2) \ := \ \neg(v_1 \subseteq x) \wedge \neg(v_2 \subseteq x).$$
For example, if we are considering partial instances of dimension $5$ and we
need to check whether instance $x$ has value $\bot$ in the first feature, then
we can use the values $c_1 = (0, \bot, \bot, \bot, \bot)$ and
$c_2 = (1, \bot, \bot, \bot, \bot)$ for the variables $v_1$ and $v_2$,
respectively. Moreover, let
$$\pred(x, y) \ := \ x \subset y \wedge \neg \exists z \, (x \subset
z \wedge z \subset y)$$
be a formula that checks whether $x$ is a predecessor of $y$ under the
order $\subset$. Now define
$$\mathsf{Strip}[v_1,v_2](x, y) \ := \ (\nf(x,v_1,v_2) \wedge x=y) \ \vee \ (\neg \nf(x,v_1,v_2) \wedge \pred(y,x) \wedge \nf(y,v_1,v_2)).$$
Notice that, with the appropriate values for $v_1$ and $v_2$, $\mathsf{Strip}[v_1,v_2](x,y)$ holds if and only if $y$ is obtained from $x$ by undefining the distinguished feature when necessary.
Then, the following formula defines a strict partial order based on $\subset$ but that disregards the $i$-th feature when comparing partial instances:
$$\rho_2[v_1,v_2](x,y) \ := \ \exists x' \exists y' \ \bigl(\mathsf{Strip}[v_1,v_2](x,x')
	\, \wedge \, \mathsf{Strip}[v_1,v_2](y,y') \, \wedge \, x' \subset y' \bigr).$$

For example, $\rho_2[c_1, c_2](x, y)$ with constants $c_1$ and $c_2$
mentioned above defines a strict partial order that disregards the first
feature when comparing partial instances of dimension $5$.

Formulas from the atomic layer of $\ffoil$ representing strict partial orders will be
used in the definition of $\optfoil$. Hence, it is necessary to have an algorithm that verifies whether this condition is satisfied in order to have a decidable syntax for $\optfoil$. We will now prove that such an algorithm exists.

\begin{proposition}\label{prop-strict-po}
The problem of verifying, given a formula $\rho[v_1, \ldots, v_\ell](x, y)$ from the atomic layer of $\ffoil$, whether it represents a strict partial order can be solved in $2^{2^{\rm{poly}\big(|\rho| \cdot 3^{\rm{wd}(\rho)} \big)}}$ space, and, hence, in $2^{2^{2^{\rm{poly}\big(|\rho| \cdot 3^{\rm{wd}(\rho)} \big)}}}$ time.
\end{proposition}

\begin{proof}

	Let $\rho(x, y, v_1, \ldots, v_\ell)$ be an arbitrary formula from the atomic layer of $\ffoil$, that is, a formula over the vocabulary $\{\subseteq, \lel\}$. Then, we consider the following sentence:
	\begin{align*}
		\varphi \ := \ \forall v_1 \cdots \forall v_\ell \, \big[ \forall x \, \neg
		\rho(x, x, v_1, \ldots, v_\ell) \ \wedge \forall x \forall y \forall z
		\,\big( (\rho(x, y, v_1, \ldots, v_\ell) \wedge \rho(y, z, v_1, \ldots,
		v_\ell)) \to \rho(x, z, v_1, \ldots, v_\ell)\big)\big].
	\end{align*}
	Notice that $\varphi$ has width $\rm{wd}(\varphi) \leq \rm{wd}(\rho)+1$. By definition, determining if $\rho$ corresponds to a strict partial order is equivalent to checking if for every structure $\mathfrak{B}_n$ it holds that $\mathfrak{B}_n \models \varphi$. Thanks to the second part of Theorem~\ref{theo:ptime-atomic}, we know that this can be done in the stated space.

\end{proof}

Proposition~\ref{prop-strict-po} serves as a theoretical upper bound to prove that $\optfoil$ has a decidable syntax. Observe that if we restrict ourselves to formulas of bounded width, then the space complexity falls from triple exponential to double exponential in $|\rho|$ (which implies that the time complexity is triple exponential in $|\rho|$).
In practice, we expect formulas representing strict partial orders to be small and to have a simple structure, so we do not expect this theoretical high computational complexity to pose an actual implementation obstacle.

We now explain how strict partial orders will be used in the logic $\optfoil$.

Given a formula $\varphi[u_1, \ldots, u_k](x)$ from the quantified layer of \textnormal{\ffoil} and another formula $\rho[v_1, \ldots, v_\ell](y, z)$ from the atomic layer of \textnormal{\ffoil} that represents a strict partial order, an {\em $\optfoil$ formula} is an expression of the following form:
$$\Psi[u_1, \ldots, u_k, v_1, \ldots, v_\ell](x) \ = \ \minf[\varphi[u_1, \ldots, u_k](x), \rho[v_1, \ldots, v_\ell](y, z)].$$

Notice that $x$, $u_1, \ldots, u_k$, $v_1, \ldots, v_\ell$ are the free
variables of this expression, while the variables $y$, $z$ will be quantified. In particular, $u_1, \ldots, u_k$ are the parameters that define the notion of explanation, $v_1, \ldots, v_\ell$ are the parameters that define the strict partial order, and $x$ is a variable used to store an explanation that is minimal in the sense given by the strict partial order.
The semantics of $\Psi[u_1, \ldots, u_k, v_1, \ldots, v_\ell](x)$ is defined by considering the following \textnormal{\ffoil} formula:
$$\theta_{\minf}(x,
u_1, \ldots, u_k, v_1, \ldots, v_\ell) \ := \ \varphi(x, u_1, \ldots, u_k) \ \wedge \
\neg \exists y \ \big(\varphi(y, u_1, \ldots, u_k) \, \wedge \, \rho(y, x, v_1, \ldots, v_\ell)\big).$$
More precisely, given a model $\M$ of dimension $n$ and partial
instances $\es$, $\es'_1, \ldots, \es'_k$, $\es''_1, \ldots, \es''_\ell$ of
dimension $n$, we define that $\M \models \Psi[\es'_1, \ldots, \es'_k, \es''_1,
\ldots, \es''_\ell](\es)$ if and only if $\M \models \theta_{\minf}(\es,
\es'_1, \ldots, \es'_k, \es''_1, \ldots, \es''_\ell)$.

\input{new-sections/comp-optdtfoil}

We now begin the study of the expressiveness of $\optfoil$.
As is customary, we say that a logic $\mathcal{L}_1$ is {\em contained} in a
logic $\mathcal{L}_2$ if for every formula in $\mathcal{L}_1$ there
exists an equivalent formula in $\mathcal{L}_2$. Moreover,
$\mathcal{L}_1$ is {\em properly contained} in $\mathcal{L}_2$ if
$\mathcal{L}_1$ is contained in $\mathcal{L}_2$ and
$\mathcal{L}_2$ is not contained in $\mathcal{L}_1$. The following proposition shows that the
expressive power of $\optfoil$ is less than that of $\ffoil$, which in turn has less expressive power than $\foil$ with predicates $\{\subseteq, \lel, \allpos, \allneg\}$.
\begin{proposition}\label{prop-ep-optdtfoil}
Assuming that the polynomial hierarchy does not collapse,
\textnormal{\optfoil} is strictly contained in \textnormal{\ffoil}, and \textnormal{\ffoil} is strictly contained in \textnormal{\foil} with extended predicates.
\end{proposition}

\begin{proof}
	For the first containment, let $$\Psi[u_1, \ldots, u_k, v_1, \ldots, v_\ell](x) = \minf[\varphi[u_1, \ldots, u_k](x), \rho[v_1, \ldots, v_\ell](y, z)]$$ be an $\optfoil$ formula. As we discussed before, we can consider the equivalent $\ffoil$ formula $$\theta_{\minf}(x, u_1, \ldots, u_k, v_1, \ldots, v_\ell).$$
	Now, for the sake of contradiction, suppose that the containment is not strict. Let $\varphi[y, z](x)$ be an $\ffoil$ formula such that \textnormal{$\comp(\varphi[y, z](x),\, \dt) \not \in \fptime^{\np}$} unless $\ph$ collapses to $\ptime^{\np}$ (whose existence is guaranteed by Theorem~\ref{thm:comp-ffoil}). Let $\Psi[y, z](x)$ be its equivalent expression in $\optfoil$.
	By Theorem~\ref{theo-comp-optdtfoil}, the problem~\logicComputation$(\Psi[y, z](x), \dt)$ is in \fpnp. Thus, \logicComputation$(\varphi[y, z](x), \dt)$ can be solved in \fpnp. This would imply the collapse of the polynomial hierarchy to $\ptime^{\np}$. We conclude that the containment is strict.

	For the second containment, each formula in $\ffoil$ is a $\foil$ formula with extended predicates by definition. It is strict because we can express $\stp$-hard problems in $\foil$ over decision trees (Theorem \ref{thm:eval-folistar}), but no $\ffoil$ formula can express a $\stp$-hard evaluation problem over decision trees unless $\stp \subseteq \bh$ (Theorem \ref{theo:thm:eval-ffoil}).
\end{proof}
The logic $\optfoil$ allows us to express in a simple way all notions
of explainability that we study in this paper. For example, recall from Section~\ref{sec-quant} that $\sr(u,x)$ can be expressed as a formula from the quantified layer of \textnormal{\ffoil}. Therefore, taking $\varphi[u](x) = \sr(u,x)$, the following $\optfoil$ formulas encode the notions of minimal and minimum abductive explanations:
\begin{eqnarray*}
\minsr[u](x) &=& \minf[\varphi[u](x), y \subset z],\\
\msr[u](x) &=& \minf[\varphi[u](x), y \lnel z ].
\end{eqnarray*}
Likewise, $\minf[\varphi[u](x), \rho_2[v_1,v_2](y, z)]$
encodes the notion of abductive explanations for the order
$\rho_2[v_1,v_2](y, z)$ that disregards a feature. The different variants of contrastive explanations can be expressed similarly.

As a second example, consider the notion of minimum change
required and the predicate $\led$ defined in
Section \ref{app:def-aux}.  Then, taking $$\varphi[u](x)
= \full(u) \wedge \full(x) \wedge \neg
(\allpos(u) \leftrightarrow \allpos(x))$$ and $\rho_3[u](y,z)
= \led(u,y,z) \wedge \neg \led(u,z,y)$, we can express the notion of
minimum change required in $\optfoil$ as follows:
\begin{align*}
	\mcr[u](x) = \minf[\varphi[u](x), \rho_3[u](y,z)].
\end{align*}
By reversing the order, the logic $\optfoil$ can also be used to express notions of
explainability that involve maximality conditions. For example, consider the query of maximum change allowed that asks for the maximum number of changes that can be made to an instance without
changing the output of the classification model.  Taking
$$\varphi[u](x) = \full(u) \wedge \full(x) \wedge (\allpos(u)
\leftrightarrow \allpos(x))$$ and defining the reverse order $\rho_4[u](y,z)
= \rho_3[u](z,y)$, we can express the notion of
maximum change allowed in $\optfoil$ as follows:
\begin{align*}
	\mca[u](x) = \minf[\varphi[u](x), \rho_4[u](y,z)].
\end{align*}
An important advantage of $\optfoil$ is that it allows for the combination of explainability notions.
For example, given
two instances $u_1$ and $u_2$ of the same dimension, consider the query $\csr[u_1, u_2](x) = \sr(u_1,x) \wedge
\sr(u_2,x)$ that checks whether $x$ is a common weak abductive explanation for $u_1$ and $u_2$. Then the following $\optfoil$ formula computes a common weak abductive explanation
for two instances (if such an explanation exists):
\begin{align*}
\Psi_1[u_1,u_2](x) = \minf[\csr[u_1, u_2](x), y \subset z].
\end{align*}
Note that an answer to this query is not necessarily minimal with respect to all weak abductive explanations for either $u_1$ or $u_2$.

Finally, another advantage of $\optfoil$ is that it allows for the exploration of the
space of explanations for a given classification. For example, assume that we already have an abductive explanation $u_1$ for an instance $u$, which can be computed using the $\optfoil$ formula $\minsr[u](x)$. Our aim is to compute a second abductive explanation
$u_2$ for $u$. Consider the formula:
\begin{align*}
\nsr[u,u_1](x) = \sr(u, x) \wedge \sr(u, u_1) \wedge \neg (u_1 \subseteq x).
\end{align*}
This formula checks whether $x$ is a weak abductive explanation for $u$ that does not subsume the abductive explanation $u_1$. Thus, an abductive explanation for the instance $u$ that is different from $u_1$ can be computed using the following $\optfoil$ formula:
\begin{align*}
\Psi_2[u,u_1](x) = \minf[\nsr[u, u_1](x), y \subset z].
\end{align*}
We can apply the same idea to other notions of explanation, such as the $\mcr$ explainability query, in order to compute multiple explanations for the output of a classification model.



%% file: new-sections/comp-optdtfoil.tex
The computation problem for $\optfoil$ has to be defined considering
the different roles of the variables in the formula
$\Psi[u_1, \ldots, u_k, v_1, \ldots, v_\ell](x)$. In particular, the
parameters $u_1, \ldots, u_k, v_1, \ldots, v_\ell$ should be given as
input, while the value of $x$ is the explanation to be computed. The
following definition takes these considerations into account. As usual,
we write $\C$ to denote some class of models.
\begin{center}
\fbox{
    \begin{tabular}{rl}
            {\sc Problem:} & \logicComputation$(\Psi[u_1, \ldots, u_k, v_1, \ldots, v_\ell](x), \C)$\\
            {\sc Input:} & A model $\M \in \C$ and partial
            instances  $\es'_1$, $\ldots$, $\es'_k$, $\es''_1$, $\ldots$, $\es''_\ell$\\
            {\sc Output:} & Partial instance
            $\es$ such that $\M \models \Psi[\es'_1, \ldots, \es'_k, \es''_1, \ldots, \es''_\ell](\es)$, \\ & and \textsc{No} if no
            such partial instance exists
    \end{tabular}
        }
\end{center}
We now show that $\optfoil$ fulfills our criteria by establishing that the computation problem for $\optfoil$ can be solved in polynomial time with a polynomial number of calls to an $\np$ oracle:

\begin{theorem}\label{theo-comp-optdtfoil}
Let $\C$ be a class of models such that \textnormal{\logicEvaluation}$(\allpos(x), \C) \in \ptime$ and \textnormal{\logicEvaluation}$(\allneg(x), \C) \in \ptime$. Then \logicComputation$(\Psi, \C) \in \fptime^{\np}$ for every formula $\Psi$ in \textnormal{\optfoil}.
\end{theorem}

As a corollary of this result, we obtain that $\optfoil$ can be used to compute explanations in polynomial time using
a polynomial number of calls to an $\np$ oracle for the class of decision trees. Moreover, the same holds for
more expressive representation classes, including $\ddnnf$ circuits and fragments of the class of circuits
corresponding to propositional formulas in conjunctive normal form ($\cnf$) whose satisfiability is
decidable in polynomial time, such as 2-CNF formulas ($\tcnf$) and Horn CNF formulas ($\horn$).
We formally state these results in the following corollary.

\begin{corollary}
	\begin{sloppypar}
Let $\varphi$ be an \textnormal{\optfoil} formula. Then \textnormal{\logicComputation}$(\varphi, \ddnnf) \in \fptime^{\np}$,
\textnormal{\logicComputation}$(\varphi, \tcnf) \in \fptime^{\np}$, and
\textnormal{\logicComputation}$(\varphi, \horn) \in \fptime^{\np}$.
	\end{sloppypar}
\end{corollary}

\begin{proof}[Proof of Theorem \ref{theo-comp-optdtfoil}]
	Let $\rho[v_1, \ldots, v_\ell](x,y)$ be a formula from the atomic layer of \textnormal{\ffoil} that represents a strict partial order. We say that a sequence $(\es_1, \ldots, \es_k)$ of partial instances of dimension $n$ is a \emph{path} of dimension $n$ in $\rho[v_1, \ldots, v_\ell](x,y)$ if there exist partial instances $\es'_1$, $\ldots$, $\es'_\ell$ of dimension $n$ such that, for every $i \in \{1, \ldots, k-1\}$, it holds that $$\mathfrak{B}_n \models \rho[\es'_1, \ldots, \es'_\ell](\es_i,\es_{i+1}).$$

	The following lemma shows that, for a fixed formula $\rho[v_1, \ldots, v_\ell](x,y)$, the lengths of such paths are polynomially bounded with respect to $n$ (refer to Section~\ref{proof:fixed-rho} for the proof).
	\begin{lemma}\label{lem-fixed-rho}
		Let $\rho[v_1, \ldots, v_\ell](x,y)$ be a formula from the atomic layer of \textnormal{\ffoil} that represents a strict partial order. Then there exists a fixed polynomial $p$ such that for every path $(\es_1, \ldots, \es_k)$ of dimension $n$ in
		$\rho[v_1, \ldots, v_\ell](x,y)$, it holds that $k \leq p(n)$.
	\end{lemma}

	Lemma \ref{lem-fixed-rho} gives us a simple algorithm to compute a solution for an
	$\optfoil$ formula
	\begin{align*}
		\minf[\varphi[u_1, \ldots, u_k](x), \rho[v_1, \ldots, v_\ell](y, z)],
	\end{align*}
	given as input a model $\M \in \C$ of dimension $n$ and partial instances $\es'_1$, $\ldots$, $\es'_k$, $\es''_1$, $\ldots$, $\es''_\ell$ of dimension $n$. We first use an $\np$ oracle to verify whether $\M \models \exists x \, \varphi[\es'_1, \ldots, \es'_k](x)$, which is a formula from the quantified layer of \textnormal{\ffoil}. If $\M \not\models \exists x \, \varphi[\es'_1, \ldots, \es'_k](x)$, then the answer is \textsc{No}. Otherwise, the $\np$
	oracle can be used to construct an initial partial instance $\es_0$ such that
	$\M \models \varphi[\es'_1, \ldots, \es'_k](\es_0)$. The idea is to maintain a current partial assignment $\es''$ (originally set to $\{\bot\}^n$) of the features that is known to extend to some witness. For each feature, we query whether there exists a witness extending $\es''$ but with that feature fixed to $0$. If the answer is positive, we keep that feature as $0$ in $\es''$, otherwise we query whether there exists one extending $\es''$ but with that feature fixed to $1$. If that answer is positive, we keep that feature as $1$, and if both answers are negative, then we leave the feature undefined. This way, the invariant is preserved at every step, and after at most $2n$ oracle queries we obtain the partial instance $\es_0$.

	We then use the $\np$ oracle to verify whether
	\begin{align*}
		\M \ \models \ \exists x \, \big(\varphi[\es'_1, \ldots, \es'_k](x) \ \wedge \ \rho[\es''_1, \ldots, \es''_\ell](x, \es_0)\big);
	\end{align*}
	which can be written as a formula from the quantified layer of \textnormal{\ffoil} by appending `$\wedge \, \rho[\es''_1, \ldots, \es''_\ell](x, \es_0)$' within all the quantifiers of $\varphi[\es'_1, \ldots, \es'_k](x)$. If the answer is positive, then again we use the $\np$ oracle as described before to construct a partial instance $\es_1$ such that
	\begin{align*}
		\M \ \models \ \varphi[\es'_1, \ldots, \es'_k](\es_1) \ \wedge \ \rho[\es''_1, \ldots, \es''_\ell](\es_1, \es_0).
	\end{align*}
	The algorithm continues in this way, constructing a sequence of
	partial instances $(\es_i, \es_{i-1}, \ldots, \es_0)$ that constitutes
	a path of dimension $n$ in $\rho[v_1, \ldots, v_\ell](y, z)$. The
	algorithm stops when the condition
	\begin{align*}
		\M \ \models \ \exists x \, \big(\varphi[\es'_1, \ldots, \es'_k](x) \wedge \rho[\es''_1, \ldots, \es''_\ell](x, \es_i)\big)
	\end{align*}
	does not hold, which by construction guarantees that $\es_i$ is a minimal
	instance. Lemma \ref{lem-fixed-rho} guarantees that $\es_i$ will be
	found in a polynomial number of steps. Since in each step we call the $\np$ oracle a polynomial number of times, this concludes the proof of the theorem.
\end{proof}

%% file: new-sections/conclusions.tex
We have proposed a declarative approach to model interpretability based
on query languages for explaining Boolean classification models. The
starting point of our work is the observation that the growing number of
explanation notions studied in formal XAI calls for a uniform language in
which such notions can be specified, combined, and analyzed. This view is
natural from a data management perspective: explanation notions become
queries, models become the structures over which these queries are
evaluated, and the main questions are those of expressiveness,
evaluation complexity, and computation of answers.

Our first contribution was to revisit $\foil$ from this perspective. We
showed that, despite its foundational role, $\foil$ is not well suited as
a practical query language for explanations. On the one hand, it cannot
express some central optimality-based notions, such as minimum abductive
explanations, even over decision trees. On the other hand, its evaluation
problem over decision trees is hard for every level of the polynomial
hierarchy. These results show that a useful explainability language must
carefully balance expressive power with controlled evaluation complexity.

To address this challenge, we introduced \ffoil, a layered query language
with an extended vocabulary for reasoning about partial instances and the
behavior of Boolean models. We showed that \ffoil\ can express a broad
family of explanation notions, including abductive, contrastive,
feature-based, and distance-based queries. At the same time, we proved
that the evaluation problem for each fixed \ffoil\ query belongs to the
Boolean hierarchy over every class of Boolean models for which the
predicates $\allpos$ and $\allneg$ can be evaluated in polynomial time.
This condition holds not only for decision trees, but also for more
general representation classes such as deterministic and decomposable
Boolean circuits. 

We also introduced \optfoil, an optimization-oriented fragment of
\ffoil\ for computing explanations that are minimal with respect to
strict partial orders. This fragment captures a wide range of
optimality-based explanation tasks while retaining controlled
computational behavior: under the same assumptions on $\allpos$ and
$\allneg$, explanations specified in \optfoil\ can be computed in
$\mathrm{FP}^{\mathrm{NP}}$. Together, the results for \ffoil\ and
\optfoil\ show that declarative specification and complexity-theoretic
analysis can provide a principled foundation for model interpretability.

Several directions remain open. A first direction is to extend the
framework beyond Boolean classification models. Although Boolean models
are standard in formal XAI and already capture many explanation tasks,
many applications involve multi-class outputs, non-Boolean features, or
structured feature domains. It would be interesting to understand which
parts of the present framework extend directly to these richer settings,
and which additional predicates or language constructs are needed.

A second direction is to study further model representations. Our upper
bounds are stated for every class of Boolean models over which
$\allpos$ and $\allneg$ can be evaluated in polynomial time, and this
already includes decision trees and deterministic decomposable Boolean
circuits. A natural next step is to identify additional representation
classes that satisfy this condition. This would help clarify the
connection between explainability languages and knowledge compilation
more broadly.

A third direction concerns query optimization. One of the motivations for
a declarative language is that different explanation notions can share
common subqueries and operators. This suggests the possibility of
developing optimization techniques for explainability queries, in the
same spirit as query optimization in databases. Such techniques could
exploit common subformulas, reuse calls to procedures for $\allpos$ and
$\allneg$, or identify fragments with better evaluation strategies.

Finally, it would be valuable to study richer answer mechanisms for
explainability queries. In this paper, explanations are treated as
partial instances satisfying a logical specification, possibly optimized
with respect to a strict partial order. However, users may require
different levels of detail, multiple alternative explanations, or
rankings of explanations according to several criteria. Extending the
language with principled mechanisms for enumeration, ranking, and
comparison of explanations is an important step toward a more complete
declarative framework for model interpretability.

%% file: new-sections/acknowledgements.tex
Part of this work has been funded by ANID - Millennium Science Initiative Program - Code ICN17002.
Diego Bustamante was partially funded by ANID - Subdirección de Capital Humano (Magíster Nacional, 2023, folio 22231282).
Mar\'ia Alejandra Schild was financially supported by ANID (Doctorado Nacional, 2025, folio 21251617).
Bernardo Subercaseaux is (partially) supported by the DARPA expMath program through the DARPA CMO contract number HR0011262E028.

%% file: new-sections/app-proofs.tex
	\subsection{Proof of Lemma~\ref{lemma:pointed}}
	\label{proof:pointed}

	We will first prove an auxiliary result. We start by introducing some terminology.

	Let $U = \{a_i \mid i > 0\}$ be a countably infinite set.
		We take a disjoint copy $\overline U = \{\overline a_i \mid i > 0\}$ of $U$.
		For an $X \subseteq U \cup \overline U$, we define
		\begin{align*}
			X_{U \setminus \overline U} \ & := \ \{a \in U \mid a \in X \text{ and } \overline a \not\in X\} \\
			X_{U \cap \overline U} \ & := \ \{a \in U \mid a \in X \text{ and } \overline a \in X\} \\
			X_{\overline U \setminus U} \ & := \ \{\overline a \in \overline U \mid \overline a \in X \text{ and } a \not\in X\}
		\end{align*}
		The {\em $\ell$-type} of $X$, for $\ell \geq 0$, is the tuple
		$$\big(\min{\{\ell,|X_{U \setminus \overline U}|\}},\, \min{\{\ell,|X_{U \cap \overline U}|\}}, \,  \min{\{\ell,|X_{\overline U \setminus U}|\}}\big).$$
		We write $X \leftrightarrows_\ell X'$, for $X,X' \subseteq U \cup \overline U$, if $X$ and $X'$ have the same
		$\ell$-type.
		If $X \subseteq U \cup \overline U$, then $X$ is {\em well formed} (wf) if for each $i > 0$ at most one element from $\{a_i,\overline a_i\}$ is in $X$.

		\begin{lemma}
			\label{lemma:double}
			Assume that $X \leftrightarrows_{3^{k+1}} Y$, for $X,Y \subseteq U \cup \overline U$ and $k \geq 0$. Then:
			\begin{itemize}
				\item
				For every wf $X_1 \subseteq X$, there exists a wf $Y_1 \subseteq Y$
				such that
				$$X_1 \leftrightarrows_{3^k} Y_1 \ \ \text{ and } \ \ X\setminus X_1 \leftrightarrows_{3^k} Y \setminus Y_1.$$
				\item
				For every wf $Y_1 \subseteq Y$, there exists a wf $X_1 \subseteq X$
				such that
				$$X_1 \leftrightarrows_{3^k} Y_1 \ \ \text{ and } \ \ X\setminus X_1 \leftrightarrows_{3^k} Y \setminus Y_1.$$
			\end{itemize}
		\end{lemma}

	\begin{proof}

		Given $Z \subseteq U$, we use $\overline Z$ to denote the set $\{ \overline a \in \overline U \mid a \in Z\}$, and given  $W \subseteq \overline U$, we use $\overline W$ to denote the set $\{ a \in U \mid \overline a \in W\}$.
	Let $X_1$ be a wf subset of $X$. Then we have that $X_1 = X_{1,1} \cup X_{1,2} \cup X_{1,3} \cup X_{1,4}$, where
	\begin{eqnarray*}
		X_{1,1} & \subseteq & X_{U \setminus \overline U},\\
		X_{1,2} & \subseteq & \{ a \in U \mid a \in X_{U \cap \overline U}\},\\
		X_{1,3} & \subseteq & \{ \overline a \in \overline U \mid a \in X_{U \cap \overline U}\},\\
		X_{1,4} & \subseteq & X_{\overline U \setminus U},
	\end{eqnarray*}
	and $\overline{X_{1,2}} \cap X_{1,3} = \emptyset$ (since $X_1$ is wf).
	We construct a set $Y_1 = Y_{1,1} \cup Y_{1,2} \cup Y_{1,3} \cup Y_{1,4}$
	by considering the following rules.
	\begin{enumerate}
		\item If $|X_{U \setminus \overline U}| < 3^{k+1}$, then $|Y_{U \setminus \overline U}|
		= |X_{U \setminus \overline U}|$ since $X \leftrightarrows_{3^{k+1}} Y$. In
		this case, we choose $Y_{1,1} \subseteq Y_{U \setminus \overline
			U}$ in such a way that $|Y_{1,1}| = |X_{1,1}|$ and $|Y_{U
			\setminus \overline U} \setminus Y_{1,1}| = |X_{U \setminus
			\overline U} \setminus X_{1,1}|$.

		If $|X_{U \setminus \overline U}| \geq 3^{k+1}$, then $|Y_{U \setminus \overline U}|
		\geq 3^{k+1}$ since $X \leftrightarrows_{3^{k+1}} Y$. In this
		case, we choose $Y_{1,1} \subseteq Y_{U \setminus \overline U}$ in
		the following way. If $|X_{1,1}| < 3^k$, then $|Y_{1,1}| =
		|X_{1,1}|$, and if $|X_{U \setminus \overline U} \setminus
		X_{1,1}| < 3^k$, then $|Y_{U \setminus \overline U} \setminus
		Y_{1,1}| = |X_{U \setminus \overline U} \setminus
		X_{1,1}|$. Finally, if $|X_{1,1}| \geq 3^k$ and $|X_{U \setminus
			\overline U} \setminus X_{1,1}| \geq 3^k$, then $|Y_{1,1}| \geq
		3^k$ and $|Y_{U \setminus \overline U} \setminus Y_{1,1}| \geq
		3^k$. Notice that we can choose such a set $Y_{1,1}$ since $|Y_{U
			\setminus \overline U}| \geq 3^{k+1}$.

		\item If $|X_{U \cap \overline U}| < 3^{k+1}$, then $|Y_{U
			\cap \overline U}| = |X_{U \cap \overline U}|$ since $X
		\leftrightarrows_{3^{k+1}} Y$. In this case, we choose $Y_{1,2}
		\subseteq \{ a \in U \mid a \in Y_{U \cap \overline U}\}$ and
		$Y_{1,3} \subseteq \{ \overline a \in \overline U \mid a \in Y_{U
			\cap \overline U}\}$ in such a way that
		$\overline{Y_{1,2}} \cap Y_{1,3} = \emptyset$,
		$|Y_{1,2}| = |X_{1,2}|$,
		$|Y_{1,3}| = |X_{1,3}|$ and
		$|Y_{U \cap \overline U} \setminus (Y_{1,2} \cup \overline{Y_{1,3}})| =
		|X_{U \cap \overline U} \setminus (X_{1,2} \cup \overline{X_{1,3}})|$.

		If $|X_{U \cap \overline U}| \geq 3^{k+1}$, then $|Y_{U
			\cap \overline U}| \geq 3^{k+1}$ since $X
		\leftrightarrows_{3^{k+1}} Y$. In this case, we choose $Y_{1,2}
		\subseteq \{ a \in U \mid a \in Y_{U \cap \overline U}\}$ and
		$Y_{1,3} \subseteq \{ \overline a \in \overline U \mid a \in Y_{U
			\cap \overline U}\}$ in the following way.
		\begin{enumerate}
			\item
			If $|X_{1,2}| < 3^k$, $|X_{1,3}| < 3^k$ and $|X_{U \cap
				\overline U} \setminus (X_{1,2} \cup \overline{X_{1,3}})| \geq
			3^k$, then $|Y_{1,2}| = |X_{1,2}|$, $|Y_{1,3}| = |X_{1,3}|$ and
			$\overline{Y_{1,2}} \cap Y_{1,3} = \emptyset$.
			Notice that we can choose such sets $Y_{1,2}$ and $Y_{1,3}$
			since $|Y_{U \cap \overline U}| \geq 3^{k+1}$.

			\item
			If $|X_{1,2}| < 3^k$, $|X_{1,3}| \geq 3^k$ and $|X_{U \cap
				\overline U} \setminus (X_{1,2} \cup \overline{X_{1,3}})| <
			3^k$, then $|Y_{1,2}| = |X_{1,2}|$,     $|Y_{U \cap \overline U} \setminus (Y_{1,2} \cup \overline{Y_{1,3}})| =
			|X_{U \cap \overline U} \setminus (X_{1,2} \cup \overline{X_{1,3}})|$ and
			$\overline{Y_{1,2}} \cap Y_{1,3} = \emptyset$.
			Notice that we can choose such sets $Y_{1,2}$ and $Y_{1,3}$
			since $|Y_{U \cap \overline U}| \geq 3^{k+1}$.

			\item
			If $|X_{1,2}| \geq 3^k$, $|X_{1,3}| < 3^k$ and $|X_{U \cap
				\overline U} \setminus (X_{1,2} \cup \overline{X_{1,3}})| <
			3^k$, then $|Y_{1,3}| = |X_{1,3}|$,     $|Y_{U \cap \overline U} \setminus (Y_{1,2} \cup \overline{Y_{1,3}})| =
			|X_{U \cap \overline U} \setminus (X_{1,2} \cup \overline{X_{1,3}})|$ and
			$\overline{Y_{1,2}} \cap Y_{1,3} = \emptyset$.
			Notice that we can choose such sets $Y_{1,2}$ and $Y_{1,3}$
			since $|Y_{U \cap \overline U}| \geq 3^{k+1}$.

			\item
			If $|X_{1,2}| < 3^k$, $|X_{1,3}| \geq 3^k$ and $|X_{U \cap
				\overline U} \setminus (X_{1,2} \cup \overline{X_{1,3}})| \geq
			3^k$, then $|Y_{1,2}| = |X_{1,2}|$, $|Y_{1,3}| \geq 3^k$, $|Y_{U \cap \overline U} \setminus (Y_{1,2} \cup \overline{Y_{1,3}})| \geq 3^k$
			and
			$\overline{Y_{1,2}} \cap Y_{1,3} = \emptyset$.
			Notice that we can choose such sets $Y_{1,2}$ and $Y_{1,3}$
			since $|Y_{U \cap \overline U}| \geq 3^{k+1}$.

			\item
			If $|X_{1,2}| \geq 3^k$, $|X_{1,3}| < 3^k$ and $|X_{U \cap
				\overline U} \setminus (X_{1,2} \cup \overline{X_{1,3}})| \geq
			3^k$, then $|Y_{1,3}| = |X_{1,3}|$, $|Y_{1,2}| \geq 3^k$, $|Y_{U \cap \overline U} \setminus (Y_{1,2} \cup \overline{Y_{1,3}})| \geq 3^k$
			and
			$\overline{Y_{1,2}} \cap Y_{1,3} = \emptyset$.
			Notice that we can choose such sets $Y_{1,2}$ and $Y_{1,3}$
			since $|Y_{U \cap \overline U}| \geq 3^{k+1}$.

			\item
			If $|X_{1,2}| \geq 3^k$, $|X_{1,3}| \geq 3^k$ and $|X_{U \cap
				\overline U} \setminus (X_{1,2} \cup \overline{X_{1,3}})| <
			3^k$, then $|Y_{U \cap \overline U} \setminus (Y_{1,2} \cup \overline{Y_{1,3}})| =
			|X_{U \cap \overline U} \setminus (X_{1,2} \cup \overline{X_{1,3}})|$, $|Y_{1,2}| \geq 3^k$, $|Y_{1,3}| \geq 3^k$ and
			$\overline{Y_{1,2}} \cap Y_{1,3} = \emptyset$.
			Notice that we can choose such sets $Y_{1,2}$ and $Y_{1,3}$
			since $|Y_{U \cap \overline U}| \geq 3^{k+1}$.

			\item
			If $|X_{1,2}| \geq 3^k$, $|X_{1,3}| \geq 3^k$ and $|X_{U \cap
				\overline U} \setminus (X_{1,2} \cup \overline{X_{1,3}})| \geq
			3^k$, then $|Y_{1,2}| \geq 3^k$, $|Y_{1,3}| \geq 3^k$,
			$|Y_{U \cap \overline U} \setminus (Y_{1,2} \cup \overline{Y_{1,3}})| \geq 3^k$ and
			$\overline{Y_{1,2}} \cap Y_{1,3} = \emptyset$.
			Notice that we can choose such sets $Y_{1,2}$ and $Y_{1,3}$
			since $|Y_{U \cap \overline U}| \geq 3^{k+1}$.

		\end{enumerate}

		\item If $|X_{\overline U \setminus U}| < 3^{k+1}$, then $|Y_{\overline U \setminus U}|
		= |X_{\overline U \setminus U}|$ since $X \leftrightarrows_{3^{k+1}} Y$. In
		this case, we choose $Y_{1,4} \subseteq Y_{\overline U \setminus
			U}$ in such a way that $|Y_{1,4}| = |X_{1,4}|$ and $|Y_{\overline U
			\setminus U} \setminus Y_{1,4}| = |X_{\overline U \setminus
			U} \setminus X_{1,4}|$.

		If $|X_{\overline U \setminus U}| \geq 3^{k+1}$, then $|Y_{\overline U \setminus U}|
		\geq 3^{k+1}$ since $X \leftrightarrows_{3^{k+1}} Y$. In this
		case, we choose $Y_{1,4} \subseteq Y_{\overline U \setminus U}$ in
		the following way. If $|X_{1,4}| < 3^k$, then $|Y_{1,4}| =
		|X_{1,4}|$, and if $|X_{\overline U \setminus U} \setminus
		X_{1,4}| < 3^k$, then $|Y_{\overline U \setminus U} \setminus
		Y_{1,4}| = |X_{\overline U \setminus U} \setminus
		X_{1,4}|$. Finally, if $|X_{1,4}| \geq 3^k$ and $|X_{\overline U \setminus
			U} \setminus X_{1,4}| \geq 3^k$, then $|Y_{1,4}| \geq
		3^k$ and $|Y_{\overline U \setminus U} \setminus Y_{1,4}| \geq
		3^k$. Notice that we can choose such a set $Y_{1,4}$ since $|Y_{\overline U
			\setminus U}| \geq 3^{k+1}$.
	\end{enumerate}
	By definition of $Y_{1,1}$, $Y_{1,2}$, $Y_{1,3}$ and $Y_{1,4}$, it is straightforward to conclude that
	$Y_1$ is wf, $X_1 \leftrightarrows_{3^{k}} Y_1$ and  $(X \setminus X_1) \leftrightarrows_{3^{k}} (Y \setminus Y_1)$.

	We have just proved that for every wf $X_1 \subseteq X$, there
	exists a wf $Y_1 \subseteq Y$ such that $X_1 \leftrightarrows_{3^k}
	Y_1$ and $X\setminus X_1 \leftrightarrows_{3^k} Y \setminus Y_1$.
	In the same way, it can be shown that for every wf $Y_1 \subseteq
	Y$, there exists a wf $X_1 \subseteq X$ such that $X_1
	\leftrightarrows_{3^k} Y_1$ and $X\setminus X_1
	\leftrightarrows_{3^k} Y \setminus Y_1$. This concludes the proof of the lemma.

	\end{proof}

	We now consider structures of the form
	$\astruct^* = \langle 2^{X},\subseteq^{\astruct^*} \rangle$, where $X \subseteq U \cup \overline U$ and
	$\subseteq^{\astruct^*}$ is the relation that contains all pairs $(Y,Z)$, for $Y,Z \subseteq X$, such that
	$Y \subseteq Z$. Given two structures $\astruct^*_1$ and $\astruct^*_2$ of this form, perhaps with constants,
	we write $\astruct^*_1 \equiv_k^{{\rm wf}} \astruct^*_2$ to denote that the Duplicator has a winning strategy in
	the $k$-round \EF game played on structures $\astruct^*_1$ and $\astruct^*_2$, but where Spoiler and Duplicator are forced to play wf subsets of $U \cup \overline U$ only.

	Consider structures $\astruct^*_1 = \langle 2^{X_1},\subseteq^{\astruct^*_1} \rangle$ and
	$\astruct^*_2 = \langle 2^{X_2},\subseteq^{\astruct^*_2} \rangle$
	of the form described above. We claim that, for every $k \geq 0$,

	\begin{equation}
		\label{eq:wf}
		X_1 \leftrightarrows_{3^k} X_2 \quad \Longrightarrow
		\quad \big(\astruct^*_1,(X_1 \cap U)\big) \ \equiv_k^{{\rm wf}} \  \big(\astruct^*_2,(X_2 \cap U)\big).
	\end{equation}

	Before proving the claim \eqref{eq:wf}, we explain how it implies Lemma \ref{lemma:pointed}. Take a structure
	of the form ${\astruct_n = \langle \{0,1,\bot\}^n,\subseteq^{\astruct_n}\rangle}$, where $\subseteq^{\astruct_n}$ is the subsumption relation
	over $\{0,1,\bot\}^n$. Take, on the other hand, the structure $\astruct_n^* = \langle 2^X,\subseteq^{\astruct_n^*} \rangle$,
	where $X = \{a_1,\dots,a_n,\bar a_1,\dots,\bar a_n\}$. It can be seen that there is an isomorphism $f$ between
	$\astruct_n$ and the substructure of $\astruct_n^*$ induced by the wf subsets of $X$. The isomorphism $f$ takes an instance $\es \in \{0,1,\bot\}^n$ and
	maps it to $Y \subseteq X$ such that for every $i \in \{1,\dots,n\}$, (a)
	if $\es[i] = 1$ then $a_i \in Y$, (b) if $\es[i] = 0$ then $\bar a_i \in Y$, and (c) if $\es[i] = \bot$ then neither $a_i$ nor $\bar a_i$ is in $Y$.
	By definition, the isomorphism $f$ maps the tuple $\{1\}^n$ in $\astruct_n$ to the set
	$X \cap U = \{a_1,\dots,a_n\}$ in $\astruct^*_n$.

	From claim \eqref{eq:wf}, it follows then that if $n,p \geq 3^k$ it is the case that $$(\astruct^*_n,\{a_1,\dots,a_n\}) \ \equiv_k^{{\rm wf}} \
	(\astruct^*_p,\{a_1,\dots,a_p\}).$$
	From our previous observations, this implies that
	$$(\astruct_n,\{1\}^n) \ \equiv_k \  (\astruct_p,\{1\}^p).$$
	We conclude, in particular, that $\astruct^+_n \equiv_k \astruct^+_p$, as desired.

	We now prove the claim in \eqref{eq:wf}.
	We do it by induction on $k \geq 0$. The base
	cases $k = 0$ and $k= 1$ are immediate. We now move to the induction case for $k+1$. Take structures
	$\astruct^*_1 = \langle 2^{X_1},\subseteq^{\astruct^*_1} \rangle$ and
	$\astruct^*_2 = \langle 2^{X_2},\subseteq^{\astruct^*_2} \rangle$
	of the form described above, such that $X_1 \leftrightarrows_{3^{k+1}} X_2$.
	Assume, without loss of generality, that for the first round the Spoiler picks the well formed element
	$X'_1 \subseteq X_1$ in the structure $\astruct^*_1$. From Lemma \ref{lemma:double}, there exists
	$X'_2 \subseteq X_2$ such that
	$$X'_1 \leftrightarrows_{3^k} X'_2 \ \ \text{ and } \ \ X_1\setminus X'_1 \leftrightarrows_{3^k} X_2 \setminus X'_2.$$
	By induction hypothesis, the following holds:
	\begin{eqnarray*}
		\big(\langle 2^{X'_1},\subseteq \rangle,(X'_1 \cap U)\big) & \equiv^{\rm wf}_k & \big(\langle 2^{X'_2},\subseteq \rangle,(X'_2 \cap U)\big)\\
		\big(\langle 2^{X_1 \setminus X'_1},\subseteq \rangle,((X_1 \setminus X'_1) \cap U)\big) & \equiv^{\rm wf}_k &
		\big(\langle 2^{X_2 \setminus X'_2},\subseteq \rangle,((X_2 \setminus X'_2) \cap U)\big).
	\end{eqnarray*}
	A simple composition argument allows to obtain the following from these two expressions:
	\begin{equation}
		\label{eq:final}
		\big(\langle 2^{X_1},\subseteq \rangle,(X_1 \cap U),X'_1\big) \ \equiv_k^{{\rm wf}} \ \big(\langle 2^{X_2},\subseteq \rangle,
		(X_2 \cap U),X'_2\big).
	\end{equation}
	This holds because $X'_1 = (X_1 \cap U)$ iff $X'_2 = (X_2 \cap U)$. In fact, assume that
	$X'_1 = (X_1 \cap U)$, so that $X'_1 \cap \overline U = \emptyset$. Since $X'_1 \leftrightarrows_{3^k} X'_2$, it follows that $X'_2 \subseteq X_2 \cap U$. On the other hand, $X_1 \setminus X_1' = X_1 \cap \overline U$. As $X_1 \setminus X_1' \leftrightarrows_{3^k} X_2 \setminus X'_2$, we conclude that $X_2 \setminus X'_2 \subseteq \overline{U}$, and hence $X_2 \cap U \subseteq X_2'$. Combining both inclusions, we obtain that $X'_2 = (X_2 \cap U)$. The other direction is completely analogous.

	But Equation \eqref{eq:final} is equivalent with the following fact:
	$$\big(\langle 2^{X_1},\subseteq \rangle,(X_1 \cap U)\big) \ \equiv_{k+1}^{{\rm wf}} \ \big(\langle 2^{X_2},\subseteq \rangle,
	(X_2 \cap U) \big).$$
	This finishes the proof of Lemma \ref{lemma:pointed}.

	\subsection{Proof of Lemma~\ref{lemma:comp}}
	\label{proof:comp}

	Let $\es_i$ and $\es'_i$ be the moves played by Spoiler and Duplicator in $\astruct_\M \oplus \astruct_{\M_1}$ and $\astruct_\M \oplus \astruct_{\M_2}$, respectively, for the first $i \leq k$ rounds of the \EF game
	\begin{align*}
		\big(\astruct_\M \oplus \astruct_{\M_1},\{1\}^{n+p},\{\bot\}^n \cdot \{1\}^p\big) \ \equiv_k \
		\big(\astruct_\M \oplus \astruct_{\M_2},\{1\}^{n+q},\{\bot\}^n \cdot \{1\}^q\big).
	\end{align*}
	We write $\es_i = (\es_{i1},\es_{i2})$ to denote that $\es_{i1}$ is the tuple formed by the first $n$ features of $\es_i$ and $\es_{i2}$ is the one formed
	by the last $p$ features of $\es_i$. Similarly, we write $\es'_i = (\es'_{i1},\es'_{i2})$ to denote that $\es'_{i1}$ is the tuple formed by the first $n$ features of $\es'_i$ and $\es'_{i2}$ is the one formed
	by the last $q$ features of $\es'_i$.

	The winning strategy for Duplicator is as follows. Suppose $i-1$ rounds have been played, and for round $i$
	the Spoiler picks element $\es_i \in \astruct_\M \oplus \astruct_{\M_1}$ (the case when he picks an element in $\astruct_\M \oplus \astruct_{\M_2}$ is symmetric).
	Assume also that $\es_i = (\es_{i1},\es_{i2})$. The duplicator then considers the position
	$$\big((\es_{12},\dots,\es_{(i-1)2}),(\es'_{12},\dots,\es'_{(i-1)2})\big)$$
	on the game $(\astruct_{\M_1},\{1\}^p) \equiv_k (\astruct_{\M_2},\{1\}^q)$, and finds his response $\es'_{i2}$ to $\es_{i2}$ in $\astruct_{\M_2}$.
	The Duplicator then responds to the Spoiler's move $\es_i \in \astruct_\M \oplus \astruct_{\M_1}$ by choosing the element $\es'_i = (\es_{i1},\es'_{i2}) \in \astruct_\M \oplus \astruct_{\M_2}$.

	Notice, by definition, that $\es_i = \{1\}^{n+p}$ iff $\es'_i =  \{1\}^{n+q}$. Similarly, $\es_i = \{\bot\}^n \cdot \{1\}^{p}$ iff $\es'_i =  \{\bot\}^n \cdot
	\{1\}^{q}$. Moreover, it is easy to see that playing in this way the Duplicator preserves the subsumption relation. Analogously, the strategy preserves the $\pos$ relation.
	In fact, $\es_i$ is a positive instance of $\astruct_\M \oplus \astruct_{\M_1}$ iff $\es_{i1}$ is a positive instance of $\astruct_\M$ or $\es_{i2}$ is a positive instance
	of $\astruct_{\M_1}$. By definition, the latter follows if and only if  $\es_{i1}$ is a positive instance of $\astruct_\M$ or $\es'_{i2}$ is a positive instance
	of $\astruct_{\M_2}$, which in turn is equivalent to $\es'_i$ being a positive instance of $\astruct_\M \oplus \astruct_{\M_2}$.

	We conclude that this is a winning strategy for the Duplicator, and hence that
	\begin{align*}
		\big(\astruct_\M \oplus \astruct_{\M_1},\{1\}^{n+p},\{\bot\}^n \cdot \{1\}^p\big) \ \equiv_k \
		\big(\astruct_\M \oplus \astruct_{\M_2},\{1\}^{n+q},\{\bot\}^n \cdot \{1\}^q\big).
	\end{align*}
	This finishes the proof of Lemma \ref{lemma:comp}.

	\subsection{Proof of Lemma~\ref{lemma:qbf}}
	\label{proof:qbf}

	To prove that \textnormal{$\Sigma_{k+1}$-\textsc{QBF}$(\dt)$} is in $\Sigma_{k}^{\text{P}}$, note that we can decide in polynomial time if a given decision tree encodes a tautology. Therefore, we can use a $\Sigma_k$-alternating Turing machine for guessing the values for the first $k$ quantifiers, we prune the decision tree according to those guesses, and then we solve the remaining universal quantifier directly.

	For the hardness, we use a reduction from the following standard $\Sigma_k^{\text{P}}$-hard problem (see \citep{MR3234575} for a reference):
	Given a 3CNF formula $\varphi$ over the set $X = \{x_1,\dots,x_m\}$ of propositional variables, is it the case that the expression $\psi = \exists X_1 \forall X_2 \cdots \exists X_k \, \varphi$ holds, where
	$X_1,\dots,X_k$ is a partition of $X$ in $k$ equivalence classes? Note that the hypothesis of $k$ being odd is important here because if the last quantifier were universal, we could solve it directly as in the case of decision trees.
	From $\psi$ we build in polynomial time a $\Sigma_{k+1}$-QBF $\alpha$ over $\M_\varphi$, where $\M_\varphi$ is a decision tree
	that can be
	built in polynomial time from $\varphi$, such that
	\begin{equation} \label{eq:holds}
		\text{$\psi$ holds} \quad \Longleftrightarrow \quad \text{$\alpha$ holds.}
	\end{equation}

	We now explain how to define $\M_\varphi$ from the CNF formula $\varphi$. Let
	$\varphi = C_1 \wedge \cdots \wedge C_n$ be a propositional formula,
	where each $C_i$ is a disjunction of three literals and does
	not contain repeated or complementary literals. Moreover, assume that
	$\{x_1, \ldots, x_m\}$ is the set of variables occurring in $\varphi$, and the proof will use partial instances of dimension $n+m$.
	Notice that the last $m$ features of such a partial instance $\es$ naturally define a truth assignment for the propositional formula $\varphi$. More precisely, for every $i \in \{1, \ldots, n\}$, we use notation $\es(C_i) = 1$ to indicate that there is a disjunct $\ell$ of $C_i$ such that $\ell = x_j$ and $\es[n+j] = 1$, or $\ell = \neg x_j$ and $\es[n+j] = 0$, for
	some $j \in \{1,\ldots,m\}$. Furthermore, we write $\es(\varphi) = 1$ if $\es(C_i) = 1$ for every $i \in \{1, \ldots, n\}$.

	For each clause $C_i$ ($i \in \{1, \ldots, n\}$), let $\M_{C_i}$ be
	a decision tree of dimension $n+m$ (but that will only use features $n+1, \ldots, n+m$) such that for every instance $\es$:
	$\M_{C_i}(\es) = 1$ if and only if $\es(C_i) = 1$.
	Notice that $\M_{C_i}$
	can be constructed in constant time as it only needs to contain at most
	eight paths of depth 3. For example, assuming that $C = (x_1 \vee x_2 \vee x_3)$,
	a possible decision tree $\M_C$ is depicted in the following figure:
	\begin{center}
		\begin{tikzpicture}[every node/.style={font=\footnotesize, scale=.8}]
			\node[circle,draw=black] (n) {$n+1$};
			\node[circle,draw=black,below left = 5mm and 28mm of n] (n0) {$n+2$};
			\node[circle,draw=black,below right = 5mm and 28mm of n] (n1) {$n+2$};
			\node[circle,draw=black,below left = 5mm and 10mm of n0] (n00) {$n+3$};
			\node[circle,draw=black,below right = 5mm and 10mm of n0] (n01) {$n+3$};
			\node[circle,draw=black,below left = 5mm and 10mm of n1] (n10) {$n+3$};
			\node[circle,draw=black,below right = 5mm and 10mm of n1] (n11) {$n+3$};
			\node[circle,draw=black,below left = 5mm and 5mm of n00] (n000) {$\false$};
			\node[circle,draw=black,below right = 5mm and 5mm of n00] (n001) {$\true$};
			\node[circle,draw=black,below left = 5mm and 5mm of n01] (n010) {$\true$};
			\node[circle,draw=black,below right = 5mm and 5mm of n01] (n011) {$\true$};
			\node[circle,draw=black,below left = 5mm and 5mm of n10] (n100) {$\true$};
			\node[circle,draw=black,below right = 5mm and 5mm of n10] (n101) {$\true$};
			\node[circle,draw=black,below left = 5mm and 5mm of n11] (n110) {$\true$};
			\node[circle,draw=black,below right = 5mm and 5mm of n11] (n111) {$\true$};

			\path[arrout] (n) edge node[above] {$0$} (n0);
			\path[arrout] (n) edge node[above] {$1$} (n1);
			\path[arrout] (n0) edge node[above] {$0$} (n00);
			\path[arrout] (n0) edge node[above] {$1$} (n01);
			\path[arrout] (n1) edge node[above] {$0$} (n10);
			\path[arrout] (n1) edge node[above] {$1$} (n11);
			\path[arrout] (n00) edge node[above] {$0$} (n000);
			\path[arrout] (n00) edge node[above] {$1$} (n001);
			\path[arrout] (n01) edge node[above] {$0$} (n010);
			\path[arrout] (n01) edge node[above] {$1$} (n011);
			\path[arrout] (n10) edge node[above] {$0$} (n100);
			\path[arrout] (n10) edge node[above] {$1$} (n101);
			\path[arrout] (n11) edge node[above] {$0$} (n110);
			\path[arrout] (n11) edge node[above] {$1$} (n111);
		\end{tikzpicture}
	\end{center}

	Moreover, define $\M_\varphi$ as the following decision tree.
	\begin{center}
		\begin{tikzpicture}
			\node[circle,draw=black] (c1) {$1$};
			\node[below left = 6mm and 6mm of c1] (tc1) {$\M_{C_1}$};
			\node[circle,draw=black,below right = 6mm and 6mm of c1] (c2) {$2$};
			\node[below left = 6mm and 6mm of c2] (tc2) {$\M_{C_2}$};
			\node[circle,draw=black,below right = 6mm and 6mm of c2] (c3) {$3$};
			\node[below left = 6mm and 6mm of c3] (tc3) {$\M_{C_3}$};
			\node[below right = 6mm and 6mm of c3] (d) {$\cdots$};
			\node[circle,draw=black,below right = 6mm and 6mm of d] (cn) {$n$};
			\node[below left = 6mm and 6mm of cn] (tcn) {$\M_{C_n}$};
			\node[circle,draw=black,below right = 6mm and 6mm of cn, minimum size=8mm] (o) {$\true$};

			\path[arrout] (c1) edge node[above] {$0$} (tc1);
			\path[arrout] (c1) edge node[above] {$1$} (c2);
			\path[arrout] (c2) edge node[above] {$0$} (tc2);
			\path[arrout] (c2) edge node[above] {$1$} (c3);
			\path[arrout] (c3) edge node[above] {$0$} (tc3);
			\path[arrout] (c3) edge node[above] {$1$} (d);
			\path[arrout] (d) edge node[above] {$1$} (cn);
			\path[arrout] (cn) edge node[above] {$0$} (tcn);
			\path[arrout] (cn) edge node[above] {$1$} (o);
		\end{tikzpicture}
	\end{center}

	Recall that the set of features of $\M_\varphi$ is $[1,n+m]$.
	The formula $\alpha$ is defined as
	$$\exists P_1 \forall P_2 \cdots \exists P_k \forall P \, \M_\varphi,$$
	assuming that $P_i$, for $1 \leq i \leq k$, is the set $\{n+\ell \mid x_\ell \in X_i\}$, and $P = \{1,\dots,n\}$.
	That is, $P_i$ is the set of features from $\M_\varphi$ that represent the variables in $X_i$ and $P$ is the set of features that are used to encode
	the clauses of $\varphi$.

	We show next that the equivalence stated in \eqref{eq:holds} holds. For simplicity, we only do it for the case $k = 1$. The proof for $k > 1$ uses exactly the same ideas, only that it is slightly more cumbersome.

	Assume, on the one hand, that $\psi = \exists X_1 \varphi$ holds. That is, there exists an assignment $\sigma_1 : X_1 \to \{0,1\}$
	such that $\varphi$ holds when variables in $X_1$ are interpreted according to $\sigma_1$. We show next
	that $\alpha = \exists P_1 \forall P \M_\varphi$ holds, where $P_1$ and $P$ are defined as above.
	Take the partial instance $\es_{\sigma_1}$ of dimension $n + m$ that naturally ``represents'' the assignment $\sigma_1$; that is:
	\begin{itemize}
		\item $\es_{\sigma_1}[i] = \bot$, for each
		$i \in \{1,\dots,n\}$,
		\item $\es_{\sigma_1}[n+i] = \sigma_1(x_i)$, for each $i \in \{1,\dots,m\}$ with $x_i \in X_1$, and
		\item $\es_{\sigma_1}[n+i] = \bot$,
		for each $i \in \{1,\dots,m\}$ with $x_i \not\in X_1$.
	\end{itemize}
	To show that $\alpha$ holds, it suffices to show that $\M_\varphi(\es) = 1$
	for every instance $\es$ of dimension $n+m$ that subsumes $\es_{\sigma_1}$. Take an arbitrary such an instance $\es \in \{0,1\}^{n+m}$.
	Notice that if $\es[i] = 1$, for every $i \in \{1,\dots,n\}$, then $\M_\varphi(\es) = 1$ by definition of $\M_\varphi$. Suppose then that there exists a minimum
	value $i \in \{1,\dots,n\}$ such that $\es[i] = 0$. Hence, to show that $\M_\varphi(\es) = 1$ we need to show that $\M_{C_i}(\es) = 1$.
	But this follows easily from the fact that $\es$ naturally represents an assignment $\sigma$ for $\varphi$ such that the restriction
	of $\sigma$ to $X_1$ is precisely $\sigma_1$. We know that any such an assignment $\sigma$ satisfies $\varphi$, and therefore it satisfies $C_i$. It follows
	that $\M_{C_i}(\es) = 1$.

	Assume, on the other hand, that $\alpha = \exists P_1 \forall P \M_\varphi$ holds. Then there exists a partial instance $\es$ of dimension $n+m$ such that the following statements hold:
	\begin{itemize}
		\item
		$\es[i] \neq \bot$ iff for some $j \in \{1,\dots,m\}$ it is the case that $i = n + j$ and $j \in P_1$, and
		\item for every $\es' \in \fullset(\es)$ we have that $\M_\varphi(\es') = 1$.
	\end{itemize}
	We show next that $\psi = \exists X_1 \varphi$ holds.
	Let $\sigma_1 : X_1 \to \{0,1\}$ be the assignment for the variables in $X_1$ that is naturally defined by $\es$. It suffices to show that each clause $C_i$ of $\varphi$, for $i \in \{1,\dots,n\}$,
	is satisfied by the
	assignment that interprets the variables in $X_1$ according to $\sigma_1$.
	Let us define a completion $\es'$ of $\es$ that satisfies the following:
	\begin{itemize}
		\item $\es'[i] = 0$,
		\item $\es'[j] = 1$, for each $j \in \{1,\dots,n\}$ with $i \neq j$, and
		\item $\es'[n+j] = \sigma_1(x_j)$, if $j \in \{1,\dots,m\}$ and $j \in P_1$.
	\end{itemize}
	We know that $\M_\varphi(\es') = 1$, which implies that $\M_{C_i}(\es') = 1$ (since $\es'$ takes value 0 for feature $i$).
	We conclude that $C_i$ is satisfied by the assignment which is naturally defined by $\es'$, which is precisely the one that interprets the variables in $X_1$ according to
	$\sigma_1$.

	\subsection{Proof of Lemma~\ref{presburger}}
	\label{proof:presburger}

	We induct on the depth of $\varphi$. First we will see the atomic cases:
\begin{itemize}
\item For atomic formulas of the form $x_i = x_j$, we take
\[
\sum_{\substack{\rho \in \{0, 1, \bot \}^k \\ \rho_i \neq \rho_j}}{z_\rho} \; = \; 0.
\]
\item For atomic formulas of the form $x_i \subseteq x_j$, we take
\[
\sum_{\substack{\rho \in \{0, 1, \bot \}^k \\ \rho_i \neq \rho_j \; \land \; \rho_i \neq \bot}}{z_\rho} \; = \; 0.
\]
\item For atomic formulas of the form $x_i \lel x_j$, we take
\[
\sum_{\substack{\rho \in \{0, 1, \bot \}^k \\ \rho_j = \bot}}{z_\rho} \; \leq \; \sum_{\substack{\rho \in \{0, 1, \bot \}^k \\ \rho_i = \bot}}{z_\rho}.
\]
\end{itemize}
Now we describe the structural induction. For the negation, it is enough to take $\operatorname{T}_\Gamma(\neg \psi)\bigl((z_\rho)_\rho\bigr)$ as $\neg \operatorname{T}_\Gamma(\psi)\bigl((z_\rho)_\rho\bigr)$.
If $\varphi$ is of the form $\psi_1 \circ \psi_2$, where $\circ$ is a binary logical connective, then we take $\operatorname{T}_\Gamma(\varphi)\bigl((z_\rho)_\rho\bigr)$ as $$\operatorname{T}_\Gamma(\psi_1)\bigl((z_\rho)\bigr) \; \circ \; \operatorname{T}_\Gamma(\psi_2)\bigl((z_\rho)\bigr).$$
Note that here it is important to use the inductive hypothesis with $\Gamma$ that contains both the free variables of $\psi_1$ and of $\psi_2$. We need to be more careful if $\varphi$ is of the form $\exists u\, \psi(y_1, \dots, y_\ell, u)$. By changing the name of the variable if necessary, we can assume that $u \not \in \Gamma$. Consider $\Lambda = (x_1, \dots, x_k, u)$ and take the formula
\[
	\operatorname{T}_\Lambda(\psi)\bigl((w_\tau)_{\tau\in\{0,1,\bot\}^{k+1}}\bigr)
\]
given by the inductive hypothesis. We take $\operatorname{T}_\Gamma(\varphi)\bigl((z_\rho)_{\rho\in\{0,1,\bot\}^{k}}\bigr)$ to be
\[
	\exists (w_\tau)_{\tau\in\{0,1,\bot\}^{k+1}} \quad \operatorname{Proj}_{\Lambda \to \Gamma}((w_\tau)_\tau,\, (z_\rho)_\rho) \; \land \; \operatorname{T}_\Lambda(\psi)\bigl((w_\tau)_{\tau\in\{0,1,\bot\}^{k+1}}\bigr),
\]
where $\operatorname{Proj}_{\Lambda \to \Gamma}((w_\tau)_\tau,\, (z_\rho)_\rho)$ is defined to be
\[
	\bigwedge_{\rho\in\{0,1,\bot\}^{k}}{z_{\rho} = w_{(\rho, 0)} + w_{(\rho, 1)} + w_{(\rho, \bot)}}.
\]
Finally, if $\varphi$ is of the form $\forall u\, \psi(y_1, \dots, y_\ell, u)$, we define $\Lambda$ as before and take $\operatorname{T}_\Gamma(\varphi)\bigl((z_\rho)_{\rho\in\{0,1,\bot\}^{k}}\bigr)$ to be
\[
	\forall (w_\tau)_{\tau\in\{0,1,\bot\}^{k+1}} \quad \bigl(\operatorname{Proj}_{\Lambda \to \Gamma}((w_\tau)_\tau,\, (z_\rho)_\rho) \; \rightarrow \; \operatorname{T}_\Lambda(\psi)\bigl((w_\tau)_{\tau\in\{0,1,\bot\}^{k+1}}\bigr) \bigr).
\]
For the second part of the lemma, note that for a subformula translated in a context $\Gamma$ of size $k$, the corresponding Presburger formula has $3^k$ variables. In the atomic cases we just need to manage sums without repetitions over those variables. Boolean connectives also do not cause any problems. In the case of quantifiers, we need to increase the size of the context from $k$ to $k+1$ and we also add $3^k$ projection formulas plus the recursive call. But because $k+1 \leq \rm{wd}(\varphi)$ and there are $O(|\varphi|)$ subformulas, the total output size is $O(|\varphi| \cdot 3^{\rm{wd}(\varphi)})$ up to polynomial factors of $\rm{wd}(\varphi)$. Notice that the same argument applies for proving that the computation itself can be done using at most that same space.

	\subsection{Proof of Lemma~\ref{lemma:maxp-hard}}
	\label{proof:maxp-hard}

	Consider the following similar problem. The input is a model $\M'$, an instance $\es'$ and a $k \in \mathbb{N}$, and the question is whether there exists a partial instance $\es$ that is a weak abductive explanation for $\es'$ on $\M'$ and whose number of defined features is at most $k$. This problem was studied in \cite{NEURIPS2020_b1adda14}, where it was shown to be $\np$-hard on decision trees. We show a reduction from this problem.

	First assume that $\M'(\es') = 1$. We create new variables $X_{i}$ for $i\in \{0,1,...,k\}$. Let $\M$ be a new decision tree such that $\dim(\M) = k + 1 + \dim(\M')$, depicted in the following figure:
	\begin{figure}[ht!]
		\begin{center}
			\includegraphics[width=0.3\textwidth]{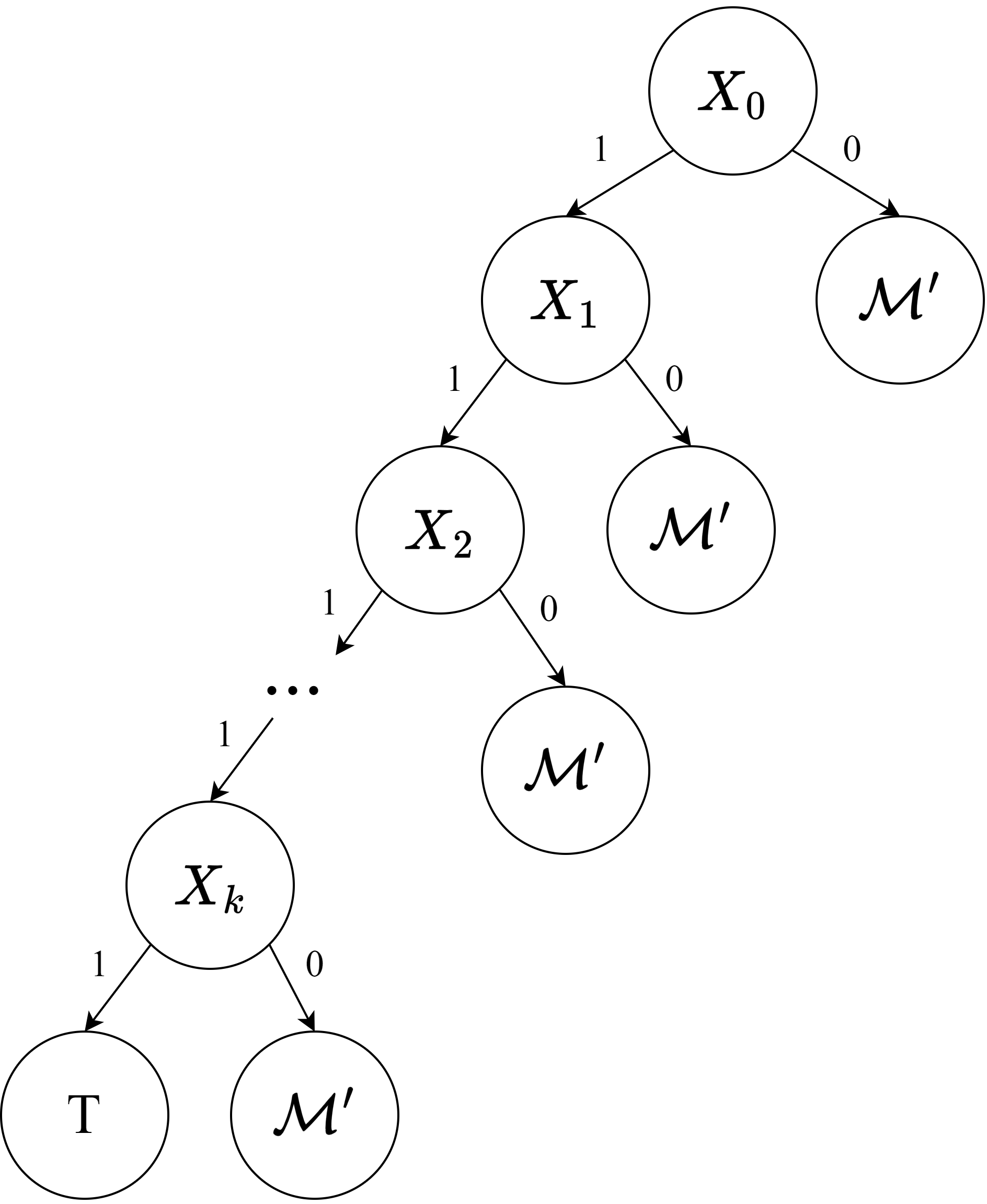}
			\Description[Visual description of the reduction.]
			{The figure shows the decision tree used in the reduction. It is a chain of new variables X_0, X_1, ..., X_k, where each edge labeled 1 continues along the chain and the final 1-edge reaches a leaf. Every edge labeled 0 branches to a copy of the original decision tree M'.}
		\end{center}
	\end{figure}

	We use $X_{0}$ as the root of $\M$. For every $i < k$, the outgoing edge of $X_i$ labeled by $1$ is connected to $X_{i+1}$, and the outgoing edge of $X_k$ labeled by $1$ is connected to a $\true$ leaf. Connect all outgoing edges labeled by $0$ to a copy of $\M'$. Let $\es_1 = \{1\}^{k+1} \cdot \es'$ and $\es_2 = \{1\}^{k+1} \cdot \{\bot\}^{\dim(\M')}$ be partial instances of size $\dm(\M)$. We claim that $\es_2$ is a minimum abductive explanation for $\es_1$ on $\M$ if and only if the answer to the original problem was negative. For the case $\M'(\es') = 0$ we can just set the value of the new leaf to $\false$ and the same construction will work.

	We now discuss why the reduction works. Suppose first that $(\M', \es', k)$ outputs $\yes$. It follows that there exists a partial instance $\es$ on $\M'$ that is a weak abductive explanation for $\es'$ with at most $k$ defined features. Now notice that the partial instance $\{\bot\}^{k+1} \cdot \es$ is a weak abductive explanation for $\es_1$ on $\M$ and has at most $k$ defined features. Because $\es_2$ has $k+1$ defined features, it follows that $\es_2$ is not a minimum abductive explanation for $\es_1$ on $\M$.

	Now suppose that $(\M', \es', k)$ outputs $\no$. Then there is no weak abductive explanation for $\es'$ on $\M'$ with at most $k$ defined features. This implies that there is also no partial instance $\es$ on $\M$ that is a weak abductive explanation for $\es_1$ with at most $k$ defined features. This is because any candidate weak abductive explanation for $\es_1$ with at most $k$ defined features leaves at least one of the new variables undefined, and therefore some completion of it reaches a copy of $\M'$. Once it enters that copy of $\M'$, its restrictions on the old coordinates would induce a weak abductive explanation for $\es'$ on $\M'$ with at most $k$ defined features, contradicting the assumption. But we know that $\es_2$ is a weak abductive explanation for $\es_1$ with $k+1$ defined features, so it is a minimum abductive explanation for $\es_1$ on $\M$, as we needed.

	\subsection{Proof of Lemma~\ref{lem-fixed-rho}}
	\label{proof:fixed-rho}

	We first treat the parameter-free case $\rho(x,y)$, and then we extend the idea to the general case.

	Fix a dimension $n$. Given a partial instance of dimension $n$, define $\#_0(\es)$ as the number of occurrences of the symbol $0$ in $\es$, and likewise for $\#_1(\es)$ and $\#_\bot(\es)$. Moreover, for every $(p,q,r) \in \mathbb{N}^3$ such that $p + q + r = n$, define
	\begin{align*}
		L_{(p,q,r)} \ := \ \{ \es \mid \es \text{ is a partial instance of
			dimension } n \text{ such that } \#_0(\es) = p,\, \#_1(\es) = q
		\text{ and } \#_\bot(\es) = r \}.
	\end{align*}
	Notice that there are at most $\binom{n+2}{2} \leq (n+1)^2$ different sets $L_{(p,q,r)}$. We claim that if $\es_1$ and $\es_2$ are partial instances of dimension $n$ such that $\es_1, \es_2 \in L_{(p,q,r)}$ for the same triple $(p, q, r)$, then $\mathfrak{B}_n \not\models \rho(\es_1, \es_2)$. From that we can conclude that the statement of the lemma holds
	for $p(n) = (n+1)^2$, since if $(\es_1, \ldots, \es_k)$ is a path of
	dimension $n$ in $\rho(x,y)$, then each $\es_i$ must belong to a
	different set $L_{(p,q,r)}$.

	For the sake of contradiction, suppose that $\es_1, \es_2$ are two different partial instances of dimension $n$ that belong to the same set $L_{(p, q, r)}$ and such that $\mathfrak{B}_n \models \rho(\es_1, \es_2)$. Then there exists a permutation $\pi : \{1, \ldots, n\}
	\to \{1, \ldots, n\}$ such that $\pi(\es_1) = \es_2$. Notice that for every pair
	$\es$, $\es'$ of partial instances of dimension $n$ it holds that:
	\begin{align*}
		\es \subseteq \es' &\quad\quad\Leftrightarrow\quad\quad   \pi(\es) \subseteq \pi(\es')\\
		\es \lel \es' &\quad\quad\Leftrightarrow\quad\quad   \pi(\es) \lel \pi(\es').
	\end{align*}
	Thus, $\pi$ is an automorphism for the structure $\mathfrak{B}_n$. Because we have that $\mathfrak{B}_n \models \rho(\es_1, \es_2)$, and $\rho$ is defined over the
	vocabulary $\{ \subseteq, \lel\}$, it follows $\mathfrak{B}_n \models \rho(\pi(\es_1), \pi(\es_2))$. But since $\es_2 = \pi(\es_1)$, we also have that $\mathfrak{B}_n \models \rho(\es_2, \pi(\pi(\es_1)))$. Because $\rho(x, y)$ is transitive, it follows that $\mathfrak{B}_n \models \rho(\es_1, \pi^2(\es_1))$. In the same way, we can conclude that $\mathfrak{B}_n \models \rho(\es_1, \pi^k(\es_1))$ for every $k \geq 1$. Given that the set of
	permutations of $n$ elements with the composition operator forms a group of order $n!$, we know that $\pi^{n!}$ is the identity permutation, so that $\pi^{n!}(\es_1) = \es_1$. Therefore, we
	conclude that $\mathfrak{B}_n \models \rho(\es_1, \es_1)$, which leads to a contradiction since $\rho(x,y)$ represents a strict partial order.

	Consider now a formula $\rho[v_1, \ldots, v_\ell](x,y)$ with parameters. As in the previous case, we fix a dimension $n$. Moreover, we also fix a sequence $\es'_1$, $\ldots$, $\es'_\ell$ of partial instances of dimension $n$ (notice that the bound $p(n)$ should not depend on those partial instances). Then, for every $(a_1, \ldots, a_\ell) \in \{0,1,\bot\}^\ell$, consider the set
	\begin{align*}
		P_{(a_1, \ldots, a_\ell)} \ = \ \{i \in \{1, \ldots, n\} \ \mid \ \es'_j[i] = a_j \text{ for all } j \in \{1, \ldots, \ell\}\},
	\end{align*}
	that is, all positions for which the sequence $\es'_1$, $\ldots$, $\es'_\ell$ realizes the pattern $(a_1, \ldots, a_\ell)$. Given $s \in \{0, 1, \bot\}$, a pattern $t \in \{0,1,\bot\}^\ell$ and a partial instance $\es$ of dimension $n$, define $\#_{s, t}(\es)$ as the number of indices $i \in P_t$ such that $\es[i]=s$. Notice that the numbers $\#_{s, t}(\es)$ are invariant under permutations of the features that map each pattern block onto itself.

	We define an equivalence relation as follows. For two partial instances $\es_1$ and $\es_2$ of dimension $n$, we write $\es_1 \sim \es_2$ if $\#_{s, t}(\es_1) = \#_{s, t}(\es_2)$ for every $s \in \{0, 1, \bot\}$ and every $t \in \{0,1,\bot\}^\ell$. Notice that there are at most
	$$\binom{n+3^{\ell+1}-1}{3^{\ell+1}-1} \leq \frac{\big(n+3^{\ell+1}-1\big)^{3^{\ell+1}-1}}{\big(3^{\ell+1}-1 \big)!}$$
	different equivalence classes. Because $\ell$ is fixed, we can consider that number as our polynomial $p(n)$.

	We claim that if $\es_1 \sim \es_2$, then $\mathfrak{B}_n \not\models \rho[\es'_1, \ldots, \es'_\ell](\es_1, \es_2)$. From that we can conclude, as in the parameter-free case, that the statement of the lemma holds. In fact, if $(\es_1, \ldots, \es_k)$ is a path of dimension $n$ in $\rho[v_1, \ldots, v_\ell](x,y)$, then each $\es_i$ must belong to a different equivalence class.

	For the sake of contradiction, suppose that $\es_1, \es_2$ are two different partial instances of dimension $n$ such that $\es_1 \sim \es_2$ and $\mathfrak{B}_n \models \rho[\es'_1, \ldots, \es'_\ell](\es_1, \es_2)$. For each pattern block $P_t$, consider a permutation $\pi_t$ of $P_t$ sending the restriction of $\es_1$ on $P_t$ to the restriction of $\es_2$ on $P_t$. Combining these permutations yields a permutation $\pi : \{1, \ldots, n\} \to \{1, \ldots, n\}$ such that $\pi(\es_1) = \es_2$. Notice that by the way we constructed the permutations $\pi_t$ and the pattern blocks $P_{t}$, we have that $\pi(\es'_j) = \es'_j$ for every $j \in
	\{1, \ldots, \ell\}$.

	As in the parameter-free case, we have that $\pi$ is an automorphism for the structure $\mathfrak{B}_n$. Since $\rho$ is a formula defined over the vocabulary $\{\subseteq, \lel\}$, and $\mathfrak{B}_n \models \rho[\es'_1, \ldots, \es'_\ell](\es_1, \es_2)$, it follows that
	$\mathfrak{B}_n \models \rho[\pi(\es'_1), \ldots, \pi(\es'_\ell)](\pi(\es_1), \pi(\es_2))$.
	Because $\pi(\es'_j) = \es'_j$ for every $j \in \{1, \ldots, \ell\}$ and $\es_2	= \pi(\es_1)$, we obtain that $\mathfrak{B}_n \models \rho[\es'_1, \ldots, \es'_\ell](\es_2, \pi(\pi(\es_1)))$, and using that $\rho[\es'_1, \ldots, \es'_\ell](x, y)$ is transitive, we also have that $\mathfrak{B}_n \models \rho[\es'_1, \ldots, \es'_\ell](\es_1, \pi^2(\es_1))$. In the same way, it is possible to conclude that $\mathfrak{B}_n \models
	\rho[\es'_1, \ldots, \es'_\ell](\es_1, \pi^k(\es_1))$ for every $k
	\geq 1$. Given that the set of permutations of $n$ elements with the
	composition operator forms a group of order $n!$, we know that
	$\pi^{n!}$ is the identity permutation, so that $\pi^{n!}(\es_1) =
	\es_1$. Therefore, we conclude that $\mathfrak{B}_n \models \rho[\es'_1, \ldots,
	\es'_\ell](\es_1, \es_1)$, which leads to a contradiction since
	$\rho[v_1, \ldots, v_\ell](x,y)$ represents a strict partial order.

        \subsection{Auxiliary predicates}
        \label{app:def-aux}

\paragraph{Definition of the formula $\led(x,y,z)$.}
\label{app:def-led} Let $\meet$ (\emph{Greatest Lower Bound}) be the following formula:
\begin{align*}
	\meet(x,y,z) \ := \ z \subseteq x \wedge z \subseteq y \wedge \forall w \, ((w \subseteq x \wedge w \subseteq y) \to w \subseteq z).
\end{align*}
The interpretation of this predicate is such that for every model $\M$ of dimension $n$ and every sequence of instances $\es_1, \es_2, \es_3$ it holds that $\M \models \meet(\es_1, \es_2, \es_3)$ if and only if $\es_3$ is the greatest partial instance subsumed by $\es_1$ and $\es_2$, i.e., the partial instance with most defined features subsumed by both. This property allows us to measure the number of defined features on which the two instances agree. By using this predicate, let $\led$ be a ternary predicate such that $\M \models \led(\es_1, \es_2, \es_3)$ if and only if the Hamming distance between $\es_1$ and $\es_2$ is less than or equal to the Hamming distance between $\es_1$ and $\es_3$. The relation $\led$ can be expressed as a formula from the atomic layer of $\ffoil$ as follows:
\begin{align*}
	\led(x,y,z) \ := \ \full(x) \wedge \full(y) \wedge \full(z) \ \wedge \exists w_1 \exists w_2 \, (\meet(x,y,w_1) \wedge
	\meet(x,z,w_2) \wedge w_2 \lel w_1).
\end{align*}

\paragraph{Definition of the formula $\add(x,y,z)$.}  Let $\lu$ (\emph{Level Up}) be the following formula:
\begin{align*}
	\lu(x, y) \ := \ x \lnel y \wedge \neg \exists z \ (x \lnel z \wedge z \lnel y),
\end{align*}
such that $\M \vDash \lu(\es_1, \es_2)$ if and only if $\es_1$ has exactly one less defined feature than $\es_2$. By using this predicate, let $\add$ be a ternary predicate such that $\M \vDash \add(\es_1, \es_2, \es_3)$ if and only if $\es_2$ is a feature subsumed by $\es_3$ and $\es_1$ is obtained from $\es_3$ by undefining the feature $\es_2$. The relation $\add$ can be expressed as formula from the atomic layer of $\ffoil$ as follows:
\begin{align*}
	\add(x,y,z) \ := \ {\sf Single}(y) \wedge x \subseteq z \wedge \lu(x, z) \wedge y \subseteq z \wedge \neg(y \subseteq x).
\end{align*}
Recall that ${\sf Single}(x)$ defines the set of partial instances with exactly one defined feature, and can be expressed as follows:
$${\sf Single}(x) := \exists y (y \subset x) \wedge \forall y (y \subset x \, \rightarrow \, \neg \exists z (z \subset y)).$$